\documentclass[11pt]{article} 

\usepackage[T1]{fontenc}
\usepackage[utf8]{inputenc}
\usepackage{lmodern}
\usepackage{url}
\usepackage{amsmath,amsthm,amssymb}
\newtheorem{theorem}{Theorem}

\newtheorem{lemma}[theorem]{Lemma}

\newtheorem{definition}[theorem]{Definition}

\usepackage{algorithm}
\usepackage{algpseudocode}

\usepackage{caption}
\usepackage{bm}
\usepackage{xcolor}
\usepackage{xspace}
\usepackage{paralist}
\usepackage{graphicx}
\usepackage{ifpdf}
\usepackage{subcaption}
\usepackage{fullpage}
\usepackage[font=small,labelfont=bf]{caption}

\newcommand{\Var}[1]{\mathbf{Var}\left[#1\right]}	

\newcommand{\Exp}[1]{\mathbf{E}\left[#1\right]}
\newcommand{\ip}[1]{\ensuremath{\langle#1\rangle}}
\newcommand{\Prob}[1]{\mathbf{Pr}\left[#1\right]}
\newcommand{\nat}{\ensuremath{\mathbb{N}}}
\newcommand{\real}{\ensuremath{\mathbb{R}}}
\newcommand{\hide}[1]{{}}
\newcommand{\grad}{\ensuremath{\bm{\mathcal \nabla}}\xspace}
\newcommand{\bx}{\ensuremath{\bm x}\xspace}
\newcommand{\by}{\ensuremath{\bm y}\xspace}
\newcommand{\bz}{\ensuremath{\bm z}\xspace}
\newcommand{\COMMENTED}[1]{{}}

\usepackage{hyperref}
\usepackage{url}
\newcommand{\out}[1]{\COMMENTED{#1}}

\title{New Bounds For Distributed \\ Mean Estimation and  Variance Reduction}

\begin{document}

\author{
	Peter Davies\\{IST Austria}\\\texttt{peter.davies@ist.ac.at}
	\and 
	Vijaykrishna Gurunathan\\{IIT Bombay}\\\texttt{krishnavijay1999@gmail.com}
	\and
	Niusha Moshrefi\\{IST Austria}\\\texttt{niusha.moshrefi@ist.ac.at}
	\\ \and 
	Saleh Ashkboos\\{IST Austria}\\\texttt{saleh.ashkboos@ist.ac.at}
	\and
	Dan Alistarh\\{IST Austria \& NeuralMagic}\\\texttt{dan.alistarh@ist.ac.at}
}
	
\date{}
\maketitle
\thispagestyle{empty}
		\begin{abstract}

            We consider the problem of \emph{distributed mean estimation (DME)}, in which $n$ machines are each given a local $d$-dimensional vector $\bx_v \in \real^d$, and must cooperate to estimate the mean of their inputs $\bm \mu = \frac 1n\sum_{v = 1}^n \bx_v$, while minimizing total communication cost. 
            DME is a fundamental construct in distributed machine learning, and there has been considerable work on variants of this problem, especially in the context of \emph{distributed variance reduction} for stochastic gradients in parallel SGD. Previous work typically assumes an upper bound on the norm of the input vectors, and achieves an error bound in terms of this norm. However, in many real applications, the input vectors are concentrated around the correct output $\bm \mu$, but $\bm \mu$ itself has large norm. In such cases, previous output error bounds perform poorly. 
            
			In this paper, we show that output error bounds need not depend on input norm. We provide a method of quantization which allows distributed mean estimation to be performed with solution quality dependent only on the \emph{distance between inputs}, not on input norm, and show an analogous result for distributed variance reduction. The technique is based on a new connection with lattice theory. We also provide lower bounds showing that the \emph{communication to error trade-off} of our algorithms is asymptotically optimal.
			As the lattices achieving optimal bounds under $\ell_2$-norm can be computationally impractical, 
			we also present an extension which leverages  easy-to-use cubic lattices, and is loose only up to a logarithmic factor in $d$. We show experimentally that our method yields practical improvements for common applications, relative to prior approaches.  
		\end{abstract}

\pagebreak
\tableofcontents
\section*{Acknowledgements}
Peter Davies is supported by the European Union’s Horizon 2020 research and innovation programme under the Marie Skłodowska-Curie grant agreement No. 754411. This  project  has  also received  funding  from  the  European  Research  Council  (ERC) under  the  European  Union’s  Horizon  2020  research  and  innovation  programme (grant agreement No. 805223 ScaleML).	
\pagebreak
	\section{Introduction}\label{sec:intro}
	
	Several problems in distributed machine learning and optimization can be reduced to variants \emph{distributed  mean estimation} problem, in which $n$ machines must cooperate to jointly estimate the mean of their $d$-dimensional inputs $\bm \mu = \frac 1n\sum_{v = 1}^n \bx_v$ as closely as possible, while minimizing communication. 
	In particular, this construct is often used for \emph{distributed variance reduction}: here, each machine receives as input an independent probabilistic estimate of a $d$-dimensional vector $\grad$, and the aim is for all machines to output a \emph{common estimate} of $\grad$ with \emph{lower variance} than the individual inputs,  minimizing communication.
	Without any communication restrictions, the ideal output would be the mean of all machines' inputs. 
	
	While variants of these fundamental problems have been considered since seminal work by Tsitsiklis and Luo \cite{tsitsiklis1987communication}, the task has seen renewed attention recently in the context of distributed machine learning. In particular, variance reduction is a key component in data-parallel distributed stochastic gradient descent (SGD), the standard way to parallelize the training of deep neural networks, e.g.~\cite{bottou2010large, abadi2016tensorflow}, where it is used to estimate the average of gradient updates obtained in parallel at the nodes.
	Thus, several prior works proposed efficient compression schemes to solve variance reduction or mean estimation, see e.g.~\cite{MeanEstimation, QSGD, ramezani2019nuqsgd, VQSGD}, and~\cite{ben2019demystifying} for a general survey of practical distribution schemes. These schemes seek to \emph{quantize} nodes' inputs coordinate-wise to one of a limited collection of values, in order to then efficiently encode and transmit these quantized values. A trade-off then arises between the  \emph{number of bits sent}, and the \emph{added variance} due of quantization.    
	
	\out{
		such that each machine $v$ receives as input a $d$-dimensional vector $\bx_v \in \real^d$, and a bound $y$ on how far any other input vector can be from $\bx_v$, i.e. $\| \bx_u - \bx_v \| < y, \forall u \in M$. Note that this also implies that each input is at most distance $y$ from the mean input. 
		In this paper, we will focus on a \emph{probabilistic}  version of this problem, called \emph{variance reduction}. 
		In variance reduction, the input vectors are themselves random estimates of a single \emph{true} vector \grad with variance $\sigma^2$ , and the goal is to output a new estimate with \emph{minimal variance}. 
		Thus, each machine should return an accurate estimate of the \emph{mean} of the input vectors $\bm\mu = \frac{1}{n} \sum_{v = 1}^M \bx_v$, while minimizing the total communication. 
		
	}

	Since the measure of \emph{output} quality is variance, it appears most natural to evaluate this with respect to \emph{input} variance, in order to show that \emph{variance reduction} is indeed achieved. Surprisingly, however, we are aware of no previous works which do so; all existing methods give bounds on output variance in terms of the \emph{squared input norm}. This is clearly suboptimal when the squared norm is higher than the variance, i.e., when inputs are not centered around the origin. In some practical scenarios this causes output variance to be \emph{higher} than input variance, as we demonstrate in Section \ref{sec:exp}.
	
	\out{
		
		and (In Section~\ref{sec:exp} we provide evidence that this bounded input assumption is much stronger than the bounded variance assumption in practical scenarios.) }
	
	\paragraph{Contributions.} In this paper, we provide the first bounds for distributed mean estimation and variance reduction which are still tight when inputs are not centered around the origin.
	Our results are based on new lattice-based quantization techniques, which may be of independent interest, and come with matching lower bounds, and practical extensions. 
	More precisely, our contributions are as follows: 
	\begin{itemize}
		
		\item For \emph{distributed mean estimation}, we show that, to achieve a reduction of a factor $q$ in the input `variance' (which we define to be the maximum squared distance between inputs), it is necessary and sufficient for machines to communicate $\Theta( d \log q)$ bits. 
		
		\item 
		For \emph{variance reduction}, we show tight $\Theta(d\log n)$ bounds on the worst-case  communication bits required to achieve optimal $\Theta(n)$-factor variance reduction by $n$ nodes over $d$-dimensional input, and indeed to achieve any variance reduction at all. We then show how incorporating \emph{error detection} into our quantization scheme, we can also obtain tight bounds on the bits required \emph{in expectation}.
		
		\item We show how to efficiently instantiate our lattice-based quantization framework in practice, with guarantees. In particular, we devise a variant of the scheme which ensures close-to-optimal  communication-variance bounds even for the standard \emph{cubic lattice}, and use it to obtain improvements relative to the best known previous methods for distributed mean estimation, both on synthetic  and real-world tasks. 
		
	\end{itemize}
	
	\subsection{Problem Definitions and Discussion}\label{sec:prelim}
 \textsc{MeanEstimation} is defined as follows: we have $n$ machines $v$, and each receives as input a vector $\bx_v \in \real^d$. We also assume that all machines receive a common value $y$, with the guarantee that for any machines $u,v$, $\|\bx_u - \bx_v\|\le y$. Our goal is for all machines to output the same value $\bm{EST} \in \real^d$, which is an unbiased estimator of the mean $\bm\mu=\frac 1n \sum_{v\in M}\bm x_v$, i.e. $\Exp{\bm{EST}}=\bm\mu$, with variance as low as possible. Notice that the input specification is entirely deterministic; any randomness in the output arises only from the algorithm used.	
		
In the variant of \textsc{VarianceReduction}, we again have a set of $n$ machines, and now an unknown \emph{true} vector \grad. Each machine $v$ receives as input an independent unbiased estimator $\bx_v$ of $\grad$ (i.e., $\Exp{\bx_v}=\grad$) with variance $\Exp{\|\bx_v-\grad\|^2}\le \sigma^2$. Machines are assumed to have knowledge of $\sigma$. Our goal is for all machines to output the same value $\bm{EST} \in \real^d$, which is an unbiased estimator of \grad, i.e., $\Exp{\bm{EST}}=\grad$, with low variance. Since the input is random, output randomness now stems from this input randomness as well as any randomness in the algorithm.	
	
	\textsc{VarianceReduction} is common for instance in the context of gradient-based optimization of machine learning models, 
	where we assume that each machine $v$ processes local samples in order to obtain a \emph{stochastic gradient} $\tilde{g}_v$, which is an unbiased estimator of the true gradient $\grad$, 
	with variance bound $\sigma^2$. 
	If we directly averaged the local stochastic gradients $\tilde{g}_v$, we could obtain an unbiased estimator of the true gradient $G$ with variance bound $\sigma^2 / n$, which can lead to faster convergence.  
	
	\paragraph{Input Variance Assumption.} 
	The parameter $y$ replaces the usual \textsc{MeanEstimation} assumption of a known bound $\mathbb M$ on the \emph{norms of input vectors}. Note that, in the worst case, we can always set $y=2\mathbb M$ and obtain the same asymptotic upper bounds as in e.g. \cite{MeanEstimation}; our results are therefore \emph{at least} as good as previous approaches in all cases, but, as we will show, provide significant improvement when inputs are not centered around the origin.
	
	The reason for this change is to allow stronger bounds in scenarios where we expect inputs to be closer to each other than to the origin. In particular, it allows our \textsc{MeanEstimation} problem to more effectively generalize \textsc{VarianceReduction}. Parameter $y$ is a deterministic analogue of the parameter $\sigma$ for \textsc{VarianceReduction}; both $y$ and $\sigma$ provide a bound on the distance of inputs from their \emph{mean}, rather than from the origin. Accordingly, \emph{input variance} $\sigma^2$ for a \textsc{VarianceReduction} instance corresponds (up to constant factors) to $y^2$ for a \textsc{MeanEstimation} instance. For consistency of terminology, we therefore refer to $y^2$ as the \emph{input variance} of the instance (despite such inputs being deterministic).
	
	It is common in machine learning applications of \textsc{VarianceReduction} to assume that an estimate of the variance $\sigma^2$ is known~\cite{QSGD,VQSGD}. To study both problems in a common framework, we make the analogous assumption about \textsc{MeanEstimation}, and assume knowledge of the \emph{input variance} $y^2$. 
	Even if the relevant bounds $y$ or $\sigma$ are not known \emph{a priori}, they can usually be  estimated in practice. We discuss how we obtain estimates of input variance for our applications in Section~\ref{sec:exp}.
	
	\paragraph{Relationship Between Problems.}
	If one allows unrestricted communication, the straightforward solution to both problems is to average the inputs. This is an exact solution to \textsc{MeanEstimation} with variance $0$, and is an asymptotically optimal solution to \textsc{VarianceReduction}, of variance at most $\frac{\sigma^2}{n}$.\footnote{For specific classes of input distribution, and for non-asymptotic concentration results, however, better estimators of the mean are known; see e.g. \cite{joly2017estimation}.}
	However, doing so would require the exchange of infinite precision real numbers. So, it is common to instead communicate \emph{quantized} values of bounded bit-length~\cite{QSGD}, which will engender additional variance caused by random choices within the quantization method. The resulting estimates will therefore have variance $\bm{Var}_{quant}$ for \textsc{MeanEstimation}, and $\frac{\sigma^2}{n}+\bm{Var}_{quant}$ for \textsc{VarianceReduction}. We will show a trade-off between bits of communication and output variance for both problems; in the case of \textsc{VarianceReduction}, though, there is an `upper limit' to this trade-off, since we cannot go below $\Omega(\frac{\sigma^2}{n})$ total output variance.
	
	The other major difference between the two problems is that in \textsc{MeanEstimation}, distances between inputs are bounded by $y$ with certainty, whereas in \textsc{VarianceReduction} they are instead bounded by $O(\sigma)$ only in \emph{expectation}. This causes extra complications for quantization, and, as we will see, introduces a gap between average and worst-case communication cost.
	\paragraph{Distributed Model.}
We aim to provide a widely applicable method for distributed mean estimation, and therefore we avoid relying on the specifics of particular distributed models. Instead, we assume that the basic communication structures we use (stars and binary trees) can be constructed without significant overhead. This setting is supported by machine learning applications, which have very high input dimension (i.e., $d \gg n$), and so the costs of synchronization or construction of an overlay (which do not depend on $d$, and are generally poly-logarithmic in $n$), will be heavily dominated by the communication costs incurred subsequently during mean estimation. They also need only be incurred once, even if mean estimation or variance reduction is to be performed many times (e.g. during distributed SGD). For these reasons, we do not include these model-specific setup costs in our stated complexities; any implementation of our techniques should take them into separate consideration.

For simplicity, we will present our algorithms within a basic synchronous fault-free message-passing model, in which machines can send arbitrary messages to any other machine, but they could naturally be extended to asynchronous and shared-memory models of communication. Our aim will be to minimize the number of bits sent and received by any machine during the course of the algorithm, i.e., we do not consider other measures such as round complexity.

	\paragraph{Vector Norms.}
	When dealing with vectors in $\real^d$, we will use names in bold, e.g. $\bx$, $\bm y$. We will state most of our results in such a way that they will apply to any of the three most commonly-used norms on $\real^d$ in applications: $\ell_1$ norm $\|\bx\|_1 := \sum_{i=1}^d |x_i|$, $\ell_2$ norm $\|\bx\|_2 := \sqrt{\sum_{i=1}^d x_i^2}$, and $\ell_\infty$ norm $\|\bx\|_\infty := \max_{i=1}^d x_i$. Throughout the paper we will therefore use the general notation $\|\cdot\|$, which should be considered to be fixed as one of these norms, other than for statements specific to particular norms. Definitions which depend on norms, such as variance $\Var{\bx}:=\Exp{\|\bx - \Exp{\bx}\|^2}$, are therefore assumed to also be under the appropriate norm. 
	
	\subsection{Related Work}
	\label{sec:related-work}

	Several recent works consider efficient compression schemes for stochastic gradients, e.g.~\cite{1bitqsgd, atomo, QSGD, TopK, stich2018sparsified, terngrad, wangni2018gradient, lu2020moniqua}. 
	We emphasize that these works consider a related, but different problem: they usually rely on assumptions on the input structure---such as second-moment bounds on the gradients---and are evaluated primarily on the practical performance of SGD, rather than isolating the variance-reduction step. 
	(In some cases, these schemes also rely on history/error-correction~\cite{aji2017sparse, dryden2016communication, TopK, stich2018sparsified}.)
	As a result, they do not provide theoretical bounds on the problems we consider.
	In this sense, our work is closer to~\cite{MeanEstimation, RDME, VQSGD}, which focus primarily on the \emph{distributed mean estimation} problem, and only use SGD as one of many potential applications. 
	
	For example, QSGD \cite{QSGD} considers a similar problem to \textsc{VarianceReduction}; the major  difference is that coordinates of the input vectors are assumed to be specified by 32-bit floats, rather than arbitrary real values. Hence, transmitting input vectors exactly already requires only $O(d)$ bits. They therefore focus on reducing the constant factor (and thereby improving practical performance for SGD), rather than providing asymptotic results on communication cost. They show that the expected number of bits per entry can be reduced from $32$ to $2.8$, at the expense of having an output variance bound in terms of \emph{input norm} rather than \emph{input variance}. 
	
	This is a common issue with existing quantization schemes, which leads to non-trivial complications when applying quantization to gradient descent and variance-reduced SGD~\cite{Kunstner} or to model-averaging SGD~\cite{lu2020moniqua}, since in this case the inputs are clearly not centered around the origin. The standard way to circumvent this issue, adopted by the latter two references, but also by other work on quantization~\cite{mishchenko2019distributed}, is to carefully adapt the quantization scheme and the algorithm to remove this issue, for instance by quantizing \emph{differences} with respect to the last quantization point. These approaches, however, do not provide improvement as `one-shot' quantization methods, and instead rely on historical information and properties of SGD or the function to optimize (such as smoothness). They are therefore inherently application-specific. Our method, by contrast, does not require ``manual'' centering of the iterates, and does not require storage of previous iterates, or any properties thereof. 
	
	Konečný and Richtárik \cite{RDME} study \textsc{MeanEstimation} under similar assumptions, and are the only prior work to use quantization centered around points \emph{other} than the origin. However, again prior  knowledge about the input distribution must be assumed for their scheme to provide any improvements. 
	
	Suresh et al. \cite{MeanEstimation} study the \textsc{MeanEstimation} problem defined on real-valued input vectors. 
	They present a series of quantization methods, providing an $O(\frac {1}{n^2} \sum_{v\le n}\|\bx _ v\|_2^2)$ upper bound, and corresponding lower bounds. Recent work by Gandikota et al. \cite{VQSGD} studies \textsc{VarianceReduction}, and uses multi-dimensional quantization techniques. However, their focus is on protocols using $o(d)$-bit messages per machine (which we show \emph{cannot} reduce input variance). They do give two quantization methods using $\Theta(d)$-bit messages. Of these, one gives an $O(\frac {1}{n} \max_{v\le n}\|\bx _ v\|_2^2)$ bound on output variance, similar to the bound of \cite{MeanEstimation} for \textsc{MeanEstimation} (the other is much less efficient since it is designed to achieve a privacy guarantee). Mayekar and Tyagi \cite{RATQ} obtain a similar error bound but with slightly longer $\Theta(d \log\log\log (\log^* d))$-bit messages.
	
	All of the above works provide output error bounds based on the norms of input vectors. This is only optimal under the implicit assumption that inputs are centered around the origin. 
	In Section~\ref{sec:exp} we provide evidence that this assumption does not hold in some practical scenarios, where the input (gradient) variance can be much lower than the input (gradient) norm: intuitively, for SGD, input variance is only close to squared norm when true gradients are close to $\bm 0$, i.e., the optimization process is already almost complete.

\section{Our Results}	
	In this work, we argue that it is both stronger and more natural to bound output variance in terms of input variance, rather than squared norm. We devise optimal quantization schemes for \textsc{MeanEstimation} and \textsc{VarianceReduction}, and prove matching lower bounds, regardless of input norms. We summarize the main ideas that lead to these results.
\subsection{Lattice-Based Quantization}
	
The reason that all prior works obtain output variance bounds in terms of input norm rather than variance is that they employ sets of quantization points which are \emph{centered} around the origin $\bm 0$. We instead cover the entire space $\real^d$ with quantization points that are in some sense uniformly spaced, using \emph{lattices}.

Lattices are subgroups of $\real^d$ consisting of the integer combinations of a set of basis vectors. It is well-known \cite{Minkowski11} that certain lattices have desirable properties for covering and packing Euclidean space, and lattices have been previously used for some other applications of quantization (see, e.g., \cite{LatticeQuant}), though mostly only in low dimension. By choosing an appropriate family of lattices, we show that any vector in $\real^d$ can be rounded (in a randomized, unbiased fashion) to a nearby lattice point, but also that there are not too many nearby lattice points, so the correct one can be specified using few bits.

Lattices contain an infinite number of points, and therefore any encoding using a finite number of bits must use bit-strings to refer to an infinite amount of lattice points. To allow the receiver in our quantization method to correctly decode the intended point, we utilize the fact that we have a bound on the distance between any two machines' inputs ($y$ for \textsc{MeanEstimation}, and $O(\sigma\sqrt{ n})$ (probabilistically, by Chebyshev's inequality) for \textsc{VarianceReduction}). Therefore, if all points that map to the same bit-string are sufficiently far apart, a machine can correctly decode based on proximity to its own input.

The simplest version of our lattice quantization algorithm can be described as follows
\begin{itemize}
\item To encode $\bx_u$, randomly map to one of a set of nearby lattice points forming a convex hull around $\bx_u$. Denote this point by $\bz$.
\item Send $\bz\bmod q$ under the lattice basis: $q$ is the quantization precision parameter.
\item To decode with respect to $\bx_v$, output the closest lattice point to $\bx_v$ matching $\bz\bmod q$ .
\end{itemize}

By showing that $\bx_u$ is contained within a convex hull of nearby lattice points, we can round to one of these points randomly to obtain $\bz$ such that the expectation of $\bz$ is $\bx_u$ itself, thereby ensuring \emph{unbiasedness}. This is because a point within a convex hull can be expressed as a linear combination of its vertices with coefficients in $[0,1]$, which we can use as rounding probabilities. For the cubic lattice this procedure is particularly simple: since the lattice basis is orthogonal, we can round $\bx_u$ coordinate-wise. We round each coordinate either up or down to the closest multiple of the relevant lattice basis vector, choosing the respective probabilities so that the expected value is the coefficient of $\bx_u$.

Our reason for using $\bmod$ $q$ with respect to the lattice basis in order to encode lattice points into bit-strings is that by exploiting properties of the particular lattices we employ, we can show a lower bound on the distance between points encoded with the same bit-string, while also controlling the number of bits we use. Then, since points encoded with the same string are sufficiently far apart, our proximity-based decoding procedure can determine the correct point.
We also have a parameter $\epsilon$ which controls the granularity of the lattice used. This method of lattice-based quantization gives the following guarantee for communicating a vector between two parties:

\begin{theorem}\label{thm:quantization}
	For any $q=\Omega(1)$, any $\epsilon>0$, and any two parties $u$, $v$ holding input vectors $\bx_u$, $\bx_v \in \real^d$ respectively, there is a quantization method in which $u$ sends $O(d\log q)$ bits to $v$, and if $\|\bx_u-\bx_v \|=O(q\epsilon)$, $v$ can recover an unbiased estimate $\bz$ of $\bx_u$ with $\|\bz-\bx_u\|=O( \epsilon)$.
\end{theorem}

Details are given in Section \ref{sec:quant}. This result is very general, and can be applied not only to DME but to any application in which high-dimensional vectors are communicated, in order to reduce communication.
\subsection{Upper Bounds}\label{sec:alg-short}

Next, we show how to apply our quantization procedure to \textsc{MeanEstimation} and \textsc{VarianceReduction}. We wish to gather (quantized estimates of) machines' inputs to a single machine, which computes the average, and broadcasts a quantized estimate of this average to all other machines. For simplicity of analysis we do so using a star and a binary tree as our communications structures, but any connected communication topology would admit such an approach. The basic star topology algorithm can be described as follows:
\begin{itemize}
\item All machines $u$ send $\bx_u$, quantized with precision parameter $q$ and $\epsilon =\Theta(y/q)$, to the leader machine $v$.
\item Machine $v$ decodes all received vectors, averages them, and broadcasts the result, using the same quantization parameters.
\end{itemize}
By choosing the leader randomly we can obtain tight bounds in expectation on the number of communication bits used per machine, and by using a more balanced communication structure such as a binary tree we can extend these bounds in expectation to hold with certainty. Further details of these algorithmic results can be found in Section \ref{sec:alg}.

\begin{theorem}\label{thm:ME2}
	For any $q=\Omega(1)$, \textsc{MeanEstimation} can be performed with each machine using strictly $O(d\log q)$ communication bits, with $O(\frac{y^2}{q})$ output variance.
\end{theorem}

\begin{theorem}\label{thm:VRshort}
	\textsc{VarianceReduction} can be performed using strictly $O(d\log n)$ bits, with $O(\frac{\sigma^2}{n})$ output variance, succeeding with high probability.
\end{theorem}

These results give optimal communication-variance bounds for these problems. However, to make this bounds practical, we address two main challenges. 

\paragraph{Challenge 1: Input Variance.} One difficulty with our approach is that we assume a known estimate of input variance ($y^2$ for \textsc{MeanEstimation}, $\sigma^2$ for	\textsc{VarianceReduction}). Furthermore, in \textsc{VarianceReduction}, even if our input variance estimate is correct, some pairs of inputs can be further apart, since the bound is probabilistic.

To address this problem, we develop a mechanism for \emph{error detection}, which allows the receiver to detect if the encode and decode vectors ($\bx_u$ and $\bx_v$ respectively) are too far apart for successful decoding. In this way, if our estimate of input variance proved too low, we can increase either it or the number of bits used for quantization until we succeed. Details are presented in Section \ref{sec:error}; the main idea is that rather than using $\bmod q$ to encode lattice points as bit-strings, we use a more sophisticated \emph{coloring} of the lattice and a new encoding procedure to ensure if $\bx_u$ and $\bx_v$ are far apart, with high probability the encoder $v$ chooses a color which is not used by any nearby point to $\bx_u$, and therefore $u$ can tell that $\bx_v$ was not nearby.

As an application, we obtain an algorithm for \textsc{VarianceReduction} which uses an optimal  expected number of bits per machine (except for an additive $\log n$, which in our applications is assumed to be far smaller than $d \log q$):

\begin{theorem}\label{thm:VR2}
	For any $q=\Omega(1)$, \textsc{VarianceReduction} can be performed using $O(d\log q+ \log n)$ communication bits per machine in expectation, with $ O(\frac{\sigma^2}{q}+\frac{\sigma^2}{n})$ output variance, succeeding with high probability.
\end{theorem}

\paragraph{Challenge 2: Computational Tractability.} Another issue is that known lattices which are optimal for $\ell_1$ and $\ell_2$-norms can be computationally prohibitive to generate and use for problems in high dimension. This is discussed further in Section \ref{sec:cubic}. We show that if we instead use the standard \emph{cubic lattice}, which is optimal under $\ell_\infty$-norm and admits straightforward $\tilde O(d)$-computation encoding and decoding algorithms, in combination with a structured random rotation using the Walsh-Hadamard transform as proposed by \cite{MeanEstimation}, we can come within a log-factor variance  of the optimal bounds of Theorems \ref{thm:ME2}, \ref{thm:VRshort}, and \ref{thm:VR2}.

\begin{theorem}\label{thm:cubicresults}
	Using the cubic lattice with a random rotation, we achieve the following output variances under $\ell_2$-norm (succeeding with high probability):
	\begin{itemize}
		\item $O(\frac{y^2\log nd}{q} )$ for \textsc{MeanEstimation}, using strictly $O(d\log q)$ bits;
		\item $O(\frac{\sigma^2\log d}{n} )$ for \textsc{VarianceReduction}, using strictly $O(d\log n)$ bits;
		\item $O(\frac{\sigma^2\log nd}{q} + \frac{\sigma^2}{n})$ for \textsc{VarianceReduction}, using $O(d\log q + \log n)$ bits in expectation.
	\end{itemize}	
Furthermore, each machine need perform only $\tilde O(d)$ computation in expectation.
\end{theorem}

\subsection{Lower Bounds}\label{sec:lb-short}
We next show matching lower bounds on the communication required for \textsc{MeanEstimation} and \textsc{VarianceReduction}. These results bound the number of bits a machine must receive (from any source) to output an estimate of sufficient accuracy, via an information-theoretic argument, and therefore apply to almost any model of communication. Our proofs essentially argue that, if a machine receives only a small amount of bits during the course of an algorithm, it has only a small number of possible (expected) outputs. We can therefore find instances such that our desired output ($\bm\mu$ or $\grad$) is far from any of these possible outputs. This argument is complicated, however, by the probabilistic nature of outputs (and, in the case of the \textsc{VarianceReduction} problem, inputs).

\begin{theorem}\label{thm:LB1b}
	For any \textsc{MeanEstimation} algorithm in which any machine receives at most $b$ bits \emph{in expectation}, \[\Exp{\|\bm{EST}- \mu\|^2}=\Omega(y^2 2^{-\frac {3 b}{d}})\enspace.\]	
\end{theorem}
To achieve an output variance of $O(\frac{y^2}{q})$, we see that machines must receive $\Omega(d\log q)$ bits in expectation, matching the upper bound of Theorem \ref{thm:ME2}. Similarly, we have the following tight bounds for \textsc{VarianceReduction}:
\begin{theorem}\label{thm:LB2b}
	For any \textsc{VarianceReduction} algorithm in which all machines receive (strictly) at most $b$ bits, \[\Exp{\|\bm{EST} - \grad\|^2}= \Omega(\sigma^2 n 2^{-\frac{2b}{d}})\enspace.\] 
\end{theorem}

This bound matches Theorem \ref{thm:VRshort}, since to reduce the variance expression to $O(\frac{\sigma^2}{n})$ (and, in fact, even to $O(\sigma^2)$, i.e., to achieve \emph{any} reduction of output variance compared to input variance), we require $b=\Omega(d\log n)$ bits.

\begin{theorem}\label{thm:LB2a}
	For any \textsc{VarianceReduction} algorithm in which any machine receives at most $b$ bits \emph{in expectation}, \[\Exp{\|\bm{EST} - \grad\|^2}= \Omega(\sigma^2 2^{-\frac{3b}{d}})\enspace.\] 
\end{theorem}
Here we match the leading terms of Theorem \ref{thm:VR2}: to reduce variance to $O(\frac{\sigma^2}{q} )$, we require $\Omega(d\log q)$ bits in expectation. Note that it is well known (and implied by e.g. \cite{braverman2016communication}) that output variance cannot be reduced below $O(\frac{\sigma^2}{n} )$ even with unlimited communication, so Theorem \ref{thm:LB2a} implies that the full variance expression in Theorem \ref{thm:VR2} is tight. These bounds are proven in Section \ref{sec:LB}.

	\section{Quantization Method}\label{sec:quant}

Our quantization scheme will take the following course: to encode a vector $\bx$, we will round it (in an unbiased fashion) to a nearby point in some pre-determined \emph{lattice}. We then map lattice points to short bit-strings in such a way that points with the same bit-string are far apart and cannot be confused. The decoding procedure can then recover the intended lattice point using proximity to a \emph{decoding vector}, in our case the decoding machine's own input. We require the following definitions:

\begin{definition}
	A \emph{lattice} $\Lambda$ in $d$ dimensions is an additively-closed subgroup of $\real^d$, defined by a basis $\bm b_1,\dots,\bm b_d\in \real^d$, and consisting of all integer combinations of the basis vectors.
	
For any lattice $\Lambda\subseteq \real^d$, the \emph{cover radius} $r_{c}$ of $\Lambda$ is the infimum distance $r$ such that $\bigcup_{\bm \lambda \in \Lambda}B_{r}(\bm \lambda) = \real^d$. The \emph{packing radius} $r_{p}$ of $\Lambda$ is the supremum distance $r$ such that for any $\bm \lambda \ne \bm \lambda' \in \Lambda$,  $B_{r}(\bm \lambda) \cap B_{r}(\bm \lambda') = \emptyset$. 
\end{definition}

Based on these notions, can bound the number of lattice points in any ball in $\real^d$.

\begin{lemma}\label{lem:latticeball}
	Let $\Lambda\subset \real^d$ be a lattice, with cover radius $r_{c}$ and packing radius $r_{p}$. Then, for any $\bm x\in \real^d$, $\delta >0$, $\left(\frac{\delta - r_{c}}{r_{c}}\right)^d\le |B_{\delta}(\bm x)\cap \Lambda| \le \left(\frac{\delta + r_{p}}{r_{p}}\right)^d.$
\end{lemma}

\begin{proof}
	We fix a point $\bm x\in \real^d$ and upper-bound the number of points in $\Lambda$ within distance $\delta$ of it: consider the ball $B_{\delta+r_{p}}(\bm x)$. For any point $\bm y$ within distance $\delta$ of $\bm x$, $B_{r_{p}}(\bm y)\subset B_{\delta+r_{p}}(\bm x)$. It is also the case that for any $\bm z\ne \bm y \in \Lambda$, $B_{r_{p}}(\bm z)\cap B_{r_{p}}(\bm y) = \emptyset$. So, 
	
	\[| B_{\delta}(\bm x)\cap \Lambda|\le \frac{Vol(B_{\delta+r_{p}})}{Vol(B_{r_{p}})}\enspace.\]
	
	Under any norm, the ratio $\frac{Vol(B_{r_{1}})}{Vol(B_{r_{2}})}$ of the volumes of two balls is $(\frac{r_1}{r_2})^d$. So:
	
	\[| B_{\delta}(\bm x)\cap \Lambda|\le \left(\frac{\delta + r_p}{r_p}\right)^d\enspace.\]
	
	The lower bound follows similarly: $B_{\delta-r_{c}}(\bm x)\subset \bigcup_{y\in \Lambda\cap B_{\delta}(\bm x)} B_{r_{c}}(\bm y)$, so
	
	\[|B_{\delta}(\bm x)\cap \Lambda| \ge \frac{Vol(B_{\delta-r_{c}})}{Vol(B_{r_{c}})} =\left(\frac{\delta - r_{c}}{r_{c}}\right)^d\enspace. \]
\end{proof}

For our quantization purposes, we want lattices which have low $r_c$ (so that we can always find a close lattice point to quantize to), and high $r_p$ (so that, by Lemma \ref{lem:latticeball}, there are not too many nearby lattice points, and therefore we can specify one with few bits). Ideally we want $r_c = O(r_p)$. Such lattices have long been known to exist:

\begin{theorem}\label{thm:latdist}
	For any $d$, under $\ell_1$, $\ell_2$, or $\ell_\infty$-norm distances, there exists a lattice $\Lambda \subset \real^d$ with $r_c\le 3r_p$.
\end{theorem}

\begin{proof}
	For $\ell_1$ or $\ell_2$-norm (or indeed any $\ell_p$ norm, $p\ge 1$), see \cite{Rogers50}, or Proposition 1.1 of \cite{Henk16}. Under $\ell_\infty$ norm, the standard cubic lattice (i.e. with the standard basis as lattice basis) clearly has this property (in fact, $r_c = r_p$).
\end{proof}

We will call a lattice $\Lambda$ an $\epsilon$-lattice if $\epsilon = r_p\le r_c\le 3\epsilon$. Note that $r_p,r_c$ scale with the lattice basis, and therefore Theorem \ref{thm:latdist} implies that for any of the three norms, an $\epsilon$-lattice exists for any $\epsilon>0$ by scaling appropriately. We will denote such an epsilon lattice $\Lambda_\epsilon$.

\subsection{Lattice Coloring}\label{sec:latticecoloring}
Rounding a vector $\bx\in \real^d$ to a nearby lattice point is not sufficient for it to be communicated using a bounded number of bits, since there are still an infinite number of lattice points (though we have reduced from the uncountably infinite $\real^d$ to the countably infinite $\Lambda$). So, we must encode an infinite subset of lattice points using the same bit-string, and we want to ensure that the points in this subset are far apart, so that a receiver can identify which point the encoder intended. There is a natural way to do this: we simply take a coordinate-wise (with respect to the lattice basis) $\bmod$ operation, thereby mapping the infinite number of points into a finite set of \emph{color classes}.

That is, given a lattice $\Lambda$, and any positive integer $q$, we define a coloring procedure $c_{q}$ as follows: for any lattice point $\bm \lambda$, represent  $\bm \lambda$ as an integer combination of the canonical basis vectors of $\Lambda$, i.e. $\bm \lambda= \alpha_1 \bm b_1 + \dots + \alpha_d \bm b_d$. Then, we set $c_q(\bm \lambda)_i= \alpha_i \bmod  q$, i.e. we obtain $c_q(\bm \lambda)$ by applying the mod operation coordinate-wise to each of the entries $\alpha_i$. We can then encode $c_q(\bm \lambda)$ using $d \log  q$ bits, since there are $q$ possible integer values for each coordinate.

\begin{lemma}\label{lem:latticecolor}
	For any $\epsilon$-lattice $\Lambda_\epsilon$, points $\bm \lambda_1\ne \bm \lambda_2 \in \Lambda_\epsilon$ with $c_q(\bm \lambda_1)=c_q(\bm \lambda_2)$ have $\|\bm \lambda_1-\bm \lambda_2\|\ge 2q\epsilon$.
\end{lemma}

\begin{proof}
	Since $c_q(\bm \lambda_1)=c_q(\bm \lambda_2)$, the vector $\frac 1q(\bm \lambda_1-\bm \lambda_2)$ must have integer coordinates under the canonical basis of $\Lambda_\epsilon$. Therefore it is a point in $\Lambda_\epsilon$. So, $\|\frac 1q(\bm \lambda_1-\bm \lambda_2)\|\ge 2\epsilon$, since otherwise $B_\epsilon(\frac 1q(\bm \lambda_1-\bm \lambda_2))\cap B_\epsilon(0) \ne \emptyset$, which cannot happen since $r_p= \epsilon$. Then $\|\bm \lambda_1-\bm \lambda_2\| \ge 2q\epsilon$.
\end{proof}

\subsection{Encoding Procedure}
We are now ready to define our parameterized quantization procedure $Q_{\epsilon,q}: \real^d \rightarrow \{0,1\}^{d\log q}$ which maps input vectors to subsets of a \emph{lattice}, specified with $d\log q$ bits: let $\Lambda_\epsilon$ be an $\epsilon$-lattice, and let $\bm x$ be a fixed vector to quantize. We show that the convex hull of nearby lattice points contains $\bm x$:

\begin{lemma}\label{lem:convexhull}
	Any $\bm x$ is within the convex hull of $B_{7\epsilon}(\bm x)\cap \Lambda_\epsilon$.
\end{lemma}

\begin{proof}
	We show that for any nonzero vector $\bm w\in \real^d$, there exists $\bm \lambda \in B_{7\epsilon}(\bm x)\cap \Lambda_\epsilon$ such that $\ip{\bm w,\bm \lambda }> \ip{\bm w,\bm x}$; it is well-known that this implies that $\bx$ is in the convex hull of $B_{7\epsilon}(\bm x)\cap \Lambda_\epsilon$. For $\ell_1$ and $\ell_2$ norms, let $\bm \lambda$ be the closest point in $\Lambda_\epsilon$ to $\bm x+4\epsilon \frac {\bm w}{\|\bm w\|}$. Since $r_c\le 3\epsilon$, $\|\bm \lambda - (\bm x+4\epsilon \frac {\bm w}{\|\bm w\|})\|<3\epsilon$ (and so $\bm \lambda \in B_{7\epsilon}(\bm x)$). So,
	
	\begin{align*}
		\ip{\bm w,\bm \lambda } &= \ip{\bm w,\bm x+4\epsilon \frac {\bm w}{\|\bm w\|} +  \bm \lambda - (\bm x+4\epsilon \frac {\bm w}{\|\bm w\|})}\\
		&= \ip{\bm w, \bm x} + \ip{\bm w,4\epsilon \frac {\bm w}{\|\bm w\|}} + \ip{\bm w,\bm \lambda - (\bm x+4\epsilon \frac {\bm w}{\|\bm w\|})}\\	
		&\ge\ip{\bm w,\bm x} + \frac{4\epsilon \|\bm w\|^2}{\|\bm w\|} - \|\bm w\|\cdot\|\bm \lambda - (\bm x+4\epsilon \frac {\bm w}{\|\bm w\|})\|\\
		&>\ip{\bm w,\bm x} + 4\epsilon \|\bm w\| - 3\epsilon\|\bm w\|
		>\ip{\bm w,\bm x}\enspace. \\
	\end{align*}
	
	Here the first inequality uses H\"older's inequality (and, in the case of $\ell_1$ norm, the fact that $\|\bm \lambda - (\bm x+4\epsilon \frac {\bm w}{\|\bm w\|})\|_\infty \le \|\bm \lambda - (\bm x+4\epsilon \frac {\bm w}{\|\bm w\|})\|_1$).
	
	For $\ell_\infty$ norm, let $\bm w'$ be the vector $4\epsilon \cdot sign(\bm w)$, i.e., the vector containing with entries $4\epsilon $ at each positive coordinate of $\bm w$ and $-4\epsilon $ for each negative one. Then, $\|\bm w'\|_\infty=4\epsilon$. Let $\bm \lambda$ be the closest point in $\Lambda_\epsilon$ to $\bx + \bm w'$; as before, we have $\|\bm \lambda-(\bx + \bm w')\|_\infty\le 3\epsilon$ and $\bm\lambda \in B_{7\epsilon}(\bm x)$. Then,
	
	\begin{align*}
		\ip{\bm w,\bm \lambda } &= \ip{\bm w,\bm x+\bm w' +  \bm \lambda - (\bm x+\bm w')}\\
		&= \ip{\bm w, \bm x} + \ip{\bm w,\bm w'} + \ip{\bm w,\bm \lambda - (\bm x+\bm w')}\\	
		&\ge\ip{\bm w,\bm x} +  4\epsilon \|\bm w\|_1  - \|\bm w\|_1\cdot\|\bm \lambda - (\bm x+\bm w')\|_\infty\\
		&>\ip{\bm w,\bm x} + 4\epsilon \|\bm w\|_1 - 3\epsilon\|\bm w\|_1
		>\ip{\bm w,\bm x}\enspace. \\
	\end{align*}
	
	Again, the first inequality uses H\"older's inequality. 
\end{proof}

We can therefore show that we can probabilistically map $\bm x$ to these nearby lattice points in such a way that the expectation of the result is $\bm x$:

Enumerate the points in $B_{7\epsilon}(\bm x)\cap \Lambda_\epsilon$ as $\bm \lambda_1,\dots,\bm  \lambda_z$. Since $\bm x$ is within the convex hull of $B_{7\epsilon}(\bm x)\cap \Lambda_\epsilon$, there must be some sequence of non-negative coefficients $a_1,\dots,a_{z}$ such that $\sum_{i=1}^{z}a_i  \bm  \lambda_i = \bm x$. Let $\bm z$ be a random vector taking value $\bm  \lambda_i$ with probability $\frac{a_i}{\sum_{i=1}^{z}a_i}$, for all $i$. Then, $\Exp{\bm z}=\bm x$, and $\|\bm x-\bm z\| <  7\epsilon$ (and so $\bm z$ is an unbiased estimator of $\bm x$ with variance most $( 7\epsilon)^2\le 49\epsilon^2$).

We then set  $Q_{\epsilon,q}(\bm x)$ to be the color class of $\bm z$ in the coloring $c_q$ of $\Lambda_\epsilon$, which can be specified in $d\log q$ bits. $Q_{\epsilon,q}(\bm x)$ now corresponds to a subset $\Lambda_\epsilon'$ containing $\bm z$, such that any two elements of $\Lambda_\epsilon'$ are of distance at least $2q\epsilon$ apart.

We summarize this process with pseudocode, taking as input the quantization parameters $\epsilon$ and $q$, $\Lambda_\epsilon$, and the vector $\bx\in \real^d$ to quantize.

\begin{algorithm}[H]
	\caption{\textsc{LatticeEncode}, to compute $Q_{\epsilon,q}(\bx)$}
	\label{alg:LE}
	\begin{algorithmic}
		\State Let $\bm  \lambda_1,\dots,\bm  \lambda_z \in B_{7\epsilon}(\bm x)\cap \Lambda_\epsilon$, and $a_1,\dots,a_{z} \ge 0$ such that $\sum_{i=1}^{z}a_i  \bm  \lambda_i = \bm x$.
		\State Let $\bm z$ be a random vector taking value $\bm  \lambda_i$ with probability $\frac{a_i}{\sum_{i=1}^{z}a_i}$, for all $i$.
		\State Output $c_q(\bm z)$, i.e., express $\bm z$ in the canonical basis of $\Lambda_\epsilon$ and take $\bmod \ q$ coordinate-wise.
	\end{algorithmic}
\end{algorithm}

We also summarize the properties of our quantization scheme in the following result:

\begin{lemma}\label{lem:quantprops}
	There is a function $Q_{\epsilon,q}:\real^d \rightarrow \{0,1\}^b$ which maps each $\bx \in \real^d$ to a string $Q_{\epsilon,q}(\bx)$ of $b=d \log q$ bits, specifying a subset $\Lambda_\epsilon'\subset \Lambda_\epsilon$ with the following properties: there exists $\bm z \in \Lambda_\epsilon'$ such that $\bm z$ is an unbiased estimator of \bx with $\|\bm z-\bx\|<7\epsilon$, and for all $\bm w \in \Lambda_\epsilon'\setminus \{\bm z\}$, $\|\bm z - \bm w\| \ge 2q\epsilon$.
\end{lemma}

\subsection{Decoding Procedure}

We must now also define a procedure $R_{\epsilon,q}: \{0,1\}^{d\log q}\times \real^d \rightarrow \real^d $ to decode quantized values, using a machine's own input $\bx_v$ (as the second input to the function): to do so, we simply take the point in the subset $\Lambda_\epsilon'$ encoded by the received quantized value $Q_{\epsilon,q}(\bx)$ which is closest to $\bx_v$:

\begin{algorithm}[H]
	\caption{\textsc{LatticeDecode}, to compute $R_{\epsilon,q}(Q_{\epsilon,q}(\bx),\bx_v)$}
	\label{alg:LD}
	\begin{algorithmic}
		\State Let $\Lambda_\epsilon':= \{\textbf s\in\Lambda_\epsilon : c_q(\textbf s) = Q_{\epsilon,q}(\bm x)\}$, the subset of lattice points matching $Q_{\epsilon,q}(\bm x)$.
		\State Output the closest point in $\Lambda_\epsilon'$ to $\bx_v$.
	\end{algorithmic}
\end{algorithm}

\begin{lemma}\label{lem:decode}
	If $\|\bx - \bx_v\| \le (q-7)\epsilon$, then the decoding procedure $R_{\epsilon,q}(Q_{\epsilon,q}(\bx),\bx_v)$ correctly returns the vector $\bm z$ which is an unbiased estimator of $\bx$.
\end{lemma}

\begin{proof}
	We upper-bound the distance to $\bm z$:
	
	\[\|\bx_v - \textbf z\| \le \|\bx_v - \bx\| + \|\bx - \textbf z\| < (q-7)\epsilon + 7\epsilon = q\epsilon\enspace,\]
	
	and lower-bound the distance to any other point $\textbf y\in \Lambda_\epsilon'\setminus \{\bm z\}$:
	
	\[\|\bx_v - \textbf y\| \ge \|\textbf y - \textbf z\| - \|\bx_v -\textbf z\|  > 2q\epsilon - q\epsilon =q\epsilon \enspace.\]
	
	Therefore $\textbf z$ is the closest point in $\Lambda_\epsilon'$ to $\bx_v$, and will be returned by the decoding procedure.
\end{proof}	

Lemmas \ref{lem:quantprops} and \ref{lem:decode} together then imply our main theorem about the properties of our pairwise encode-decode process, Theorem \ref{thm:quantization}.

\section{Mean Estimation and Variance Reduction Algorithms}\label{sec:alg}
We now apply our quantization procedure to \textsc{MeanEstimation} and \textsc{VarianceReduction}.	We first give a simple algorithm using a star topology, in which machines communicate directly with a randomly chosen \emph{leader}, in order to prove communication bounds in expectation. Then, we show how to use a binary-tree communication structure in order to obtain these same bound with certainty on each machine.

\subsection{Star-Topology Algorithm}\label{sec:staralg}

\begin{algorithm}[H]
	\caption{\textsc{MeanEstimation}}
	\label{alg:ME}
	\begin{algorithmic}
		\State Nominate one machine $v$ at random to be \emph{leader}
		\State All other machines $u$ send $Q_{\epsilon, q}(\bx_u)$ to $v$ (and $v$ simulates sending $Q_{\epsilon, q}(\bx_v)$)
		\State Machine $v$ decodes (using $\bx_v$) and averages all received inputs to obtain $\bm{\hat \mu}$
		\State Machine $v$ sends $Q_{\epsilon, q}(\bm{\hat \mu})$ to all machines
		\State All machines $u$ decode (using $\bx_u$) and output the result
	\end{algorithmic}
\end{algorithm}

\begin{theorem}\label{thm:ME}
	For any $q=\Omega(1)$, Algorithm \ref{alg:ME} performs \textsc{MeanEstimation} using $O(d \log q)$ communication bits per machine in expectation, with $ O(\frac{y^2}{q})$ output variance.
\end{theorem}

\begin{proof}
	Let $q\ge28$, and $\epsilon = \frac{2y}{q}$. We then have $(q-7)\epsilon \ge \frac 34 q\epsilon = 1.5y$. By Lemma \ref{lem:decode}, for any $\bx$, $\bm z$ with $\|\bx - \bm z\|\le 1.5y \le (q-7)\epsilon$, decoding $R_{\epsilon,q}(Q_{\epsilon,q}(\bx),\bm z)$ is successful. Therefore, $v$ correctly decodes all received messages (since all inputs are within distance $y$), obtaining independent unbiased estimates of each machine's input, each within distance $7\epsilon $, and therefore with variance at most $49\epsilon^2$. Then $\hat{\bm \mu}$ is an unbiased estimator of $\bm \mu$ with variance at most $\frac{49\epsilon^2}{n} $, and at distance at most $7\epsilon$ from $\bm \mu$. $\hat{\bm \mu}$ is also at most distance $y+ 7\epsilon \le 1.5y$ from any machine's input, so again by Lemma \ref{lem:decode}, all machines correctly decode the encoded $\hat{\bm \mu}$. The output is therefore is unbiased estimate of $\hat{\bm \mu}$ with variance at most $49\epsilon^2$, i.e., an unbiased estimate of $\bm \mu$ with variance at most $49\epsilon^2+ \frac{49\epsilon^2}{n} \le 74\epsilon^2 = O(\frac{y^2}{q^2})$. We thereby obtain a \textsc{MeanEstimation} algorithm achieving $ O(\frac{y^2}{q^2})$ output variance (and we simplify to $ O(\frac{y^2}{q})$; this does not weaken the result because replacing $q$ with $q^2$ does not asymptotically increase the number of bits required). All nodes except $v$ use using $O(d \log q)$ communication; $v$ uses $O(nd \log q)$. Since $v$ is chosen uniformly at random, this gives the stated expected communication bounds.
\end{proof}

We can perform a simple reduction to extend Algorithm \ref{alg:ME} to \textsc{VarianceReduction}:
\begin{theorem}\label{thm:VR}
	For any $\alpha>1$, and any $q$ between $\Omega(1)$ and $O(n^2\alpha)$, Algorithm \ref{alg:ME} performs \textsc{VarianceReduction} using $O(d \log q)$ communication bits per machine in expectation, with $ O(\frac{\alpha n\sigma^2}{q})$ output variance, succeeding with probability at least $1-\frac 1\alpha$.
\end{theorem}

\begin{proof}
	By Chebyshev's inequality, for any $\alpha>1$, $\Prob{\|\bx_v-\grad\|\ge \sigma \sqrt{\alpha n}}\le \frac{1}{\alpha n}$, and therefore by a union bound,
	\begin{align*}\Prob{\text{All inputs within pairwise distance $ 2\sigma \sqrt{\alpha n}$}}&\ge \Prob{\text{All inputs within distance $ \sigma \sqrt{\alpha n}$ of $\grad$}}\\
		&\ge 1-\frac 1\alpha\enspace.
	\end{align*}
	
	Therefore, we can reduce to a  \textsc{MeanEstimation} instance with $y=2\sigma \sqrt{\alpha n}$, succeeding with probability $1-\frac 1\alpha$. Using $q=O(n^2\alpha)$ (and $\epsilon = \frac{2y}{q} = \frac{4\sigma \sqrt{\alpha n}}{q}$), we obtain an unbiased estimate of $\mu$ with $ O(\frac{\alpha n\sigma^2}{q})$ variance, and therefore an unbiased estimate of $\grad$ with $ O(\frac{\sigma^2}{n}  + \frac{\alpha n\sigma^2}{q}) = O(\frac{\alpha n\sigma^2}{q})$ variance, using $O(d \log q$) expected communication per machine.
\end{proof}

Notably, setting $\alpha=n^c$, for arbitrarily large constant $c$, and $q =\Theta(n^2\alpha)$, Theorem \ref{thm:VR} implies that we can achieve optimal $O(\frac{\sigma^2}{n})$ output variance using $O(d\log n)$ expected communication per machine, and succeeding with high probability in $n$ (i.e., with success probability $1-n^{-c}$).

\subsection{Tree-Topology Algorithm}\label{sec:treealg}
Algorithm \ref{alg:ME} uses a star topology for simplicity (all machines communicate via a randomly chosen leader $v$). This accomplishes mean estimation in only two communication \emph{rounds}, but has the drawback that the bounds on bits communicated are only in expectation, and machine $v$ must perform significantly more communication. In this section, we show that, using a tree topology to more evenly distribute the work, we can obtain the communication bounds of Theorems \ref{thm:ME} and \ref{thm:VR} with certainty for all machines. 

We note that, while we use a binary tree for simplicity, any connected communication graph would suffice. However, we require communication rounds proportional to the diameter of the communication graph. Therefore, our tree-topology algorithm requires $O(\log q)$ rounds (though they need not be synchronous). In systems where round complexity is a major concern, the two-round Algorithm \ref{alg:ME} may be preferable.  

Our algorithm to achieve worst-case bounds (as opposed to bounds in expectation) on communication cost is the following (Algorithm \ref{alg:ME2}). It uses a parameter $m$ which controls both the quantization parameters and the number of input estimates to average; we will set $m$ to achieve a trade-off in the same form as Theorem \ref{thm:ME} later.

\begin{algorithm}[H]
	\caption{\textsc{MeanEstimation}$(m)$}
	\label{alg:ME2}
	\begin{algorithmic}
		\State Sample a set $T$ of $\min{(m,n)}$ machines uniformly at random.
		\State Arrange nodes into a complete binary tree, with nodes in $T$ as leaves.
		\State Collect estimates from $T$ to root of tree, averaging and encoding with $Q_{\frac{y}{m^2}, m^3}$ at every step.
		\State Compute final average at root, and broadcast to all machines (via a binary tree) encoded with $Q_{\frac{y}{m^2}, m^3}$.
		\State Decode and output result at all machines.
	\end{algorithmic}
\end{algorithm}

We now describe in more detail the steps of the algorithm:

\paragraph{Sampling a set of machines.}
We begin by sampling a set $T$ of $\min{(m,n)}$ machines (so if $m\ge n$, $T$ is simply the set $M$ of all machines). Our goal will then be to estimate the average of the inputs $\bx_v, v\in T$, and broadcast this estimate to all machines to output. If $m\ge n$ then clearly $\bm\mu_T:=\frac 1n \sum_{v\in T}\bx_v$ is equal to $\bm\mu=\frac 1n \sum_{v\in M}\bx_v$. Otherwise, it is well-known that the sample mean $\bm\mu_T:=\frac 1m \sum_{v\in T}\bx_v$ is an unbiased estimator of the population mean $\bm\mu=\frac 1n \sum_{v\in M}\bx_v$, with $O(\frac{y^2}{m})$ variance. Therefore, if we can guarantee that all machines output the same value $\bm{EST}$ which is an unbiased estimator of $\bm\mu_T$ with variance $O(\frac{y^2}{m})$, we have correctly performed \textsc{MeanEstimation} with output variance $O(\frac{y^2}{m})$.

\paragraph{Arranging a communication tree.}
Our communication structure will be a complete binary tree, with the machines in $T$ acting as the leaves (since we have choice of $m$ and provide asymptotic bounds, we may assume it is a power of $2$). The roles of the remaining nodes of the communication tree may be taken by any arbitrary machines, so long as all machines take only $O(1)$ roles. We will then send messages \emph{up} the tree (i.e., from child nodes to their parents), in order to convey estimates of the inputs of the leaf nodes (machines in $T$) to the root, which can then compute the final average.

\paragraph{Collecting estimates from leaves to root.}
When a node in the communication tree receives and decodes a vector from both of its children, it takes the average, encodes with $Q_{\frac{y}{m^2}, m^3}$, and sends the result to its parent. We will denote by $\bm A_v$ the average input of all descendant leaves of a tree node $v$. Our goal is then to show that the average computed by node $v$ is an unbiased estimator of $\bm A_v$ with low variance. Since $\bm A_r= \bm\mu_T$ for the root $r$ of the communication tree, we will then have an unbiased estimator of $\bm\mu_T$ (and therefore $\bm\mu$) as desired.

To bound the estimator error at each step, we will employ an inductive argument. Considering nodes by their depth in the tree (with leaves at depth $0$ and the root $r$ at depth $\log \min(m,n)$), we show the following:

\begin{lemma}\label{lem:treeerror}
	A node $v$ at depth $i$ sends to its parent an unbiased estimator $\bm a_v$ of $\bm A_v$ with $\|\bm A_v-\bm a_v\| \le \frac{7iy}{m^2}$, encoded with $Q_{\frac{y}{m^2}, m^3}$.
\end{lemma}

\begin{proof}
	By induction. Clearly the claim is true for leaf nodes at $0$, which encode exactly $\bm A_v$ (i.e. their own input). Assuming the claim is true for $i$, we prove for $i+1$:
	
	Node $v$ receives two values $Q_{\frac{y}{m^2}, m^3}(\bm a_u)$ and $Q_{\frac{y}{m^2}, m^3}(\bm a_w)$ from its children $u$ and $w$. By the inductive assumption, $\|\bm A_u-\bm a_u\| \le \frac{7iy}{m^2}$, and so $\|\bx_v-\bm a_u\| \le \|\bm A_u-\bm a_u\| + \|\bx_v-\bm A_u\| \le  \frac{7iy}{m^2}+y < (m^3-7)\frac{y}{m^2}$ (for $m$ at least a sufficiently large constant). By Lemma \ref{lem:decode}, therefore, $v$ correctly decodes the message from $u$, to recover an unbiased estimator $\bm z_u$ of $\bm a_u$ with $\|\bm z_u-\bm a_u\|< \frac{7y}{m^2}$. The same argument holds for $w$.
	
	Node $v$ then takes as $\bm a_v$ the average of $\bm z_u$ and $\bm z_w$. Since 
	$\|\bm z_u-\bm A_u\|\le \|\bm A_u-\bm a_u\| + \|\bm z_u-\bm a_u\| < \frac{7iy}{m^2}+   \frac{7y}{m^2}= \frac{7(i+1)y}{m^2}$ (and the same holds for $w$):
	
	\[
	\|\bm A_v-\bm a_v\| \le \frac{\|\bm z_u-\bm A_u\| + \|\bm z_w-\bm A_w\|}{2} < \frac{7(i+1)y}{m^2}\enspace.
	\]
	
	This completes the proof by induction.
\end{proof}

\paragraph{Computing the final average at the root.}
By Lemma \ref{lem:treeerror}, the root node $r$, at depth $\log  \min(m,n)$, computes an unbiased estimator $\bm a_r$ of $\bm\mu_T$ with $\|\bm a_r - \bm\mu_T\| \le \frac{7y\log \min(m,n)}{m^2}\le \frac{7y\log m}{m^2}$. It then encodes this vector with $Q_{\frac{y}{m^2}, m^3}$, and broadcasts it to all other machines via an arbitrary binary tree (all machines performing the role of one node in the tree). Nodes in the tree relay the \emph{same} message to their children until all nodes have received the message. Then, all machines decode the message and output the resulting vector.

\paragraph{Decoding and outputting $\bm{EST}$.}

All machines have now received an unbiased estimator $\bm a_r$ of $\bm\mu_T$ with $\|\bm a_r - \bm\mu_T\|\le \frac{7y\log m}{m^2}$, encoded with $Q_{\frac{y}{m^2}, m^3}$. Any machine $v$ has 
\[\|\bm a_r - \bx_v\| \le \|\bm a_r - \bm\mu_T\|+\|\bx_v - \bm\mu_T\|\le \frac{7y\log m}{m^2}+ y < (m^3-7)\frac{y}{m^2}\enspace,\] and so by Lemma \ref{lem:decode}, $v$ correctly decodes the message to recover an unbiased estimator $\bm z_r$ of $\bm a_r$ with $\|\bm z_r-\bm a_r\|< \frac{7y}{m^2}$.

All nodes therefore output $\bm z_r$, which is an unbiased estimator of $\bm \mu_T$, with 
\[\|\bm z_r-\bm\mu_T\| \le \|\bm z_r-\bm a_r\|+ \|\bm\mu_T-\bm a_r\| \le \frac{7y}{m^2}+ \frac{7y\log m}{m^2} = O(\frac ym)\enspace.\]

\begin{proof}[Proof of Theorem \ref{thm:ME2}]
	As noted earlier, an unbiased estimator of $\bm\mu_T$ is also an unbiased estimator of $\bm\mu$ (with $O(\frac{y^2}{m})$ additional variance), and $\bm z_r$ is therefore a correct solution to \textsc{MeanEstimation} with variance $\Var{\bm z_r}\le \|\bm z_r-\bm\mu_T\|^2 + O(\frac{y^2}{m}) = O(\frac{y^2}{m})$. 
	
	We can bound the communication cost as follows: all machines have sent and received $O(1)$ vectors encoded with $Q_{\frac{y}{m^2}, m^3}$, which require $O(d\log m)$ bits. We set $m=q$ to reach an expression of the same form as Theorem \ref{thm:ME}, and obtain Theorem \ref{thm:ME2}.
	
\end{proof}
We can apply the same reduction as for Theorem \ref{thm:VR} to obtain the following for \textsc{VarianceReduction}, using that a \textsc{MeanEstimation} algorithm with $y= 2\sigma\sqrt{\alpha n}$ solves \textsc{VarianceReduction} with probability at least $1-\frac {1}{\alpha}$:

\begin{theorem}\label{thm:VR1}
	For any $\alpha>1$, and any $q$ between $\Omega(1)$ and $O(n^2\alpha)$, Algorithm \ref{alg:ME2} performs \textsc{VarianceReduction} using $O(d \log q)$ communication bits per machine in total, with $ O(\frac{\alpha n\sigma^2}{q})$ output variance, succeeding with probability at least $1-\frac 1\alpha$.
\end{theorem}

As a corollary to this, we can obtain Theorem \ref{thm:VRshort}:

\begin{proof}[Proof of Theorem \ref{thm:VRshort}]
	We set $\alpha=n^c$, for arbitrarily large constant $c$, and $q =\Theta(n^2\alpha)$. Theorem \ref{thm:VR1} then implies that we can achieve optimal $O(\frac{\sigma^2}{n})$ output variance using $O(d\log n)$ total communication per machine, and succeeding with high probability in $n$ (i.e., with success probability $1-n^{-c}$).
\end{proof}
\section{Error Detection In Quantization}\label{sec:error}
In this section we equip our quantization method with a mechanism of \emph{error detection}: that is, we wish to detect when the encode and decode vectors are too far for a successful decoding. This will allow us to better accommodate inputs where some pairs of vectors are much further away than the average distance, and thereby derive better upper bounds on expected communication for \textsc{VarianceReduction}.

We first replace the coloring used in Section \ref{sec:latticecoloring} with a new coloring method, which will allow us to detect this kind of error with high probability. 

\begin{lemma}\label{lem:fancycolor}
	For any $\Lambda_\epsilon$ and any $r$ at least a sufficiently large constant, there exists a coloring $c_r:\Lambda_\epsilon \rightarrow [r^{O(d)}]$ such that for any $q\in \nat$ with $4\le q\le r$ and any $\bm \lambda_1, \bm \lambda_2 \in \Lambda_\epsilon$ with $\|\bm\lambda_1- \bm \lambda_2\|\le 1.5 q^{q^{d}}\epsilon $, there are at most $2q^{d}$ points in $S:= \Lambda_\epsilon\cap (B_{r^3\epsilon }(\bm\lambda_1)\cup B_{r^3\epsilon }(\bm\lambda_2))$ whose colors are not unique in $S$.
\end{lemma}

\begin{proof}
	We show that such a coloring exists by the probabilistic method: we randomly construct a candidate coloring, and prove that it satisfies the necessary properties with positive probability. Then, one such coloring must exist.
	
	Let $f:\left[\left(r^{q^d}\right)^d\right]\rightarrow [r^{10d}]$ be a uniformly random function. Our candidate coloring $c$ is then given by $c(\bm\lambda)= f(\bm\lambda \bmod r^{q^d})$ (where $\bmod$ is taken coordinate-wise with respect to the canonical basis of the lattice $\Lambda_\epsilon$). So, in effect, we divide the lattice points into $r^{dq^d}$  \emph{classes} by taking $\bmod\ r^{q^d}$ , and then randomly color the classes. Any two points $\bm  \lambda,  \bm \lambda'$  in the same class are of distance at least $2\cdot  r^{q^d}\epsilon$ apart, since $r^{-q^d}(\bm  \lambda- \bm  \lambda')$ is a non-zero lattice vector, and so is of length at least $2\epsilon$. 
	
	Fixing some pair of points $\bm \lambda_1, \bm \lambda_2 \in \Lambda_\epsilon$ (and therefore some set $S= \Lambda_\epsilon\cap (B_{r^3\epsilon }(\bm\lambda_1)\cup B_{r^3\epsilon }(\bm\lambda_2))$), we note that all pairs of points in $S$ are at most distance $1.5q^{q^d}\epsilon+2r^3\epsilon < 2\cdot r^{q^d}\epsilon$ apart, so they are all in different classes, and therefore their colors were chosen independently uniformly at random.
	
	By Lemma \ref{lem:latticeball}, in any ball of radius $r^3\epsilon$, there are at most $\left(\frac{r^3\epsilon + \epsilon}{\epsilon}\right)^d = \left(r^3+1\right)^d $ lattice points. Therefore, $|S|\le 2\left(r^3+1\right)^{d}  $. Our goal now is to show that few points in $S$ are given non-unique colors.
	
	Order the points in $S$ arbitrarily. If there are more than $2q^{d}$ colors used by multiple points in $S$, then there are at least $q^d$ points assigned a color already used by another point prior in the ordering. For any fixed set $S'$ of $q^{d}$ points in $S$, the probability that all receive a color already used by another point prior in the ordering is at most $\left(\frac{2\left(r^3+1\right)^{d}}{r^{10d}} \right)^{q^d}$, since the colors are chosen uniformly and independently, and at most $2\left(r^3+1\right)^{d}$ are used by prior points.
	
	So,
	
	\begin{align*}
		\Prob{|\{\text{points with non-unique colors in }S\}| >2q^d}&\le \sum\limits_{\substack{S' \subseteq S\\|S'|=q^d}} \left(\frac{2\left(r^3+1\right)^{d}}{r^{10d}} \right)^{q^d}
		\\& = \binom{2\left(r^3+1\right)^{d} }{q^d}\left(\frac{2\left(r^3+1\right)^{d}}{r^{10d}} \right)^{q^d}
		\\&\le \left(\frac{2e\left(r^3+1\right)^{d} }{q^d}\right)^{q^d}\left(\frac{2\left(r^3+1\right)^{d}}{r^{10d}} \right)^{q^d}
		\\&< \left(\frac{r^{7d}}{r^{10d}} \right)^{q^d}
	= r^{-3dq^{d}}\enspace.
	\end{align*}
	
	We take a union bound over all pairs of classes for $\bm\lambda_1,\bm\lambda_2$; this incorporates all choices for $\bm\lambda_1,\bm\lambda_2$ since the coloring is identical with respect to two pairs of points from the same pair of classes. There are fewer than $ r^{2dq^{d}}$ such pairs, so the probability that any do not satisfy the condition is at most $r^{-dq^{d}}$. Finally, we take a union bound over all $q$ with $4\le q\le r$, and see that the probability that any such $q$ does not satisfy the condition is less than $\sum_{q=4}^r r^{-dq^{d}} < 1$. Therefore, we have a positive probability that $c$ satisfies the criteria of the lemma, so by the probabilistic method, such a good coloring must exist.
	
\end{proof}

We utilize this coloring in the following way: when encoding a vector $\bx_u$ with a lattice point $\bz$, we can now ensure that for any decode vector $\bx_v$, with high probability, the \emph{color} of $\bz$ is unique among points within $r^3\epsilon$ of $\bx_u$ and $\bx_v$. Therefore, when machine $v$ decodes, it either recovers $\bz$ or knows that $\|\bx_v - \bz\|> r^3\epsilon$.

Now that we have means of detecting whether the encode and decode vectors are too far apart for successful decoding, we can design a process that iteratively attempts to quantize using more bits until decoding succeeds:

\begin{algorithm}[H]
	\caption{\textsc{RobustAgreement}$(\epsilon,q)$}
	\label{alg:RA}
	\begin{algorithmic}
		
		\State Let $\hat\bx_u$ be chosen uniformly at random from $B_{q^2\epsilon}(\bx_u)$.
		\State Map $\hat\bx_u$ to a lattice point $\bm z\in B_{7\epsilon}(\hat\bx)\cap \Lambda_\epsilon$ in unbiased fashion, as in Lemma \ref{lem:quantprops}.
		\State $r\gets q$
		\Loop
		\State Machine $u$ sends $c_r( \bm z)$ to $v$
		\State Machine $v$ computes $\bm y\in\Lambda_\epsilon$, the closest lattice point to $\bx_v$ such that $c_r( \bm y) =c_r( \bm z)$
		\If{$\|\bx_v-\by\|\le \frac 12 r^3\epsilon$}
		\State $v$ outputs $\by$, procedure terminates
		\Else
		\State $v$ sends the message \textsc{Far} to $u$
		\State $r \gets r^2$
		\EndIf
		\EndLoop
	\end{algorithmic}
\end{algorithm}

Here the purpose of first choosing a vector $\hat\bx_u$ uniformly from $B_{q^2\epsilon}(\bx_u)$ is to ensure that the probability that $\bz$ takes any particular value is low, since there are a small number of points which do \emph{not} have non-unique colors which we wish to avoid with high probability.

\begin{lemma}\label{lem:robustz}
	Lattice point $\bm z$ is an unbiased estimator of $\bx_u$, with $\|\bz-\bx_u\|\le (q^2 + 7)\epsilon$. Furthermore, the probability of $\bz$ taking any particular value $\hat \bz$ is at most $(\frac{7}{q^2})^d$.
\end{lemma}

\begin{proof}
	By our choice of $\bz$ we clearly have $\|\bz-\bx_u\|\le (q^2 + 7)\epsilon$; unbiasedness follows from Lemma \ref{lem:quantprops}. For $\bz$ to take some particular value $\hat \bz$, we must have $\Prob{\|\hat \bz-\hat\bx\|\le 7\epsilon}$, which is the case with probability at most $\frac{Vol(B_{7\epsilon})}{Vol(B_{q^2\epsilon})}= (\frac{7}{q^2})^d$.
\end{proof}

We next analyze what the properties of the communication procedure within the loop of Algorithm \ref{alg:RA}, for some fixed $r$:
\begin{lemma}\label{lem:robustloop}
	During a loop of \textsc{RobustAgreement}:
	\begin{itemize}
		\item $c_r(\bz)$ takes $O(d\log r)$ bits to send.
		\item If $\|\bx_u-\bx_v\|\le \frac 12 r^3\epsilon$, with probability at least $1-(\frac{14}{q})^d$, $v$ outputs $\bz$.
		\item If $\|\bx_u-\bx_v\|\le  q^{q^d}\epsilon$, with probability at least $1-(\frac{14}{q})^d$, $v$ either outputs $\bz$ or sends \textsc{Far} to $u$.
	\end{itemize}
\end{lemma}

\begin{proof}
	$c_r(\bz)$ is a color from $[r^{O(d)}]$, which takes $O(d\log r)$ bits to specify.

	Let $\bm\lambda_u$ be the closest lattice point to $\bx_u$, and $\bm\lambda_v$ be the closest lattice point to $\bx_v$. By the properties of an $\epsilon$-lattice, $\|\bm\lambda_u-\bx_u\|$ and $\|\bm\lambda_v-\bx_v\|$ are at most $3\epsilon$. 
	
	If $\|\bx_u-\bx_v\|\le  q^{q^d}\epsilon$, then $\|\bm \lambda_u-\bm \lambda_v\|\le q^{q^d}\epsilon+6\epsilon\le1.5 q^{q^d}$. So, by Lemma \ref{lem:fancycolor}, there are at most $2q^{d}$ points with non-unique colors in $S= \Lambda_\epsilon\cap (B_{r^3\epsilon }(\bm\lambda_u)\cup B_{r^3\epsilon }(\bm\lambda_v))$. We have $\bz\in \Lambda_\epsilon\cap (B_{r^3\epsilon }(\bm\lambda_u))\subset S$, so the probability of its color being unique in $S$ is at least $1-2q^{d}\cdot (\frac{7}{q^2})^d = 1-(\frac{14}{q})^d$, by Lemma \ref{lem:robustz}.
	
	In this case, $\by$ is either \bz, or is outside $B_{r^3\epsilon(\bm\lambda_v)}$. If $\|\bx_u-\bx_v\|\le \frac 12 r^3\epsilon$, then $\|\by-\bm\lambda_v\|\le r^3\epsilon$, so we must have $\by=\bz$, and $v$ correctly outputs it. Otherwise, we may have $\by\notin B_{r^3\epsilon(\lambda_v)}$, but if so, $v$ sends \textsc{Far} to $u$, which is also permitted.

\end{proof}

We can then show the properties of the algorithm as a whole:

\begin{lemma}\label{lem:ErrorCorrect}
	If $\|\bx_u-\bx_v\|\le  q^{q^d}\epsilon$, with probability at least $1-\log\log (\frac 1\epsilon \|\bx_u-\bx_v\|)\cdot O(q^{-d} )$, Algorithm \ref{alg:RA} provides machine $v$ with an unbiased estimate $\bz$ of $\bx_u$, with $\|\bz-\bx_u\|\le (q^2+7)\epsilon$, and uses $O\left(d\log (\frac q\epsilon \|\bx_u-\bx_v\|)\right)$ bits.
\end{lemma}

\begin{proof}
	After performing $i$ iterations we have $r = q^{2^i}$.
	Therefore, after iteration $2\log\log( \|\bx_u-\bx_v\|/\epsilon) $ we have $\|\bx_u-\bx_v\|\le \frac 12 r^3\epsilon $. Then, so long as this iteration, and all previous iterations, are successful, $v$ outputs an estimate $\bz$ of $\bx_u$ with $\|\bz-\bx_u\|\le (q^2 + 7)\epsilon$, by Lemma \ref{lem:robustloop}.
	
	Since we have $O(\log\log (\|\bx_u-\bx_v\|/\epsilon) )$ iterations, and each fails with probability at most $\Omega(q)^{-d}$, our total failure probability is $\log\log ( \|\bx_u-\bx_v\|/\epsilon)\cdot O(q^{-d} )$ by a union bound over all iterations. The number of bits used per iteration is $O(d \log r)$; since $\log r$ doubles each iteration, the total is at most twice the bound for the final iteration. If this is the first iteration, the bound is $O(d\log q)$; otherwise it is $O(d\log(\frac{1}{\epsilon} \|\bx_u-\bx_v\|))$. So, the total number of bits used is $O(d\log(\frac{q}{\epsilon} \|\bx_u-\bx_v\|))$.
\end{proof}

\subsection{Application to \textsc{VarianceReduction}}
We now apply our error detection in order to improve our bounds on expected communication for performing \textsc{VarianceReduction}. The idea is that, since we only have a probabilistic bound on distance between inputs, some pairs of inputs may be significantly further away: we can now detect these cases using error detection, and increase the number of bits used for quantization accordingly.

We use the star-topology algorithm (Algorithm \ref{alg:ME}), and simply replace each pairwise encoding and decoding interaction with \textsc{RobustAgreement}$(\epsilon,q)$:

\begin{algorithm}[H]
	\caption{\textsc{VarianceReduction}}
	\label{alg:VRStar}
	\begin{algorithmic}
		\State Nominate one machine $v$ uniformly at random to be \emph{leader}
		\State All other machines $u$ perform \textsc{RobustAgreement}$(\epsilon,q)$ to send $\bx_u$ to $v$
		\State Machine $v$ averages all received estimates to obtain $\hat \grad$
		\State Machine $v$ performs \textsc{RobustAgreement}$(\epsilon,q)$ with each other machine to send $\hat  \grad$, 
		\State \hspace{1em} using the same choice of $\bz$
		\State All machines output the resulting estimate
	\end{algorithmic}
\end{algorithm}

We are now ready to prove our main result applying error detection to \textsc{VarianceReduction}, Theorem \ref{thm:VR2}.

\begin{proof}[Proof of Theorem \ref{thm:VR2}]
	
	We set $\mathfrak q$ to be $q+ cn^{\frac 2d}$ for some sufficiently large constant $c$, and run Algorithm \ref{alg:VRStar} with $\mathfrak q$ and $\epsilon=\sigma/(\mathfrak q^4)$. We first consider the distances between encode and decode vectors in the first stage of messages, where all other machines $u$ send their input to the leader $v$ using \textsc{RobustAgreement}. In each case, $\Exp{\|\bm x_v-\bx_u\|}\le\Exp{\|\bm x_v-\grad\|+\|\bm x_u-\grad\|} \le 2\sigma$. So long as we succeed (as specified by Lemma \ref{lem:ErrorCorrect}), node $v$ then receives the estimate $\tilde\bx $ of $\bx_u$, with $\|\tilde\bx-\bx_u\|\le (\mathfrak q^2+7)\epsilon < \sigma/\mathfrak q$ (since $\mathfrak q$ is chosen to be at least 3), and uses $O\left(d\log (\frac {\mathfrak q}{\epsilon} \|\bx_u-\bx_v\|)\right)$ bits. We have $\Exp{d\log (\frac {\mathfrak q}{\epsilon} \|\bx_u-\bx_v\|)} = O(d\log \mathfrak q) = O(d\log q+\log n)$.
	
	The leader node $v$ now averages these estimates to obtain $\hat\grad$. Each of the received estimates is an independent unbiased estimator of $\grad$ with variance $O(\sigma^2)$ and therefore $\hat\grad$ is an unbiased estimator of $\grad$ with variance $O(\sigma^2/n)$.
	
	In the second stage of messages, $v$ sends $\hat\grad$ to all other nodes using \textsc{RobustAgreement}, taking the same choice of $\bz$ in each, in order to ensure that machines all have the same output. $z$ is an unbiased estimate of $\grad$ with variance $O(\frac{\sigma^2}{n}+ \frac{\sigma^2}{\mathfrak q}) =  O(\frac{\sigma^2}{\min n,q})$. Again, for each such node $u$ we have $\Exp{\|\bx_u-\hat\grad\|}=O(\sigma)$, so we use $O(d\log q+\log n)$ bits in expectation. Therefore each node uses $O(d\log q+\log n)$ total bits in expectation, except the leader $v$ which uses $O(nd\log q+n\log n)$. Since $v$ is chosen uniformly at random, we have an $O(d\log q+\log n)$-bit bound on expected communication for each machine.
	
	In all applications of \textsc{RobustAgreement} we perform, denoting the encode and decode vectors $\bm a$, $\bm b$ respectively (i.e., in the first stage $\bm a = \bx_u$ and $\bm b = \bx_v$ for each $u$, and in the second $\bm a = \hat\grad$ and $\bm b = \bx_u$), we have $\Exp{\|\bm a-\bm b\|^2} \le (2\sigma)^2$ and so $\Prob{\|\bm a-\bm b\|\le 4n\sigma}\le \frac{1}{4n^2}$. Since we perform \textsc{RobustAgreement} $2n$ times, with probability at least $1-\frac{1}{ 2n}$, $\|\bm a-\bm b\|\le 4n\sigma$ in all cases by a union bound. So, with high probability we always satisfy the condition $\|\bm a-\bm b\|\le  \mathfrak q^{\mathfrak q^d}\epsilon = \mathfrak q^{\mathfrak q^d-4}\sigma$, and have failure probability at most $\log\log (\frac 1\epsilon \|\bm a-\bm b\|)\cdot O(\mathfrak q^{-d} ) =O(\mathfrak q^{-d} \log\log(n\mathfrak q))$ for each application of \textsc{RobustAgreement}. So, the probability that all applications succeed is at least $(1-\frac{1}{2n})(1-O(\mathfrak q^{-d} \log\log(n\mathfrak q))) \ge 1-\frac 1n$.
	
	Then, all machines correctly receive $\bz$, an estimate of \grad with variance at most $ O(\frac{\sigma^2}{\min n,q})$.
\end{proof}

\section{$\ell_2$ Guarantees for the Cubic Lattice}\label{sec:cubic}
One key question when applying our scheme in practice is \textbf{the choice of lattice}. Ideally this would depend on the norm under which the results will be evaluated. However, asymptotically optimal lattices for $\ell_1$ and $\ell_2$ norms can be computationally expensive to generate and use in practice. For this reason, in our practical implementations we will employ the standard \textbf{cubic lattice}: this gives us computationally-efficient ($\tilde O(d)$-time) quantization algorithms, is optimal under $\ell_\infty$ norm, and, as we will see, performs surprisingly well even evaluated under $\ell_1$ or $\ell_2$ norm. We will investigate more complex lattices, tailored to $\ell_1$ or $\ell_2$ norm, in future work. There are two possible approaches to doing so: one would be to apply \emph{general} lattice algorithms (e.g. \cite{laarhoven2016sieving}) in applications where $d$ is still fairly low. This is often the case in, for example, gradient descent for neural network training, where coordinates are already divided into fairly small \emph{buckets}. The second possible approach would be to find \emph{specific} lattices which admit more efficient algorithms, and also have a good $r_c/r_p$ ratio under $\ell_1$ or $\ell_2$ norm.

Using the cubic lattice, though, need not sacrifice too much by way of theoretical guarantees under $\ell_2$-norm, since, as noted in \cite{MeanEstimation}, a random rotation using the Walsh-Hadamard transform can ensure good bounds on the ratio between $\ell_2$ norm and $\ell_\infty$ norm of vectors.

Let $H$ be the $d\times d$ normalized Hadamard matrix $H_{i,j} = d^{-1/2} (-1)^{\ip{i-1,j-1}}$, where $\ip{i,j}$ is the dot-product of the $\log_2 d$-dimension $\{0,1\}$-valued vectors given by $i$, $j$ expressed in binary (we must assume here that $d$ is a power of two, but this does not affect asymptotic results). We use the following well-known properties of $H$ (see e.g. \cite{Hadamard}):

\begin{itemize}
	\item $H$ is orthonormal, and so preserves $\ell_2$-norm distances;
	\item $H^{-1}=H$.
\end{itemize}

Let $D$ be a $d\times d$ diagonal matrix where each $D_{i,i}$ is drawn uniformly at random from $\{-1,1\}$. $H$ is a fixed, known matrix, and $D$ costs $d$ bits to communicate, so we can assume that both matrices are known to all machines at a cost of only $O(d)$ bits of communication per machine. 

Before applying our \textsc{MeanEstimation} or \textsc{VarianceReduction} algorithms, we apply the transformation $HD$ to all inputs - we then invert the transform (i.e. apply $(HD)^{-1}=D^{-1}H$) before final output. As shown in \cite{FJLT}, both the forward and inverse transform require only $O(d\log d)$ computation.

\begin{lemma}\label{lem:rotation}
	For any set $S$ of $n^2$ vectors in $\real_d$, with probability at least $1-2d^{-1}$, all vectors $\bm x\in S$ have 
	\[\|HD\bm x\|_\infty  = O(d^{-1/2}\|\bm x\|_2\sqrt{\log nd})\]
\end{lemma}

\begin{proof}
	We follow a similar argument to \cite {FJLT}. Fix some vector $\bm x\in S$ and some coordinate $j \in [d]$. Notice that $(HD\bx)_j= \sum_{i=1}^{d} \alpha_i \bm x_i$, where each $\alpha_i = \pm d^{-1/2}$ is chosen independently and uniformly. By a Chernoff-type bound (c.f. \cite {FJLT}), we therefore obtain $\Prob{|(HD\bm x)_j| \ge s\|\bm x\|_2}\le 2e^{-s^2d/2}$. Plugging in $s=2\sqrt{\frac{\ln nd}{d}}$, we get:
	
	\[\Prob{|(HD\bm x)_j| \ge 2\sqrt{\frac{\ln nd}{d}}\|\bm x\|_2}\le 2e^{-2\ln nd} = 2(nd)^{-2}\enspace.\]
	
	We then take a union bound over all $\bm x\in S$ and $j \in [d]$ to find that with probability at least $1-2d^{-1}$, all $\bm x\in S$ have $\|HD\bm x\|_\infty  = O(d^{-1/2}\|\bm x\|_2\sqrt{\log nd})$.
\end{proof}

We can now obtain analogue of Theorem \ref{thm:ME2} using the cubic lattice:

\begin{theorem}\label{thm:MEcubic}
	For any $q=\Omega(1)$, \textsc{MeanEstimation} can be performed using the cubic lattice with each machine using $O(d\log q)$ communication bits in total, with $O(\frac{y^2\log nd}{q})$ output variance under $\ell_2$ norm, succeeding with probability at least $1-2d^{-1}$.
\end{theorem}

\begin{proof}
	We apply the random rotation $HD$ as in Lemma \ref{lem:rotation}, and then proceed as usual using Algorithm \ref{alg:ME2}, applying the inverse rotation to the final output. Let $y_2$ denote the usual maximum input distance parameter under $\ell_2$-norm. We first show that after applying the random rotation we have $y_\infty = O(d^{-1/2}y_2 \sqrt{\log nd})$:
	
	There are $n$ machines $v$ with input vectors $\bm x_v \in \real^d$. We form a set $S=\{\bm x_v -\bm x_u \text{ for all machines } u,v\}$. This set then has at most $n^2$ elements. Applying the rotation $HD$ to all input vectors implicitly applies it to this set $S$ of \emph{differences}, since $HD$ is a linear transformation. Therefore, by Lemma \ref{lem:rotation}, we have
	\[\max_{\bm z \in S}\|HD\bm z\|_\infty = O(d^{-1/2}\max_{\bm z \in S}\|\bm z\|_2 \sqrt{\log nd}) = O(d^{-1/2}y_2 \sqrt{\log nd}),\]
	
	so we can set $y_\infty = O(d^{-1/2}y_2 \sqrt{\log nd})$ and satisfy that $\|HD(\bm x_v -\bm x_u)\|_\infty \le y_\infty$ for all machines $u,v$, with probability at least $1-2d^{-1}$.
	
	Now we apply Algorithm \ref{alg:ME2} using the cubic lattice. Since the cubic lattice is optimal under $\ell_\infty$ norm, we obtain an output variance (under $\ell_\infty$ norm) of $O(\frac{y_\infty^2}{q}) = O(\frac{d^{-1}y_2^2 \log nd}{q})$ when using $O(d\log q)$ bits per machine, by Theorem \ref{thm:ME2} (but our output $HD\bm{EST}$ is currently an unbiased estimator of the rotated mean $HD\bm{\mu}$ rather than $\bm\mu$ itself). Since, for any vector $\bm z \in \real^d$, $\|\bm z\|^2_2 \le d\|\bm z\|^2_\infty$, when applied with $\bm z= HD(\bm{EST}-\bm \mu)$ we get an output variance under $\ell_2$-norm of $O(\frac{y_2^2 \log nd}{q})$
	
	When applying the final inverse rotation to the output vector $HD\bm{EST}$, we also implicitly imply it to the error vector $HD(\bm{EST}- \bm \mu)$. This inverse rotation preserves distances under $\ell_2$-norm, so we still have $O(\frac{y_2^2 \log nd}{q})$ $\ell_2$-norm variance, now of an unbiased estimator $\bm{EST}$ of the mean $\bm\mu$ of our original inputs (unbiasedness is preserved by linearity of expectation, since $HD$ and $(HD)^{-1}$ are linear transformations).
\end{proof}

By applying the same analysis to Theorems \ref{thm:VRshort} and \ref{thm:VR2}, we obtain the \textsc{VarianceReduction} results of Theorem \ref{thm:cubicresults}.

\begin{proof}[Proof of Theorem \ref{thm:cubicresults}]
	We perform the same procedure as in Theorem \ref{thm:MEcubic}, applying the random rotation $HD$ to obtain a bound on $\ell_\infty$-norm, applying Algorithms \ref{alg:ME} and \ref{alg:ME2} respectively, using the $\ell_infty$-norm-optimal cubic lattice, and then applying the inverse rotation before output to obtain an $\ell_2$-norm output variance bound at only a $O(\log nd)$ factor higher than optimal.
	
	Specifically, we can obtain unbiased estimates of input mean $\bm\mu$ with variance $O(\frac{\sigma^2\log nd}{n^2})$ using strictly $O(d\log n)$ bits (with sufficiently high constant within the asymptotic notation), and $O(\frac{\sigma^2\log nd}{q})$ using $O(d\log q + \log n)$ bits in expectation. Since $\bm\mu$ is itself an $O(\frac{\sigma^2}{n})$-variance unbiased estimator of our true vector \grad, we obtain final output variances of $O(\frac{\sigma^2\log nd}{n^2}+ \frac{\sigma^2}{n}) = O(\frac{\sigma^2 \log d}{n})$ and $O(\frac{\sigma^2\log nd}{q}+\frac{\sigma^2}{n})$ respectively.
\end{proof}
\out{
	\begin{theorem}\label{thm:VRcubic}
		For any $\alpha>1$, and any $q$ between $\Omega(1)$ and $O(n^2\alpha)$, Algorithm \ref{alg:ME2} using the cubic lattice performs \textsc{VarianceReduction} using $O(d \log q)$ communication bits per machine in total, with $ O(\frac{\alpha n\sigma^2\log nd}{q})$ output variance under $\ell_2$ norm, succeeding with probability at least $1-\frac 1\alpha-2d^{-1}$.
	\end{theorem}
	
	\begin{proof}
		The proof follows in the same way as that of Theorem \ref{thm:MEcubic}: we apply the rotation $HD$, using Lemma \ref{lem:rotation} to show that we can set $\sigma_\infty  = O(d^{-1/2}\sigma_2 \sqrt{\log nd})$ successfully with probability at least $1-2d^{-1}$. We apply Algorithm \ref{alg:ME2} and obtain (with probability at least $1-\frac 1\alpha-2d^{-1}$) an $O(\frac{d^{-1}\alpha n\sigma^2\log nd}{q})$ $\ell_\infty$-norm variance unbiased estimator of $HD\grad$, which is then an $O(\frac{\alpha n\sigma^2\log nd}{q})$ $\ell_2$-norm variance estimator. Finally, we again note that since $D^{-1}H$ is again a linear transformation which preserves $\ell_2$-norm distances, our final output is an unbiased estimator of $\grad$ with $O(\frac{\alpha n\sigma^2\log nd}{q})$ $\ell_2$-norm variance.
\end{proof}}

\section{Sublinear Communication}\label{sec:sublinear}

In this section we show an extension to the quantization scheme (and thereby also Algorithm \ref{alg:ME}) allowing us to use a \emph{sublinear} (in $d$) number of bits in expectation, providing a variance trade-off for the \emph{full range} of communication for both \textsc{MeanEstimation} and \textsc{VarianceReduction}. Here our methods apply only to $\ell_2$-norm.

\begin{theorem}\label{thm:sublinear}
	For any $q>0$, \textsc{MeanEstimation} can be performed using $O(d \log (1+q))$ communication bits per machine in expectation, with $ O(\frac{y^2}{q^2})$ output variance under $\ell_2$ norm.
\end{theorem}

This communication expression matches the existing bounds of Theorem \ref{thm:ME}, since when $q=\Omega(1)$, it simplifies to $O(d\log q)$. However, the method described in this section now works also for $q=o(1)$; here, the expression simplifies to $O(dq)$ bits, which is sublinear in $d$. We can further show good concentration on the amount of communication bits required, and extend to \textsc{VarianceReduction} as before.

Our sublinear-communication quantization method will involve mapping points to their \emph{closest} lattice point (rather than randomly to a set of near by points who form a convex hull around them). We will therefore be concerned with the (open) Voronoi regions of lattice points, since these are the regions that will be quantized to the same point:

\begin{definition}
	For any lattice $\Lambda$, the Voronoi region of a lattice point $\bm\lambda$ is the set of all $\bx \in \real^d$ to whom $\bm\lambda$ is closer than any other lattice point, i.e., 
	$\bm{Vor}(\bm\lambda):= \{\bx\in \real^d:  \|\bx-\bm\lambda\|_2<\|\bx-\bm\lambda'\|_2, \forall \bm\lambda' \in \Lambda \setminus\{\bm\lambda\}\}$.
\end{definition}

Voronoi regions of a lattice are known to have several useful properties, of which we will use the following:

\begin{lemma}
	The Voronoi regions of a lattice are:
	\begin{enumerate}
		\item open convex polytopes,
		\item symmetric, i.e., $\bm \lambda+ \bx \in \bm{Vor}(\bm\lambda) \iff \bm\lambda- \bx\in \bm{Vor}(\bm\lambda)$, and
		\item identical, i.e., $\bm\lambda+ \bx \in \bm{Vor}(\bm\lambda) \iff\bm\lambda'+ \bx \in \bm{Vor}(\bm\lambda')$ for any $\bm\lambda'\in \Lambda$.
	\end{enumerate}
\end{lemma}
\begin{proof}
	Property 1 follows since $\bm{Vor}(\bm\lambda)$ is the intersection of half-spaces formed by the set of points which are closer to $\bm\lambda$ than $\bm\lambda'$, for any $\bm\lambda' \in \Lambda \setminus\{\bm\lambda\}\}$. 
	
	Property 2 follows if there exists $\bm\lambda'\in \Lambda$ such that $\|\bm\lambda+ \bx - \bm\lambda'\|_2\le \|\bx\|_2$, then $\hat{\bm\lambda} = 2\bm\lambda-\bm\lambda'$ is also a lattice point in $\Lambda$, and 
	\[\|\bm\lambda - \bx - \hat{\bm\lambda}\|_2 = \|\bm\lambda - \bx - (2\bm\lambda+\bm\lambda')\|_2 = \|-(\bm\lambda+ \bx - \bm\lambda')\|_2 \le \|\bx\|_2\enspace.\]
	Therefore $\bm\lambda+ \bx\notin \bm{Vor}(\bm\lambda)\implies \bm\lambda- \bx\notin \bm{Vor}(\bm\lambda)$, which by contradiction proves the property.
	
	Property 3 follows since the relative positions of all other lattice points are identical with respect to $\bm\lambda$ and $\bm\lambda'$.
\end{proof}

In particular, Property 3 implies that all Voronoi regions of a lattice have the same volume, which we denote $\bm{V}$. We note that, within any finite ball, the set of points which do \emph{not} fall in a Voronoi region (i.e. have multiple equidistant closest lattice points) has measure $0$ (under the standard Lebesgue measure of $\real^d$), and therefore do not affect the probabilities of events under this measure.

We now wish to bound the number of Voronoi regions which are close to any point in $\real^d$, in expectation; these will be the regions of lattice points which are sufficiently close to cause decoding errors if they are encoded with the same bit-string. For this we introduce the concept of an \emph{expanded} Voronoi region:

\begin{definition}
	The \emph{expanded Voronoi region} $\bm{Vor}^+(\bm\lambda)$ of a lattice point $\bm\lambda$ is the set $\bm{Vor}^+(\bm\lambda)$ of points within distance $2q\epsilon$ of $\bm{Vor}(\bm\lambda)$.
\end{definition}

We bound the volume of such sets:

\begin{lemma}\label{lem:vorvol}
	For any $\bm\lambda$ in $\Lambda_\epsilon$, the expanded Voronoi region $\bm{Vor}^+(\bm\lambda)$ has volume at most $(1+2q)^d \bm{V}$.
\end{lemma}

\begin{proof}
	Let $\widehat{\bm{Vor}}(\bm\lambda)$ be $\bm{Vor}(\bm\lambda)$ dilated by a factor of $1+2q$ around $\bm\lambda$; that is:
	\[\widehat{\bm{Vor}}(\bm\lambda):= \{\lambda + (1+2q) \bx:  \lambda +\bx \in  \bm{Vor}(\bm\lambda) \}\enspace.\]
	
	Clearly the volume of $\widehat{\bm{Vor}}(\bm\lambda)$ is $(1+2q)^d \bm{V}$; we now show that $\bm{Vor}^+(\bm\lambda)\subseteq \widehat{\bm{Vor}}(\bm\lambda)$. $\bm{Vor}(\bm\lambda)$ is the intersection of open half-spaces bounded by the hyperplanes of points equidistant from $\bm\lambda$ and $\bm\lambda'$, for any $\bm\lambda' \in \Lambda_\epsilon \setminus\{\bm\lambda\}$. Dilation by a factor of $1+2q$ therefore translates each of these hyperplanes by $2q\cdot \frac{\bm\lambda'-\bm\lambda}{2}$ respectively. Since $\epsilon$ is the packing radius of $\Lambda_\epsilon$, $\|\frac{\bm\lambda'-\bm\lambda}{2}\|_2 \ge \epsilon$, and therefore each such hyperplane is translated a distance of at least $2q\epsilon $ away from $\bm\lambda$. So, $\widehat{\bm{Vor}}(\bm\lambda)$ contains all points within distance $2q\epsilon $ of $\bm{Vor}(\bm\lambda)$, completing the proof.
\end{proof} 

We can now present the encoding algorithm (Algorithm \ref{alg:sublinearencode}):

\begin{algorithm}[H]
	\caption{\textsc{SublinearEncode}, to compute $Q'_{\epsilon,q}(\bx)$}
	\label{alg:sublinearencode}
	\begin{algorithmic}
		\State $i\gets 0$
		\Loop
		\State Let $\bm \theta$ be a uniformly random vector in $\bm{Vor}(\bm 0)$.
		\State Let $\bm z$ be the closest lattice point in $\Lambda_\epsilon$ to $\bx+\bm\theta$.
		\State Let $c'\sim C$ be a random coloring $\Lambda_\epsilon \rightarrow [1+2q]^{3d}$
		\If{there is no $\bm z'\in\Lambda_\epsilon$ with $\bx+\bm\theta \in \bm{Vor}^+(\bm z')$ and  $c'(\bm z) = c'(\bm z')$}
		\State send $c'(\bm z)$ and $i$, terminate
		\EndIf
		\State $i\gets i+1$
		\EndLoop
	\end{algorithmic}
\end{algorithm}

The algorithm works as follows: we first apply a random offset vector $\bm \theta$, with the purpose of making the quantization unbiased. We then round to the closest lattice point. To convert this lattice point to a finite bit-string, we then apply a random coloring (from a distribution we will specify shortly); we show a lower bound on the probability that the color given to $\bm z$ is unique among those lattice points for whom $\bx+\bm\theta$ falls in the expanded Voronoi region. In this case, we send the color and the number $i$ of the current iteration; if not, we repeat the whole process with fresh shared randomness.

We first show a bound on the expected number of lattice points for which $\bx+\bm\theta$ falls in the expanded Voronoi region:

\begin{lemma}\label{lem:vorno}
	The number $V_{exp}(\bx+\bm\theta)$ of expanded Voronoi regions containing $\bx+\bm\theta$ is at most $(1+2q)^{2d}$ with probability at least $1-(1+2q)^{-d}$.
\end{lemma}

\begin{proof}
	Since $\bm \theta$ is uniformly distributed in $\bm{Vor}(\bm 0)$,  $\bx+\bm\theta - \bm z$ is uniformly distributed in $\bm{Vor}(\bm 0)$ (due to Voronoi regions being open, the probability distribution function differs from uniformity, but only on a set of $0$ measure).
	
	Then, since Voronoi regions are identical (including in their intersections with expanded Voronoi regions), and $\bx+\bm\theta$ falls in exactly $1$ Voronoi region in expectation, it falls in $\frac{Vol(\bm{Vor}^+(\bm 0))}{Vol(\bm{Vor}(\bm 0))} \le (1+2q)^d$ expanded Voronoi regions in expectation, by Lemma \ref{lem:vorvol}. Therefore, by Markov's inequality, the probability of falling within at least $(1+2q)^{2d}$ expanded Voronoi regions is at most $(1+2q)^{-d}$.
\end{proof}

We will need one further property of expanded Voronoi regions in order to show unbiasedness: that the number of expanded Voronoi regions containing a point is symmetric around any lattice point:

\begin{lemma}\label{lem:vornumber}
	For any $\bx \in \real^d$, $\bm \lambda\in \Lambda_\epsilon$, the points $\bm \lambda+\bx$ and $\bm \lambda-\bx$ are in the same number of expanded Voronoi regions $\bm{Vor}^+(\bm\lambda')$.
\end{lemma}

\begin{proof}
	This follow from the symmetry of $\Lambda_\epsilon$ with respect to $\lambda$: if $\bm \lambda+\bx\in \bm{Vor}^+(\bm\lambda')$, then $2\bm\lambda - \bm\lambda'$ is also a lattice point, and $\bm \lambda-\bx\in \bm{Vor}^+(2\bm\lambda - \bm\lambda')$.
\end{proof}

We next define our distribution of colorings: we will first apply the deterministic coloring $c_{3+2q}$ as described in Section \ref{sec:quant}. Then, by Lemma \ref{lem:latticecolor}, any two points of the same color are of distance at least $2(3+2q)\epsilon$ apart. Since $\bm{Vor}(\bm \lambda)\subset B_{r_c}(\bm \lambda) \subseteq B_{3\epsilon}(\bm \lambda) $, we have $\bm{Vor}^+(\bm\lambda)\subset B_{(3+2q)\epsilon}(\bm \lambda) $, and therefore there is no intersection between the expanded Voronoi regions of any points of the same color under $c_{3+2q}$. The purpose of this initial coloring is to limit the amount of randomness required for second random stage of the coloring process.

We then choose uniformly random colorings $\hat c : [3+2q]^d \rightarrow [1+ 2q]^{3d}$; the final coloring is then $\hat c\circ c_{3+2q} $, and we define $C$ to be the distribution of colorings generated in this way.

We will call such a coloring $c'$ \emph{successful} for a point $\bx+\bm\theta$ if it meets the condition described in the algorithm, i.e., if there is no $\bm z'\in\Lambda_\epsilon$ with $\bx+\bm\theta \in \bm{Vor}^+(\bm z')$ and  $c'(\bm z) = c'(\bm z')$.

\begin{lemma}\label{lem:colorsuccess}
	Over choice of $\bm\theta $ and $c'$, $\Prob{c'\text{ is successful}} \ge 1- 2(1+ 2q)^{-d}$, and \[\Exp{\bm z -\bm \theta  | c'\text{ is successful}}= \bx\enspace.\]
\end{lemma}

\begin{proof}
	The probability of $c'$ being successful is dependent entirely on the number $V_{exp}(\bx+\bm\theta)$ of expanded Voronoi regions containing $\bx+\bm\theta$. By Lemma \ref{lem:vorno}, this number is at most $(1+2q)^{2d}$ with probability at least $1- (1+2q)^{-d}$. In this case, the probability that $\bm z$ does not receive a unique color under $c'$ is at most $(1+2q)^{2d} \cdot (1+ 2q)^{-3d} = (1+2q)^{2d} \cdot (1+ 2q)^{-d}$ by a union bound. Therefore, the total probability of $c'$ not being successful is at most $2(1+ 2q)^{-d}$.
	
	To show unbiasedness (that $\Exp{\bm z -\bm \theta  | c'\text{ is successful}}= \bx$), we note that $\bx+\bm\theta - \bm z$ is uniformly distributed in $\bm{Vor}(\bm 0)$, and therefore $\Exp{\bm z -\bm \theta }= \bx$; by Lemma \ref{lem:vornumber}, $V_{exp}(\bx+\bm\theta)$, and therefore the probability of successful coloring, is symmetric around any lattice point, and so conditioning on successful coloring preserves unbiasedness.
\end{proof}

Once we have a coloring in which there is no $\bm z'\in\Lambda_\epsilon$ with $\bx+\bm\theta \in \bm{Vor}^+(\bm z')$ and  $c'(\bm z) = c'(\bm z')$, the following simple decoding procedure can find $\bm z$ so long as quantization input $\bx$ and decoding vector $\bx_v$ are sufficiently close:

\begin{algorithm}[H]
	\caption{\textsc{SublinearDecode}, to compute $R'_{\epsilon,q}(Q'_{\epsilon,q}(\bx),\bx_v)$}
	\label{alg:sublineardecode}
	\begin{algorithmic}
		\State Using $i$ and shared randomness, reconstruct $\bm\theta$ and $c'$
		\State Let $\bm {\hat{z}}\in \Lambda_\epsilon$ such that $B_{q\epsilon}(\bx_v+\bm\theta)$ intersects $\bm{Vor}(\bm {\hat{z}})$, and $c'(\bm {\hat{z}})$ matches $Q'_{\epsilon,q}(\bx)$
		\State Output $\bm {\hat{z}} - \bm\theta$
	\end{algorithmic}
\end{algorithm}

\begin{lemma}\label{lem:sublineardecode}
	If $\|\bx - \bx_v\|_2 \le q\epsilon$, then the decoding procedure $R'_{\epsilon,q}(Q'_{\epsilon,q}(\bx),\bx_v)$ correctly returns the vector $\bm z-\bm\theta$ which is an unbiased estimator of $\bx$.
\end{lemma}

\begin{proof}
	Clearly $\bm z$ is has the received color; we must show that $B_{q\epsilon}(\bx_v+\bm\theta)$ intersects $\bm{Vor}(\bm z)$, and does not intersect $\bm{Vor}(\bm z')$ for any other $\bm z'$ with the same color. The former is the case since $\bx+\bm\theta \in \bm{Vor}(\bm z)$ and $\|(\bx+\bm\theta) - (\bx_v+\bm\theta)\| \le q\epsilon$. The latter holds since if $B_{q\epsilon}(\bx_v+\bm\theta)$ intersects $\bm{Vor}(\bm z')$, then so does $B_{2q\epsilon}(\bx+\bm\theta)$, and so $\bx+\bm\theta \in \bm{Vor}^+(\bm z')$. However, the coloring was successful, so there is no $\bm z' \ne \bm z$ for which this is the case. So, Algorithm \ref{alg:sublineardecode} must successfully decode $\bm z$.
\end{proof}

\begin{theorem}
	Algorithms \ref{alg:sublinearencode} and \ref{alg:sublineardecode} give a quantization procedure with the following properties:
	If $\|\bx - \bx_v\|_2 \le q\epsilon$, then the decoder outputs an unbiased estimator $\bm {\hat{z}} - \bm\theta$ of $\bx$ with $\|\bx - (\bm {\hat{z}} - \bm\theta)\|_2\le 3\epsilon$. Each machine communicates $O(b)$ bits in expectation (and with probability at least $1-2^{(1-b)2^b}$), for $b= d\log(1+q)$.
\end{theorem}

\begin{proof}
	It remains only to prove a bound on the number of bits transmitted: transmitting $c'(\bm z)$ requires $\log \left((1+2q)^{3d}\right) = O(b)$ bits. To bound the number of encoding iterations required (and therefore the size of $i$), we see that each iteration independently succeeds with probability at least $1-2(1+2q)^{-d}\ge 1- 2^{1-b}$, by Lemma \ref{lem:colorsuccess}. Then, for any $j\in \nat$, $\Prob{i> 2^j} \le 2^{(1-b)2^j}$. Transmitting a value of $i \le 2^j$ requires $j$ bits, and therefore, setting $j = b$, we see that transmitting $i$ uses $b$ bits with probability at least $1-2^{(1-b)2^b}$, and in this case the total amount of communication used is $O(b)$.
\end{proof}

Note that we use Algorithms \ref{alg:sublinearencode} and \ref{alg:sublineardecode} primarily for the sublinear communication regime, i.e., when $q<1$, and in this case we have $b = d\log(1+q) = \Theta(dq)$. 
\subsection{\textsc{MeanEstimation} and \textsc{VarianceReduction} with Sublinear Communication}

If we wish to use $o(d)$-bit messages, then by Theorems \ref{thm:LB1} and \ref{thm:LB2a}, we cannot achieve lower output variance than input variance. Therefore, there is no longer any benefit to averaging; we can instead simply choose a random machine to broadcast its quantized input to all other machines.

\begin{algorithm}[H]
	\caption{\textsc{SublinearMeanEstimation}$(q)$}
	\label{alg:MEsublinear}
	\begin{algorithmic}
		\State Choose a source machine $u$ uniformly at random
		\State Machine $u$ broadcasts $Q'_{\frac{y}{q}, q}(\bx_u)$ to all other machines
		\State Each machine $v$ outputs $R'_{\frac{y}{q}, q}(Q'_{\frac{y}{q}, q}(\bx_u), \bx_v)$
	\end{algorithmic}
\end{algorithm}

\begin{theorem}\label{thm:MEsublinear}
	For any $q=O(1)$, Algorithm \ref{alg:MEsublinear} performs \textsc{MeanEstimation} with each machine using $O(dq)$ communication bits in expectation (and with probability at least $1-2^{(1-dq)2^{\Theta(dq)}}$), and with $O(\frac{y^2}{q^2})$ output variance under $\ell_2$ norm.
\end{theorem}

\begin{proof}
	The input of the chosen source $\bx_u$ is an unbiased estimator of the mean input $\bm \mu$ with, $\|\bx_u-\bm \mu\|_2\le y$. Since all machines have $\|\bx_u - \bx_v\|_2\le \frac{y}{q}\cdot q = y$, all machines correctly decode a common unbiased estimate $\bm z$ of $\bx_u$ with $\|\bm z - \bx_u\|_2\le 3\frac{y}{q}$. For $q=O(1)$, the output is therefore an unbiased estimate of $\bm \mu$ with $O(\frac{y^2}{q^2})$ variance.
	
	If the broadcast is performed using a binary tree communication structure, then each machine receives at most one message and sends at most two. We have $b=d\log (1+q) = \Theta(dq)$, and therefore each machine requires $O(dq)$ communication bits, in expectation and indeed with probability at least $1-2^{(1-dq)2^{\Theta(dq)}}$.
\end{proof}

Together with Theorem \ref{thm:ME}, this implies Theorem \ref{thm:sublinear}.

We again apply the same reduction as for Theorem \ref{thm:VR} to obtain the following sublinear-communication result for \textsc{VarianceReduction} (though the name is less fitting, since variance is now necessarily \emph{increased}):

\begin{theorem}\label{thm:VRsublinear}
	For any $\alpha>1$, and any $q = O(1)$, Algorithm \ref{alg:MEsublinear} performs \textsc{VarianceReduction} using $O(dq)$ communication bits per machine (with probability at least $1-2^{(1-dq)2^{\Theta(dq)}}$), with $ O(\frac{\alpha n\sigma^2}{q^2})$ output variance under $\ell_2$ norm, succeeding with probability at least $1-\frac 1\alpha$.
\end{theorem}

\begin{proof}
	As in proof of Theorem \ref{thm:VR}, we have that a \textsc{MeanEstimation} algorithm using $y=2\sigma\sqrt{\alpha n}$ performs \textsc{VarianceReduction}, succeeding with probability at least $1-\frac{1}{\alpha}$, which we plug into Theorem \ref{thm:MEsublinear}.
\end{proof}	

\section{Lower Bounds}\label{sec:LB}
We prove matching lower bounds for \textsc{MeanEstimation} and \textsc{VarianceReduction}, using an argument based on bounding the volume of space for which a node can output a good estimate, if it receives a limited number of communication bits. Our lower bounds will be against algorithms with \emph{shared randomness}: that is, we assume access to a common random string $s$, drawn from some distribution, for all machines.

We begin with the simpler \textsc{MeanEstimation} bound:	
\begin{theorem}\label{thm:LB1}
	For any \textsc{MeanEstimation} algorithm in which any machine receives at most $b$ bits, \[\Exp{\|\bm{EST} - \mu\|^2}=\Omega(y^2 2^{-\frac {2 b}{d}})\enspace.\]
\end{theorem}

\begin{proof}
	We construct a hard input for mean estimation as follows: fix a machine $v$, and arbitrarily fix its input vector $\bx_v$. Machine $v$'s output $\bm{EST}$ is dependent only on $\bx_v$, the random string $s$, and the string of bits $B_v$ that $v$ receives (from any other machines) during the course of the algorithm. We denote by $b_v$ the number of such bits. If $b_v<b$ for some $b$, then $v$ has $\sum_{i=0}^{b-1} 2^i < 2^b$ possible strings $B_v$, and therefore fewer than $2^b$ possible output distributions (over choice of $s$). 
	
	For each $B_v$ we denote by $OUT_{B_v}$ the set of points $\bm z\in \real^d$ for which $\bm P_s[\|\bm{EST}- \bm z\|\le \delta] > \frac12$, when $B_v$ is the string received. For any $B_v$, $OUT_{B_v}$ has volume at most $Vol(B_{2\delta})$, since the $\delta$-balls of any two points in $OUT_{B_v}$ must intersect (as the probabilities of $\bm{EST}$ falling within the balls sum to more than $1$). We further denote $OUT_{v}$ to be the union of these sets over all $B_v$. Then:
	
	\[Vol(OUT_{v})\le \sum_{B_v} Vol(OUT_{B_v}) < 2^b \cdot Vol(B_{2\delta}) \le 2^b \left(\frac{4 \delta}{y}\right)^d Vol(B_{\frac y2}) \enspace.  \]
	
	We choose $\delta = \frac{y}{4} 2^{-\frac bd}$, and see that: 
	
	\[Vol(OUT_{v}) < 2^b \left(\frac{4 \cdot \frac y4 2^{-\frac bd}}{y}\right)^d Vol(B_{\frac y2})  =  2^b \cdot 2^{-b} \cdot Vol(B_{\frac y2}) = Vol(B_{\frac y2})\enspace.  \]
	
	Therefore there is some point $\bm z$ in $B_{\frac y2}(\bx_v) \setminus Vol(OUT_{v})$. We choose $\mu$ to be such a point $\bm z$, by setting one other machine's input to $2\bm z-\bx_v$, and any others to $\bm z$ (note that this satisfies the condition that all inputs are within distance $y$). Then, regardless of $B_v$, $\bm P_s[\|\bm{EST}- \mu\|\le \delta] \le \frac12$, and so $\Exp{\|\bm{EST} - \mu\|^2} \ge \frac12 \cdot \delta^2 = \Omega(y^2 2^{-\frac {2b}{d}})$.
\end{proof}

This lower bound is for algorithms with a fixed communication cost, but we show that it can easily be extended to apply to \emph{expected} communication cost, proving Theorem \ref{thm:LB1b}:

\begin{proof}[Proof of Theorem \ref{thm:LB1b}]
	By Markov's inequality, a machine $v$ which receives at most $b$ bits in expectation receives at most $1.5 b$ bits with probability at least $\frac 13$. Then, \[\Exp{\|\bm{EST} - \mu\|^2} \ge \frac 13 \Exp{\|\bm{EST}- \mu\|^2\mid \text{$v$ receives at most $1.5 b$ bits}} =\Omega(y^2 2^{-\frac {3 b}{d}})\enspace.\]
\end{proof}

The bounds for \textsc{VarianceReduction} are somewhat more complex since we must define a hard \emph{input distribution} in which the input vectors are independent estimates of the true vector $\grad$ and have variance at most $\sigma^2$. 

We first show a bound on the amount of communication bits that \emph{all} machines must use (Theorem \ref{thm:LB2a}, before proceeding to Theorem \ref{thm:LB2b} which shows that some machines must use more bits.

\begin{proof}[Proof of Theorem \ref{thm:LB2a}]
	We define an input distribution $I$ as follows: we choose  $\grad$ uniformly at random from $B_{d^2 \sigma}(\textbf 0)$. We then independently choose each machine $v$'s input $\bx_v$ uniformly at random from $B_{\sigma}(\grad)$.
	
	For any machine $v$, we define event $A_v = \{\|\bx_v\|\le (d^2-1) \sigma\}$. 
	
	\begin{align*}	
		\Prob{A_v} &\ge \Prob{\|\grad\|\le (d^2-2)\sigma}\\
		& \ge \frac{Vol(B_{(d^2-2)\sigma})}{Vol(B_{d^2 \sigma})}\\
		&\ge \left(\frac{(d^2-2)\sigma}{d^2 \sigma}\right)^d\\
		&\ge \left(1-\frac{2}{d^2}\right)^d\\
		&\ge 1-\frac{4}{d}\enspace.
	\end{align*}
	
	Now, conditioning on $A_v$, $\grad$ is distributed uniformly in $B_{\sigma}(\bx_v)$. We next show that in this case, $v$ will, with high probability, not be able to closely estimate $\grad$.
	
	Using the same argument and definitions as in proof of Theorem \ref{thm:LB1}, we have:
	
	\[Vol(OUT_{v})\le \sum_{B_v} Vol(OUT_{B_v}) < 2^b \cdot Vol(B_{2\delta}) \le 2^b \left(\frac{2 \delta}{\sigma }\right)^d Vol(B_{\sigma}) \enspace.  \]
	
	We choose $\delta = \frac{\sigma}{4} 2^{-\frac {b}{d}}$, and see that $Vol(OUT_{v}) < 2^{-d} Vol(B_{\sigma})$. Therefore, $\bm P_I [\grad \in OUT_v\mid A_v]\le 2^{-d}$, and so $P_I [\grad \notin OUT_v]\le (1-\frac 4d)(1-2^{-d}) \le\frac 12$. When $\grad \notin OUT_v$, with probability at least $\frac 12$, $\bm P_s[\|\bm{EST}- \mu\|\le \delta] \le \frac12$. 
	
	So, we have that:
	\[\Prob{\|\bm{EST}- \mu\|> \delta} \ge \frac 12  \Prob{\grad \notin OUT_v}  > \frac 12  (1-\frac{4}{d})\bm P_I[\bigcup_{v\in M}A_v]>0.2\enspace.\]
	
	Then, $\Exp{\|\bm{EST}-\grad\|^2} > 0.2 \delta^2 = \Omega(\sigma^2 n 2^{-\frac{2b}{d}})$.
\end{proof}

Theorem \ref{thm:LB2b} is proven similarly, but using a different input distribution that causes some inputs to be further from $\grad$.

\begin{proof}[Proof of Theorem \ref{thm:LB2b}]
	We define an input distribution $I$ as follows: we choose  $\grad$ uniformly at random from $B_{d^2 n\sigma}(\textbf 0)$. We then independently choose each machine $v$'s input $\bx_v$ as follows: with probability $\frac 1n$, $\bx_v$ is chosen uniformly from $B_{\sigma \sqrt{n}}(\grad)$, and otherwise $\bx_v = \grad$. 
	
	For any machine $v$, we are interested in the event $A_v$ that 
	\begin{itemize}
		\item $\bx_v$ is chosen from $B_{\sigma \sqrt{n}}(\grad)$, and
		\item $\|\bx_v\|\le d^2 n\sigma - \sigma \sqrt{n}$;
	\end{itemize}
	
	If we condition on $\grad \in B_{d^2n\sigma-2\sigma \sqrt{n}}(\bm 0)$, the second criterion will be true for all $v$, and so the events $A_v$ will occur independently with probability $\frac 1n$. We use this to show that the probability that $A_v$ occurs for \emph{some} $v$ is greater than $\frac 12$:

	\begin{align*}
		\bm P_I [\bigcup_{v \in M}A_v] &\ge \bm P_I [\{\grad \in B_{d^2n\sigma-2\sigma \sqrt{n}}(\bm 0)\}\cap \bigcup_{v \in M}A_v]\\
		&\ge \bm P_I [\{\grad \in B_{d^2n\sigma-2\sigma \sqrt{n}}(\bm 0)\}]\cdot \bm P_I[\bigcup_{v \in M}A_v\mid \grad \in B_{d^2n\sigma-2\sigma \sqrt{n}}(\bm 0)]\\
		&= \frac{Vol(B_{d^2n\sigma-2\sigma \sqrt{n}})}{Vol(B_{d^2 n\sigma})}\cdot \left(1-(\frac {n-1}{n})^n\right)\\
		&\ge \left(\frac{d^2n\sigma-2\sigma \sqrt{n}}{d^2 n\sigma}\right)^d \cdot (1-e^{-1})\\
		&> 0.6 \left(1-\frac{1}{d^2}\right)^d \ge 0.6\left(1-\frac{2}{d}\right) > 0.5\enspace.
	\end{align*}
	
	Now, conditioning on $A_v$, $\grad$ is distributed uniformly in $B_{\sigma \sqrt{n}}(\bx_v)$. We next show that in this case, $v$ will, with high probability, not be able to closely estimate $\grad$.
	
	Again we have:
	
	\[Vol(OUT_{v})\le \sum_{B_v} Vol(OUT_{B_v}) < 2^b \cdot Vol(B_{2\delta}) \le 2^b \left(\frac{2 \delta}{\sigma \sqrt{n}}\right)^d Vol(B_{\sigma \sqrt{n}}) \enspace.  \]
	
	We choose $\delta = \frac{\sigma \sqrt{n}}{4} 2^{-\frac {b}{d}}$, and see that $Vol(OUT_{v}) < 2^{-d} Vol(B_{\sigma \sqrt{n}})$. Therefore, $\bm P_I [\grad \in OUT_v\mid A]\le 2^{-d}$, and so $P_I [\grad \in OUT_v]\le \frac 12$. When $\grad \notin OUT_v$, with probability at least $\frac 12$, $\bm P_s[\|\bm{EST}- \mu\|\le \delta] \le \frac12$. 
	
	So, we have that:
	\[\Prob{\|\bm{EST}- \mu\|> \delta} \ge \frac 12  \Prob{\grad \notin\bigcap_{v\in M} OUT_v}  > \frac 12  (1-2^{-d})\bm P_I[\bigcup_{v\in M}A_v]>0.2\enspace.\]
	
	Then, $\Exp{\|\bm{EST}-\grad\|^2} > 0.2 \delta^2 = \Omega(\sigma^2 n 2^{-\frac{2b}{d}})$. Again, by Markov's inequality, we can obtain that if machine $v$ receives at most $b$ bits \emph{in expectation}, we have $\Exp{\|\bm{EST}-\grad\|^2}= \Omega(\sigma^2 n 2^{-\frac{3b}{d}})$.
\end{proof}

Theorems \ref{thm:LB1} and \ref{thm:LB2a} imply that to obtain output variances of $O(\frac{y^2}{q})$ and $O(\frac{\sigma^2}{q})$ for \textsc{MeanEstimation} and \textsc{VarianceReduction} respectively, we require $\Omega(d\log q)$ communication bits per machine, thereby matching the upper bounds of Theorems \ref{thm:ME} and \ref{thm:VR2} (the latter in expectation). Theorem \ref{thm:LB2b} further implies that to achieve any variance reduction \emph{at all} (i.e. for output variance to be lower than input variance $\sigma^2$), at least one machine must receive $\Omega(d\log n)$ bits, matching the absolute upper bound on bits required to achieve the optimal output variance $O(\frac{\sigma^2}{n})$ by Theorem \ref{thm:VRshort}.

\out{
\section{Experimental Validation}\label{sec:exp-short}

	\paragraph{Example 1: Compressing Gradients.} We first apply compression to parallel solver for least-squares regression, which is given as input a matrix $A$, partitioned among the nodes, and a target vector $\bm b$, with the goal of finding $\bm w^* = \text{argmin}_{\bm w}||A\bm w-\bm b||^2_2$. In this example, we generate $\bm w^*\in \mathbb{R}^d$ and entries of $A\in \mathbb{R}^{S\times d}$ by sampling from $\mathcal{N}(0,1)$, and we set $\bm b = A\bm w^*$. Note that the input data (and gradients) are thus naturally normalized. We then run distributed gradient descent with our quantization scheme, using $3$ bits per coordinate for all methods, and examine the characteristics of the gradients sent, as well as the convergence and variance of the overall process. Figure~\ref{fig:regression} (left) shows the gradient norms and distances for the case of two nodes executing gradient descent for $d = 100$ and $8$K samples, while the center and right panels show the  variance of the gradient estimate for each method and convergence, respectively. 
	
	\begin{figure}[h]  
    \begin{subfigure}{0.3\textwidth}
        \includegraphics[width=5cm]{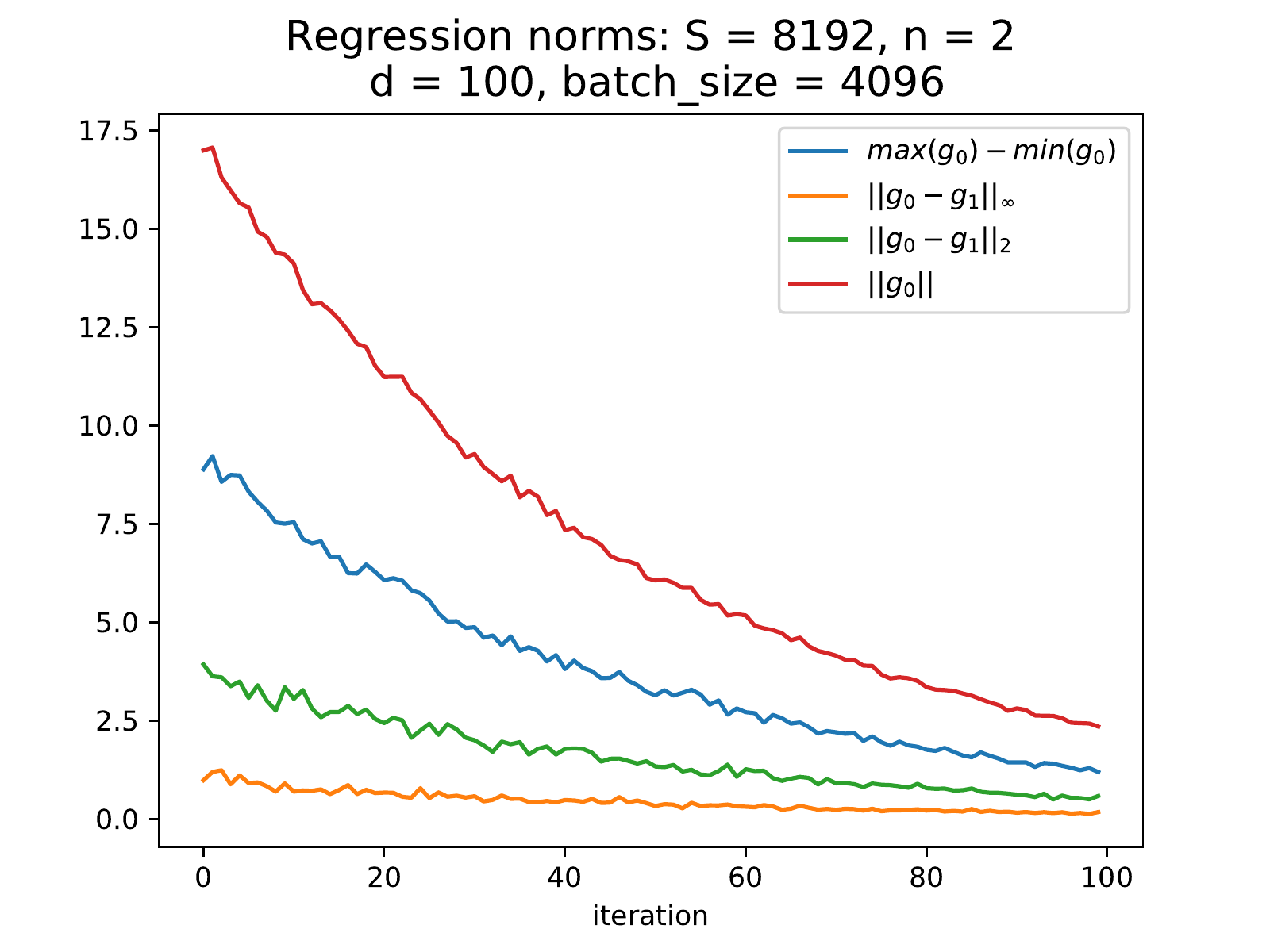}
    \end{subfigure}
    \hfill 
    \begin{subfigure}{0.3\textwidth}
        \includegraphics[width=5cm]{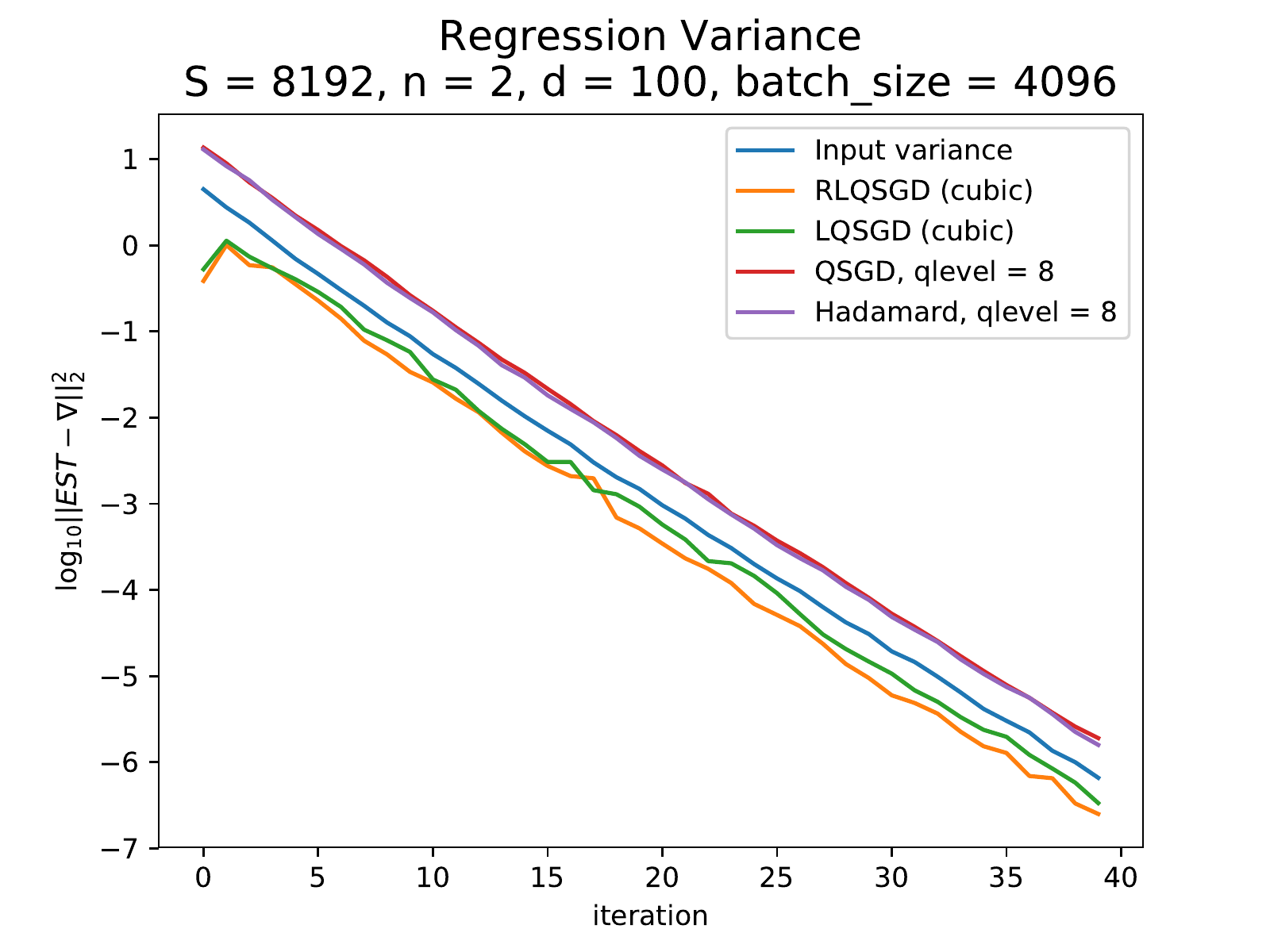}
    \end{subfigure}
    \hfill
    \begin{subfigure}{0.3\textwidth}
        \includegraphics[width=5cm]{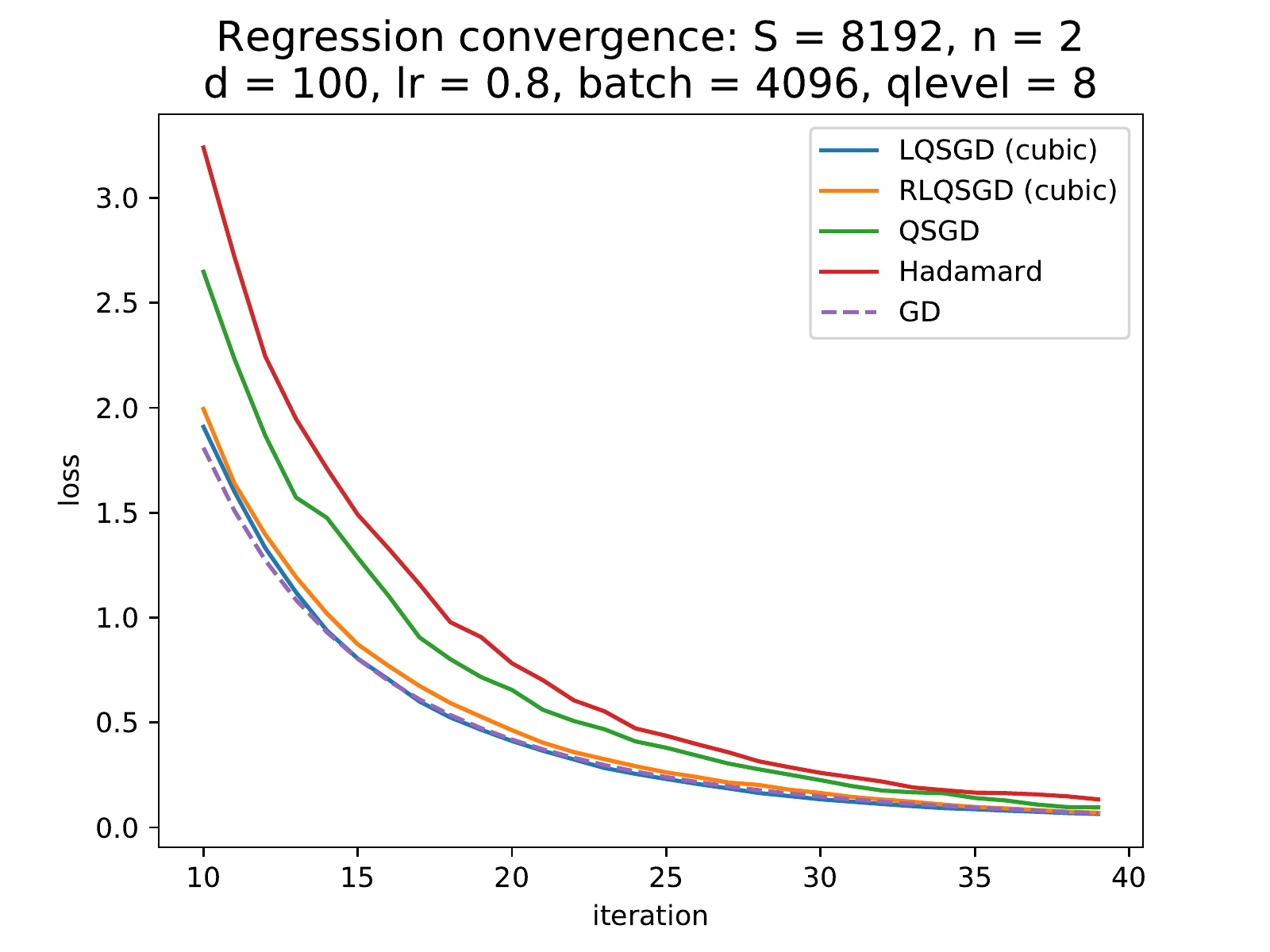}

    \end{subfigure}
    \caption{Gradient quantization results for the regression example.}
    \label{fig:regression}
\end{figure}

	Figure~\ref{fig:regression} (left) shows that, even in this simple normalized example, the distance between the two gradients throughout training is much smaller than the norm of the gradients themselves. Figure~\ref{fig:regression} (center) shows that both variants of our scheme (RLQSGD and LQSGD, with and without rotation) have significantly lower variance relative to Hadamard and standard QSGD, and in fact are the only schemes to get below input variance in this setting. Figure~\ref{fig:regression} (right) shows that this leads to better convergence using our method. 
	The full report in Section~\ref{sec:exp} contains experiments on other datasets, and node counts, as well as a larger-scale application of our scheme to train neural networks. This application shows that our scheme matches or slightly improves the performance of specialized gradient compression schemes, at the same bit budget.



\begin{figure}[h!]
    \begin{subfigure}{0.45\textwidth}
        \includegraphics[width=6cm]{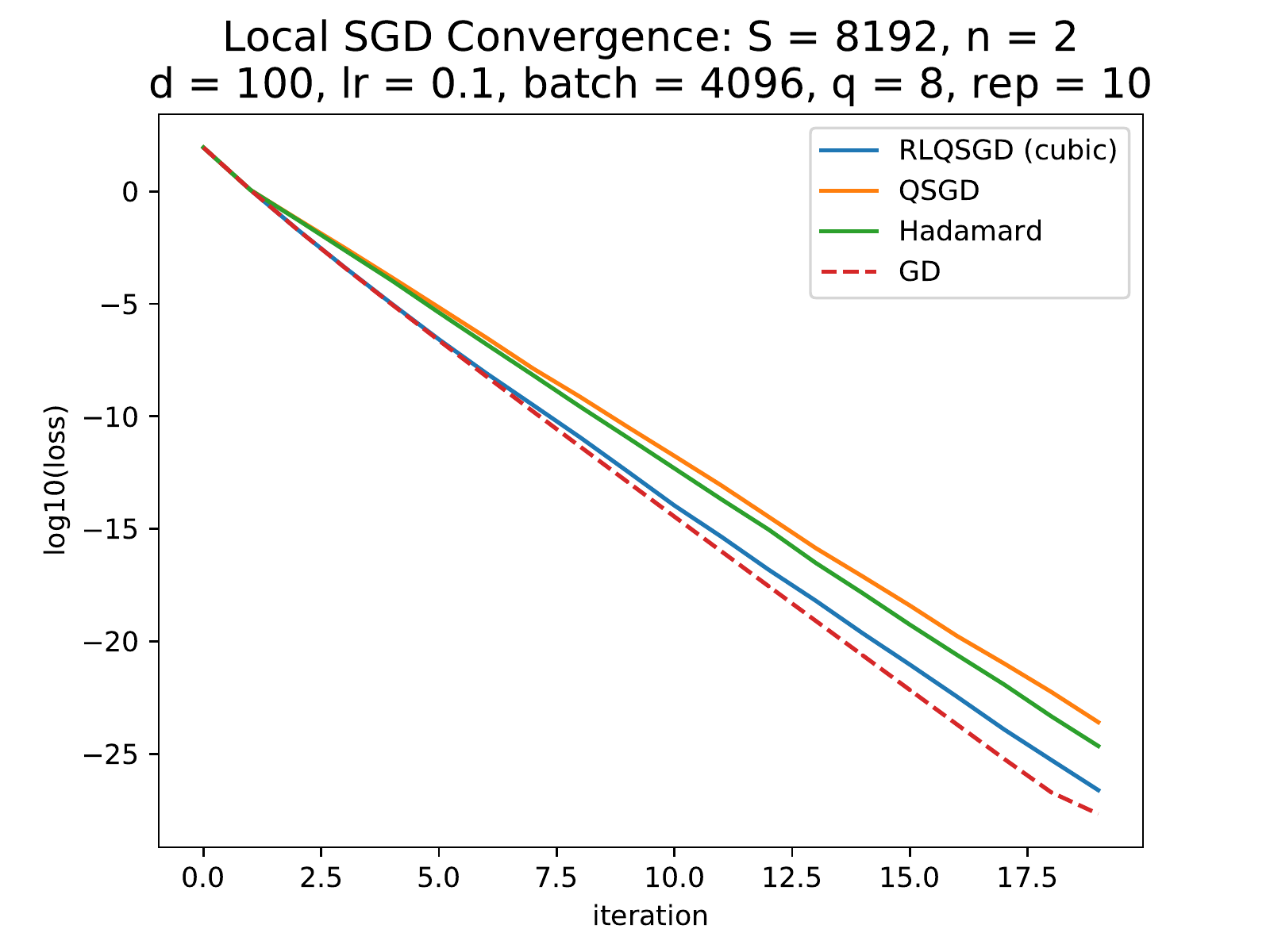}
    \end{subfigure}
    \hfill
    \begin{subfigure}{0.45\textwidth}
        \includegraphics[width=6cm]{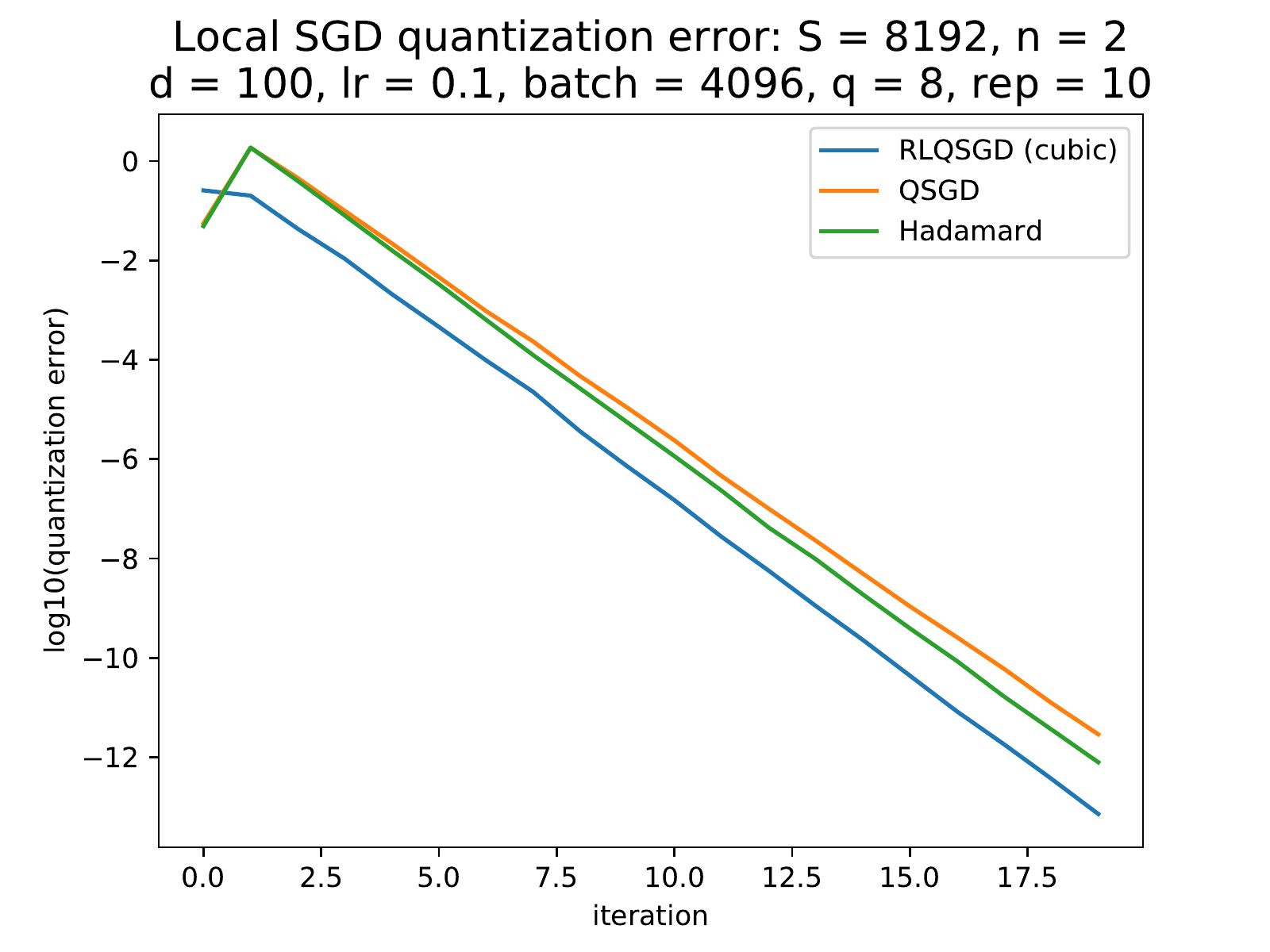}
    \end{subfigure}
    \caption{Local SGD: convergence for different quantizers (left) and quantization error (right).}
    \label{fig:local}
\end{figure}
	
	\paragraph{Example 2: Local SGD.} 
	A related example is that of compressing models in LocalSGD~\cite{stich2018local}, where each node takes several SGD steps on its local model, followed by a global \emph{model averaging} step, among all nodes. 
(Similar algorithms are popular in Federated Learning~\cite{kairouz2019advances}.) 
We use RLQSGD to quantize the models transmitted by each node as part of the averaging: to avoid artificially improving our performance, we compress the \emph{model difference} $\Delta_i$ between averaging steps, at each node $i$. 
RQSGD is a good fit since neither the models nor the $\Delta_i$ are zero-centered. We consider the same setup as for the previous example, averaging every 10 local SGD iterations. We illustrate the convergence behavior and the quantization error in Figure~\ref{fig:local}, which shows better convergence and higher accuracy for lattice-based quantization.

	\paragraph{Example 3: Distributed Power Iteration.} 
	The power iteration method estimates the principal eigenvector of an input matrix $X$, whose rows are partitioned across machines, to form input matrices $X_i$ at the nodes. Each row of the input matrix $X$ is generated from a multivariate gaussian with first two eigenvalues large and comparable.
	In each iteration, each machine updates their estimate relative to its input matrix as $\bm u_i = X_i^T X_i \bx$, and nodes average these estimates. We apply $8$-bit quantization to communicate these vectors $\bm u_i$, following~\cite{MeanEstimation}. Notice that our method is a good candidate here since these estimates clearly need not be zero-centered, a fact that is evident in Figure~\ref{fig:power} (left). Please see Section~\ref{sec:power-iteration} for a full description and additional results. Figure \ref{fig:power} (center) shows convergence under different quantization schemes, while Figure \ref{fig:power} (right) shows that our method provides significantly lower quantization error across iterations.
	
	\begin{figure}[h]
		\centering
		\begin{minipage}[b]{0.3\linewidth}
			\centering
			\includegraphics[width=5cm]{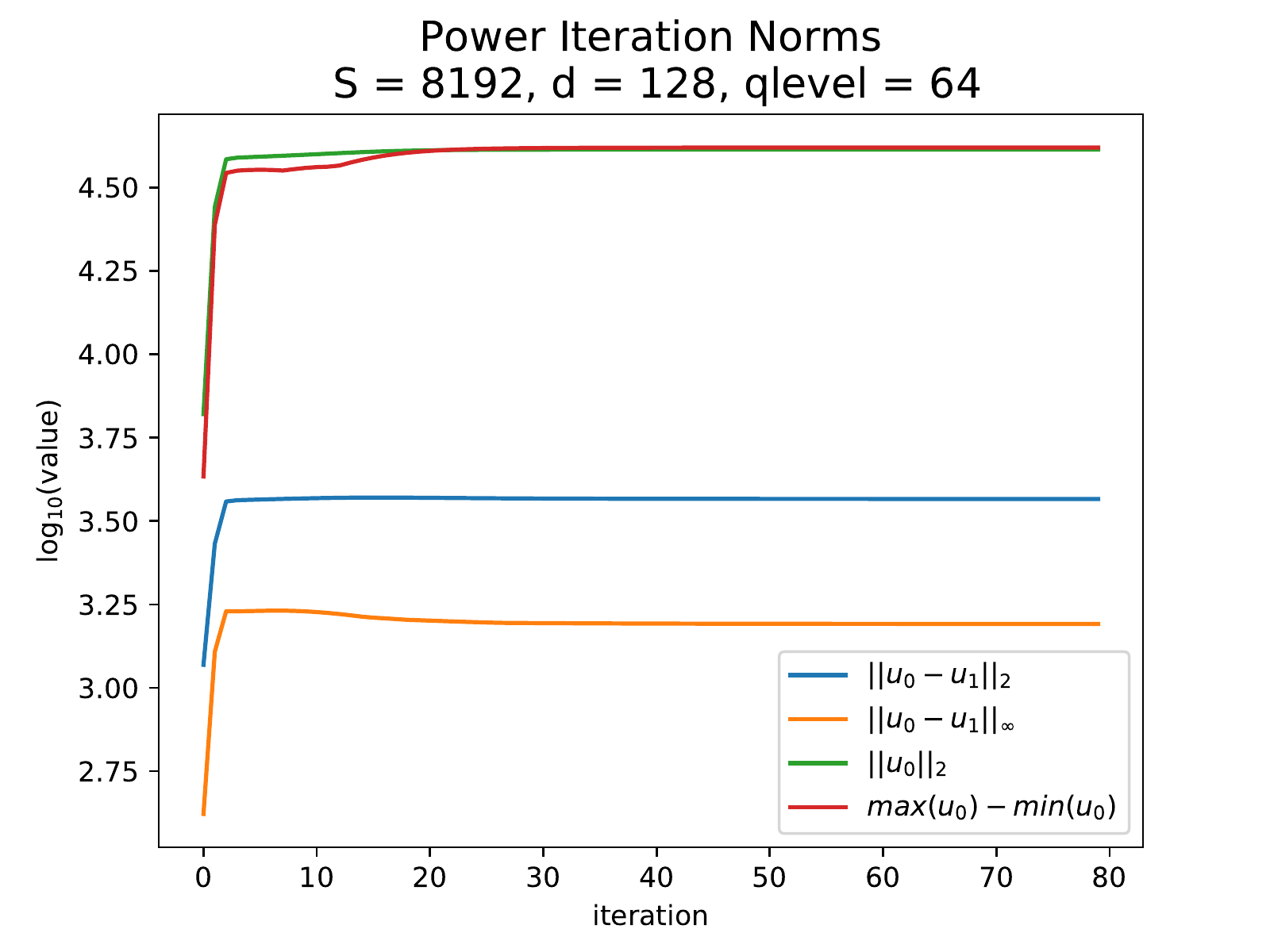}
		\end{minipage}
		\quad
		\begin{minipage}[b]{0.3\linewidth}
			\centering
			\includegraphics[width=5cm]{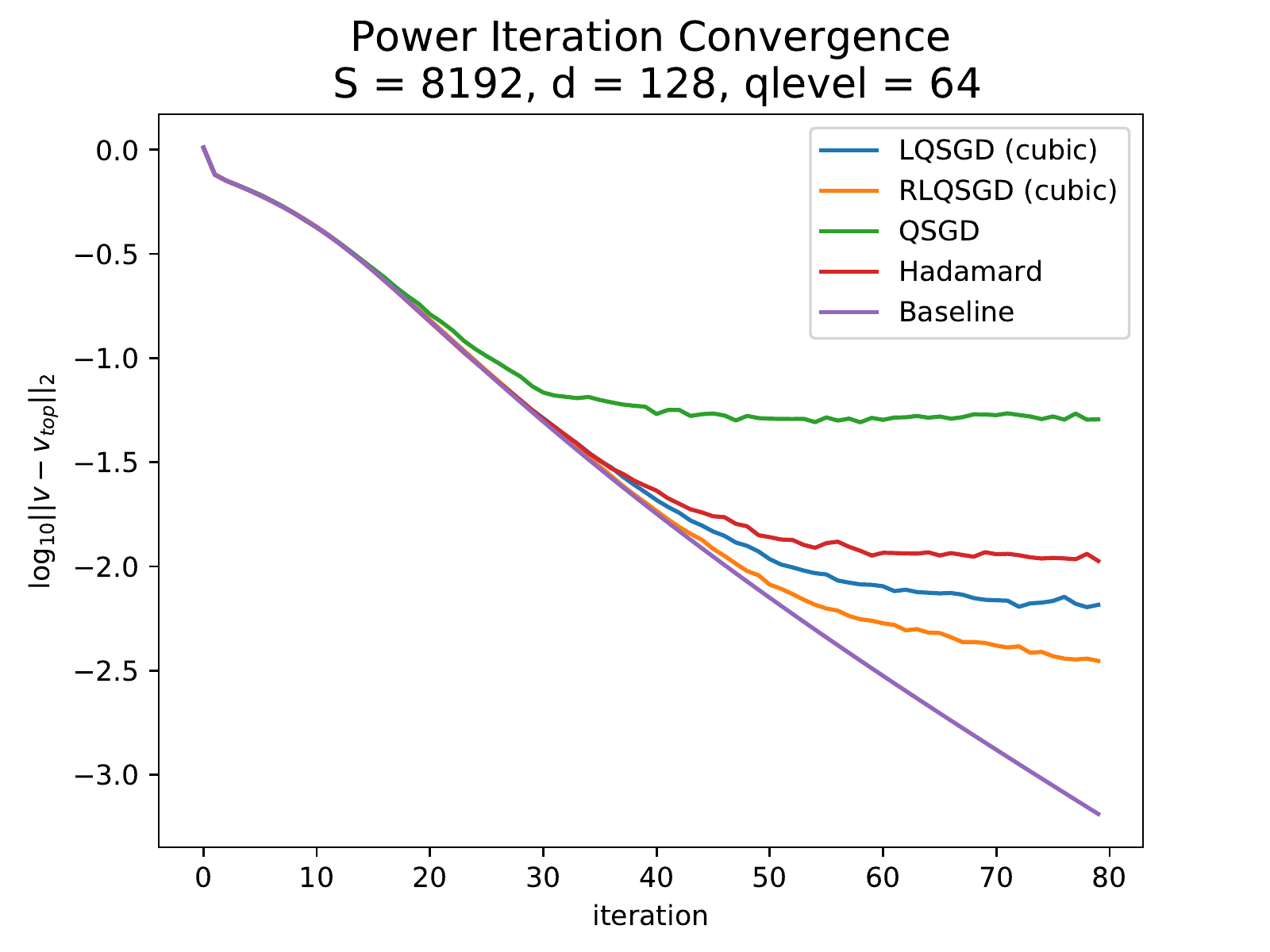}
		\end{minipage}
		\quad
		\begin{minipage}[b]{0.3\linewidth}
			\centering
			\includegraphics[width=5cm]{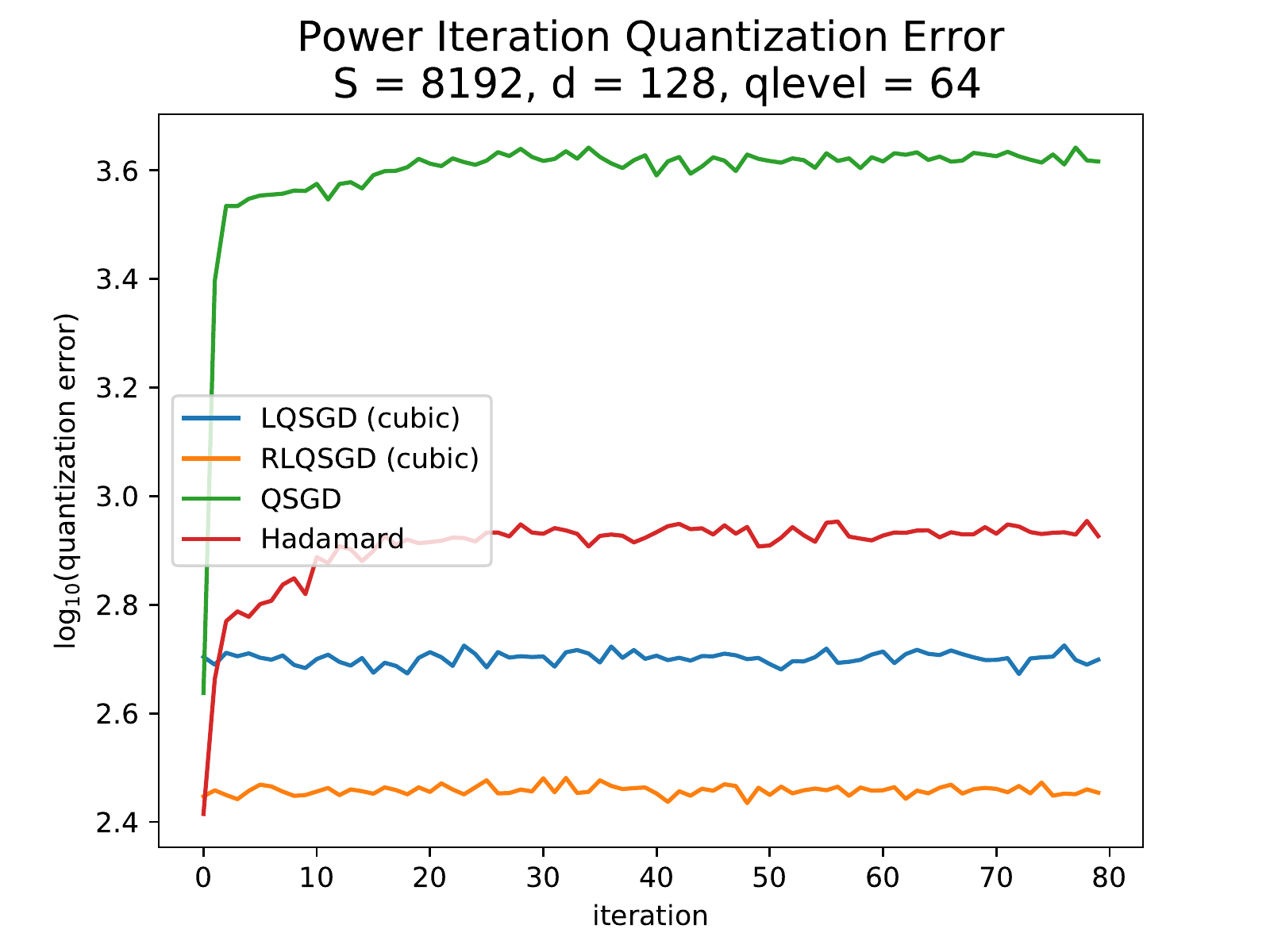}

		\end{minipage}
    \caption{Power iteration: input norms (left), convergence (center) and quantization error (right).}
    \label{fig:power}
\end{figure}

}

	\section{Experimental Report}\label{sec:exp}
We consider three distinct applications:  data-parallel SGD, distributed power iteration for eigenvalue computation~\cite{MeanEstimation}, and adding communication compression to local SGD~\cite{stich2018local}. We implement the practical version of our algorithm (see Section~\ref{sec:practical} below), which for simplicity we call LQSGD, using a cubic lattice and $\ell_\infty$ to measure input variance, and a version we call RLQSGD with the Hadamard rotation to ensure close-to-optimal variance. We compare against QSGD~\cite{QSGD}, the Hadamard-based scheme of~\cite{MeanEstimation}, as well as uncompressed baselines.

\subsection{The Algorithm in Practice}\label{sec:practical}

Practically, the LQSGD algorithm with two machines ($u$ and $v$) works as follows: we first offset the cubic lattice by a uniformly random vector in $[-\frac{s}{2},\frac{s}{2}]^d$, using shared randomness. This ensures that quantizing to the \emph{nearest} lattice point now gives an unbiased estimator, avoiding the extra cost of the convex-hull method (though that can still be employed practically if shared randomness is unavailable, and for the cubic lattice is a simple coordinate-wise rounding procedure). Machine $u$ rounds to its closest lattice point and sends its mod-$q$ color to machine $v$, who decodes it by finding its closest lattice point with the same color, and vice versa. Finding the closest lattice point under the cubic lattice can be performed coordinate-wise using $\tilde O(d)$ computation, and gives a lattice point within $\frac s2$ $\ell_\infty$ norm distance. Mod-$q$ coloring can also be efficiently computed, and has the property that any two lattice points $\bm\lambda$, $\bm\lambda'$ with the same coloring have $\|\bm\lambda-\bm\lambda'\|_\infty \ge qs $. Therefore, if the input gradients $\bm g_0$, $\bm g_1$ have $\|\bm g_0-\bm g_1\|_{\infty} \le \frac{(q-1)s}{2}$, then decoding is successful. So, assuming we have an estimate $y$ such that for all $\bm g_0$, $\bm g_1$ we have $\|\bm g_0-\bm g_1\|_{\infty} \le y$, we set our side-length $s=\frac{2y}{q-1}$. 

We also implement the algorithm using the structured random rotation described in Section \ref{sec:cubic}, which we call RLQSGD. Here, we also generate the matrix $D$ on machines using shared randomness, and then apply the transformation $HD$ to inputs before quantization. The algorithm then proceeds exactly as LQSGD; when setting lattice side length we use an estimate $y_R$ of $\ell_\infty$-norm distance \emph{after} applying $HD$,
i.e., $||HD(g_0-g_1)||_\infty\leq y_R$.

We will describe in each experiment how we set and update our estimate of $y$; generally, this can be done by simply measuring $\ell_\infty$ norms between inputs during the course of the algorithm, and multiplying these by a small constant factor ($1.5$ to $3.5$) to ensure sufficient slack that all decodes succeed. 

\subsection{Least-Squares Regression}\label{sec:LSexp}
The classic problem of least-squares is as follows: given as input some matrix $A$ and target vector $\bm b$, our goal is to find $\bm w^* = \text{argmin}_{\bm w}||A\bm w-\bm b||^2_2$, i.e. the vector which, upon multiplication by $A$, minimizes Euclidean distance to $\bm b$.

To obtain instances of the problem in order to test our approach, we generate $\bm w^*\in \mathbb{R}^d$ and entries of $A\in \mathbb{R}^{S\times d}$ by sampling from $\mathcal{N}(0,1)$, and we set $\bm b = A\bm w^*$. We then run distributed gradient descent using our quantization scheme, with the following settings: 
\begin{enumerate}
	\item $S = 8192$ samples in $d = 100$ dimensions
	\item $n = 2$ worker machines
	\item results are averaged over 5 random seeds [0, 10, 20, 30, 40]
\end{enumerate}
In all experiments, in each iteration of gradient descent, the dataset (i.e., the rows of matrix $A$) will be randomly divided into two equal $\frac{S}{2}$-sized groups and provided to the two machines. We denote $\frac{S}{2}$-batch gradients by $\bm g_0$, $\bm g_1$ respectively. We use large batch sizes in order to be in a regime where our algorithm's assumptions can give a visible advantage over other bounded norm based approaches. 

\paragraph{Experiment 1: Norms relevant to quantization schemes.}
Our first experiment is to provide justification of our main message, that \emph{input variance} (or similar measure of mutual distance depending on the setting) is often far lower than \emph{input norm} in practical applications, and is a more natural quantity with which to measure output variance of algorithms.

In Figures \ref{fig:invar1} and \ref{fig:invar2} we compare four quantities, in the setting of least-squares on two worker machines: 

\begin{itemize}
	\item batch gradient distance $\|\bm g_0-\bm g_1\|_2$ (the closest proxy of our theoretical quantity $y$ from \textsc{MeanEstimation}),
	\item $\|\bm g_0-\bm g_1\|_{\infty}$ (the equivalent quantity under $\ell_\infty$ norm, which is the appropriate norm when using the cubic lattice),
	\item $\|\bm g_0\|_2$, the batch gradient norm, (which is used as the measure of input size in QSGD-L2 \cite{QSGD}, as well as most other prior work)
	\item $\max(\bm g_0)-\min(\bm g_0)$, the batch gradient coordinate difference, used as the measure of input size in the QSGD implementation \cite{QSGD}
\end{itemize}  

We see that the former two quantities, used by our algorithms, are far lower than the latter two, i.e., input variance/distance is significantly lower than input norm, and batch gradients are not centered around the origin. Similar effects are seen in a wide variety of applications and parameter settings, and justify our focus on input distances over norms. Iterations in the figures are under gradient descent performed using full (unquantized) gradient.

\begin{figure}[ht]
	\centering
	\begin{minipage}[b]{0.45\linewidth}
		\centering
		\includegraphics[width=7cm]{figures/norms/norms_S_8192_d_100.pdf}
		\caption{norms using fewer samples}
		\label{fig:invar1}
	\end{minipage}
	\quad
	\begin{minipage}[b]{0.45\linewidth}
		\centering
		\includegraphics[width=7cm]{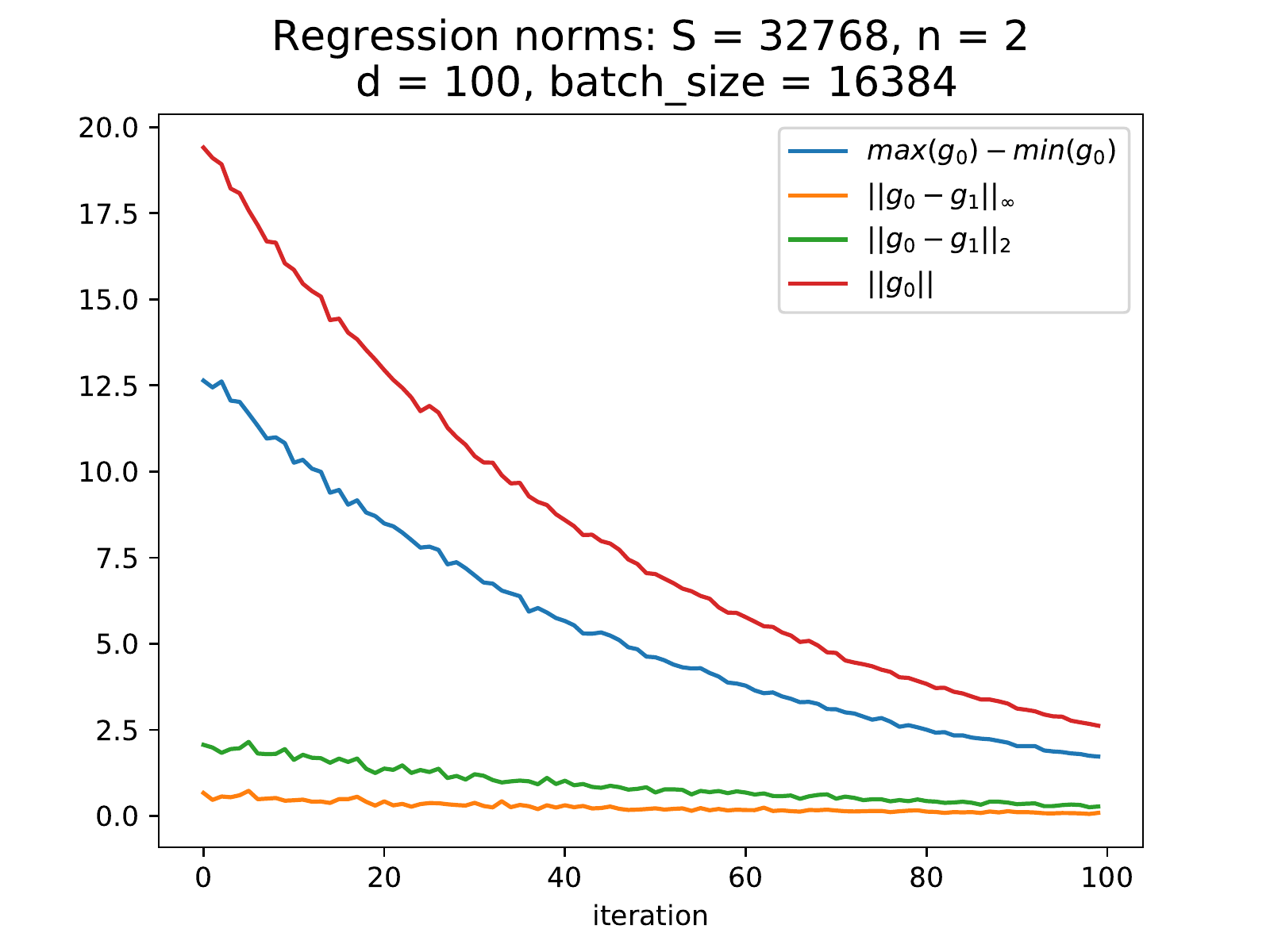}
		\caption{norms using more samples}
		\label{fig:invar2}
	\end{minipage}
\end{figure}

\paragraph{Experiment 2: Variances using quantization methods.}
We now apply quantization schemes to the communication gradients, and measure the variances induced, in order see whether the input norm versus variance difference demonstrated by Figures \ref{fig:invar1} and \ref{fig:invar2} translates to lower output variance for LQSGD over other schemes.

In this experiment, we perform distributed stochastic gradient descent, by communicating, averaging, and applying the \emph{quantized} batch gradients $\bm g_0$ and $\bm g_1$ of the two worker machines in each iteration. By construction, this resulting value $EST$ will be an unbiased estimator of the full true gradient $\grad$, and we will examine the associated \emph{output variance} $\mathbb{E}[\| EST-\grad \|^2_2]$, for different quantization schemes. We experiment with $q= 8$ for all methods (for QSGD the equivalent quantity is referred to as qlevel), which means that messages comprise of $\log_2 8 = 3$ bits per coordinate. In norm-based methods like QSGD, there is an additional communication overhead of one/two $64-$bit floating point values for every quantization, used to convey input norms.

To allow machines using LQSGD to maintain an estimate $y$ of distance between inputs, in the first iteration alone a pre-computed estimate is provided to both machines. (To ensure that this does not convey an unfair advantage, we allow the other methods to update using the full true gradient $\grad$ in the first iteration.) Henceforth, machines dynamically adjust $y$, to ensure that we always have $\|\bm g_0-\bm g_1\|_\infty < y$, in order for decoding to be successful. For each iteration $t$, denoting as $Q(\bm g_0),Q(\bm g_1)$ the encoded (i.e., rounded to closest lattice point) estimates, machines use the value $y(t+1) = 1.5 \cdot\|Q(\bm g_0)-Q(\bm g_1)\|_\infty$ for the next iteration. Similarly, for RLQSGD, $y_R(t+1) = 1.5\cdot \|HD(Q(\bm g_0)-Q(\bm g_1))\|_\infty$. This value is common knowledge to both machines as long as decoding in iteration $t$ is successful, and sufficed, in this experiment, to ensure that all decodes succeeded.

Our baseline is \emph{naive averaging} (i.e., what would be attained if we could communicate at full precision between the machines), and we compare against variants of \emph{QSGD}~\cite{QSGD}, and the \emph{Hadamard}-based quantization scheme of~\cite{MeanEstimation}. 

Figures \ref{fig:sgdvar1} and \ref{fig:sgdvar2} demonstrate that LQSGD provides the lowest output variance among the methods, and is the only method to achieve \emph{variance reduction}, i.e. the output variance of the average of the quantized batch gradients is lower than the input variance of the \emph{individual} batch gradients.

\vspace{10pt}
\begin{figure}[ht]
	\centering
	\begin{minipage}[b]{0.45\linewidth}
		\centering
		\includegraphics[width=7cm]{figures/superlinear/variance_S_8192_d_100_log.pdf}
		\caption{variance at 3 bits per coordinate, fewer samples}\label{fig:sgdvar1}
	\end{minipage}
	\quad
	\begin{minipage}[b]{0.45\linewidth}
		\centering
		\includegraphics[width=7cm]{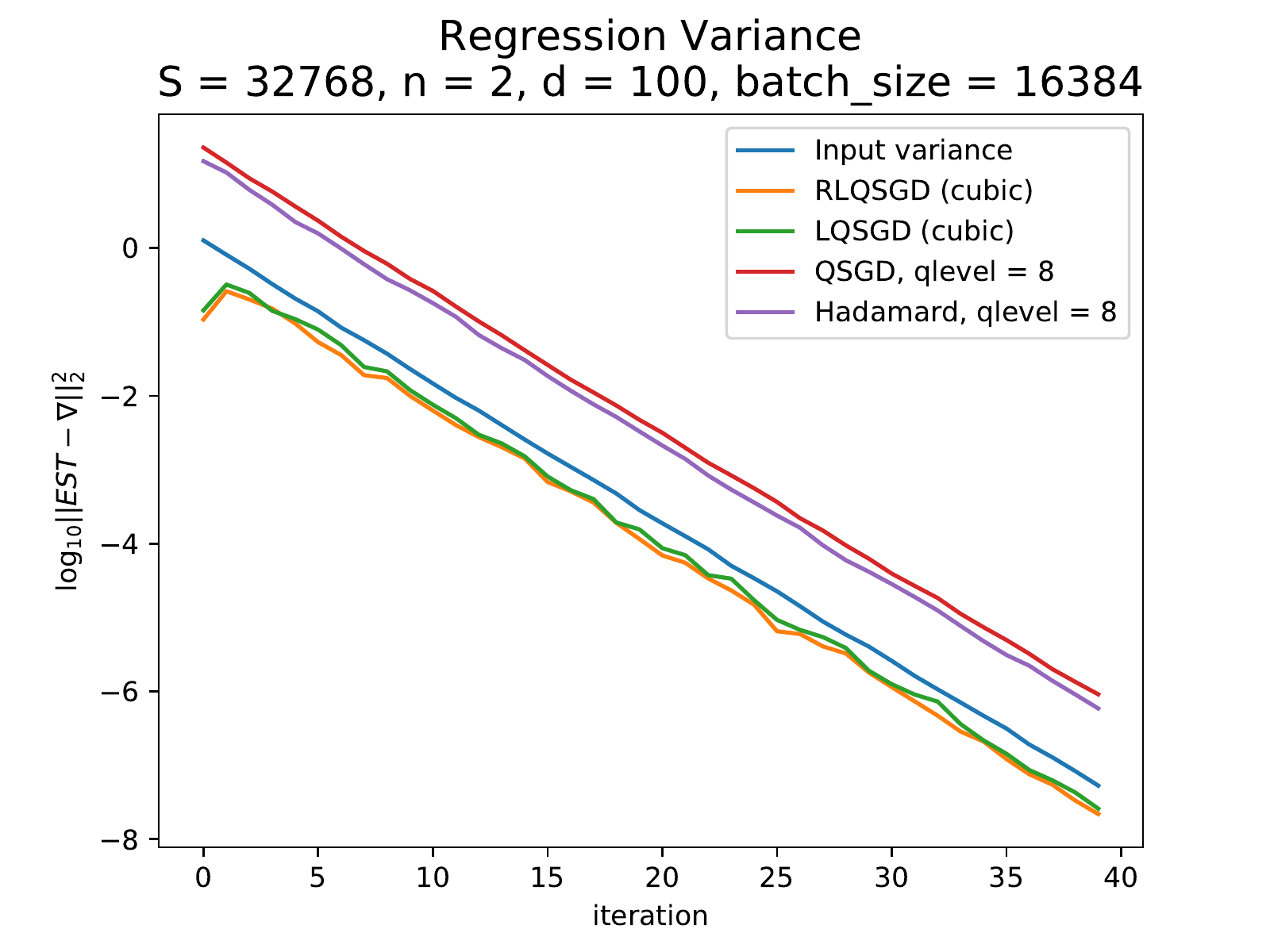}
		\caption{variance at 3 bits per coordinate, more samples}\label{fig:sgdvar2}
	\end{minipage}
\end{figure}

\paragraph{Experiment 3: Convergence using quantization methods.}
We next measure the effect on the \emph{convergence} of the SGD process of the quantization schemes, using the same input data. That is, machines now apply the unbiased estimate of $\grad$ they obtain by averaging quantized batch gradient in each iteration. To clearly show the effects of quantization on convergence, we will use a high learning rate of 0.8. Estimation of $y$, and other parameter settings, remain the same as previously. Figures \ref{fig:conv1} and \ref{fig:conv2} demonstrate faster convergence for LQSGD over other methods.

\vspace{10pt}
\begin{figure}[ht]
	\centering
	\begin{minipage}[b]{0.45\linewidth}
		\centering
		\includegraphics[width=7cm]{figures/superlinear/convergence_S_8192_d_100.pdf}
		\caption{convergence at 3 bits per coordinate, fewer samples}\label{fig:conv1}
	\end{minipage}
	\quad
	\begin{minipage}[b]{0.45\linewidth}
		\centering
		\includegraphics[width=7cm]{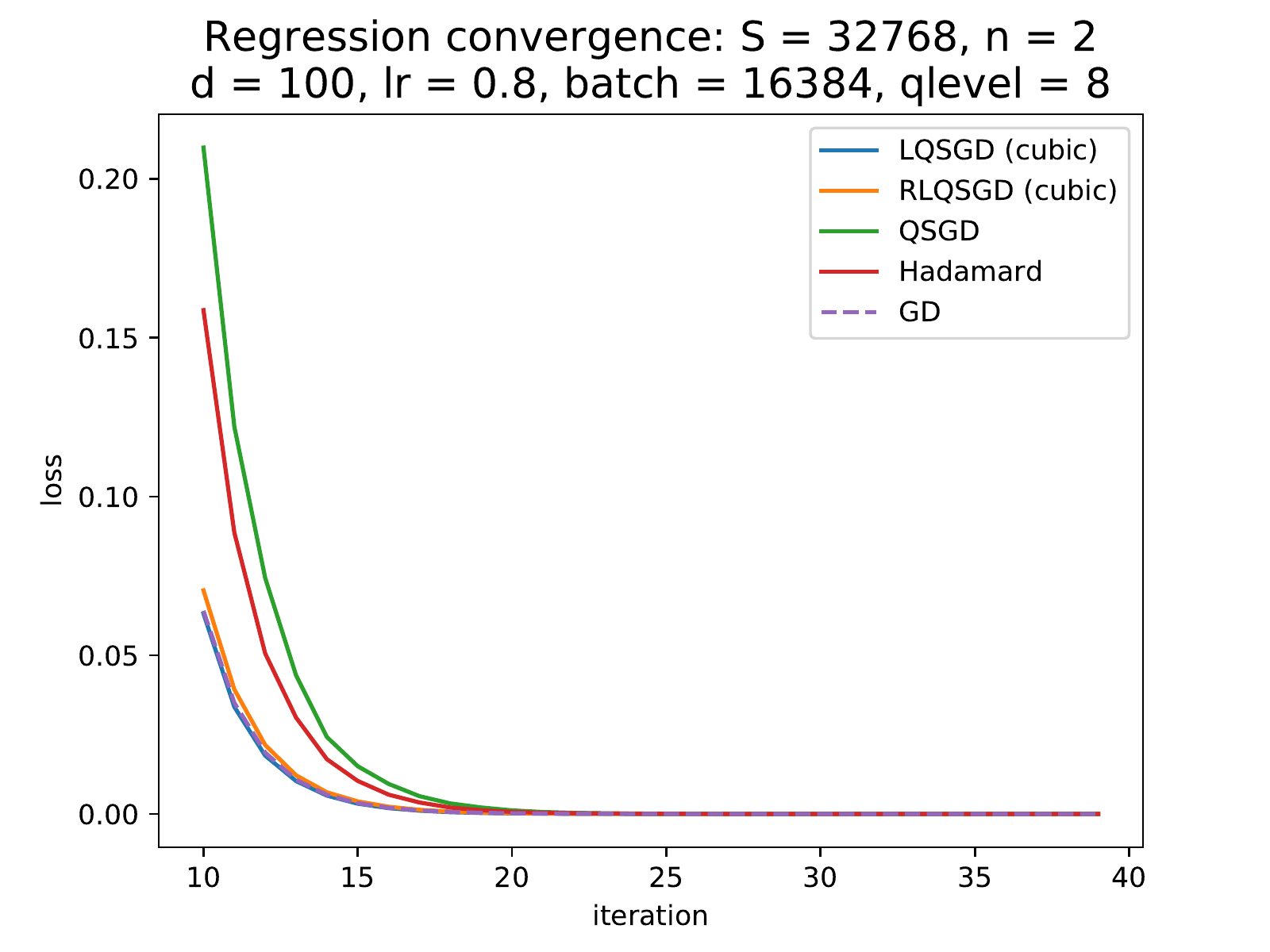}
		\caption{convergence at 3 bits per coordinate, more samples}\label{fig:conv2}
	\end{minipage}
\end{figure}

\paragraph{Experiment 4: Exploration of potential variance with sub-linear quantization.}
We also wish to test the performance of our sublinear quantization scheme from Section \ref{sec:sublinear}. Unfortunately, a naive implementation of the scheme (particularly, of finding the closest lattice point of a particular color under the chosen random coloring) is not computationally feasible under high enough dimension for meaningful results. Therefore, here we will only simulate what the effect of the quantization would be on output variance compared to state-of-the-art schemes; in future work we will explore how the scheme can be adapted for efficient implementation.

In this experiment we again have two worker machines $u,v$, and we compare the variance induced when $u$ sends its quantized batch gradient $\bm g_0$ to $v$, who decodes it using $\bm g_1$. The sublinear-communication quantization method that we compare with is the vQSGD cross polytope method  with repetition \cite{VQSGD}. We will perform our experiments using $\frac{d}{2}$ bits, i.e., 0.5 bits per coordinate (and so we set the number of vQSGD repetitions accordingly).

\out{
	At $O(\sqrt{d})$ bits, vQSGD achieves $O(\sqrt{d}\log(d)\Vert g_0\Vert^2_2)$ variance whereas LQSGD achieves $O(d\Vert g_0-g_1\Vert^2_2)$ variance. So if $\Vert g_0\Vert_2$ is not too much larger than $\Vert g_0 - g_1\Vert_2$, then vQSGD has a clear edge here. On the other hand, at $O(d)$ bits, vQSGD achieves $O(\log(d)\Vert g_0\Vert^2_2)$ variance whereas LQSGD achieves $O(\Vert g_0-g_1\Vert^2_2)$ variance. Hence, we will perform our experiments at $\frac{d}{2}$ bits i.e 0.5 bit per coordinate.}

We use a slightly different means of dynamically updating $y$, since variance from quantization means that the previous method no longer gives good estimates. Now, once in 5 iterations, machine $u$ receives 2 batches, which allows it to compute two gradient estimates $\bm g_0$ and $\bm g_0'$, and compute $y = 1.6\cdot||\bm g_0-\bm g_1||_\infty$ (which suffices to ensure correct decoding in this experiment). It then sends this to $v$ as a 64-bit floating point value. Note that this method of updating $y$ generalizes to many machines, and that the constant factor $1.6$ can be changed as necessary.

We use the cubic lattice, and as before denote its side length by $s$. Given $\|\bm g_0-\bm g_1\|_\infty < y$. We can see from the analysis of the sublinear scheme (in particular, in proof of Lemma \ref{lem:vorno}) that if we choose a number of \emph{colors} (i.e., bit-strings) equal to the expected number of expanded Voronoi regions covering a point, quantization will succeed with at least constant probability. This gives an expression of $\log_2\left(1+2y/r_p\right)^d = d\log_2\left(1+4y/s\right)$ bits. To use $0.5d$ bits, we set $\log_2\left(1+4y/s\right) = 0.5$, from which we get $s = 4y/(\sqrt{2}-1)$ (though our number of bits and therefore value of $s$ slightly differs from this in order to exactly match the communication used by vQSGD, where number of repetitions must be an integer). Since the randomization in the quantization effectively shifts each coordinate independently in $[-s/2,s/2]$, we get variance $ds^2/12$, which is what we plot in the figures. 

Figures \ref{fig:subvar1} and \ref{fig:subvar2} demonstrate that sublinear LQSGD could be competitive with state-of-the-art methods, though it only outperforms when using large numbers of samples with respect to dimension. The steps in the graph are due to the periodic updates to $y$.

\vspace{10pt}
\begin{figure}[ht]
	\centering
	\begin{minipage}[b]{0.45\linewidth}
		\centering
		\includegraphics[width=7cm]{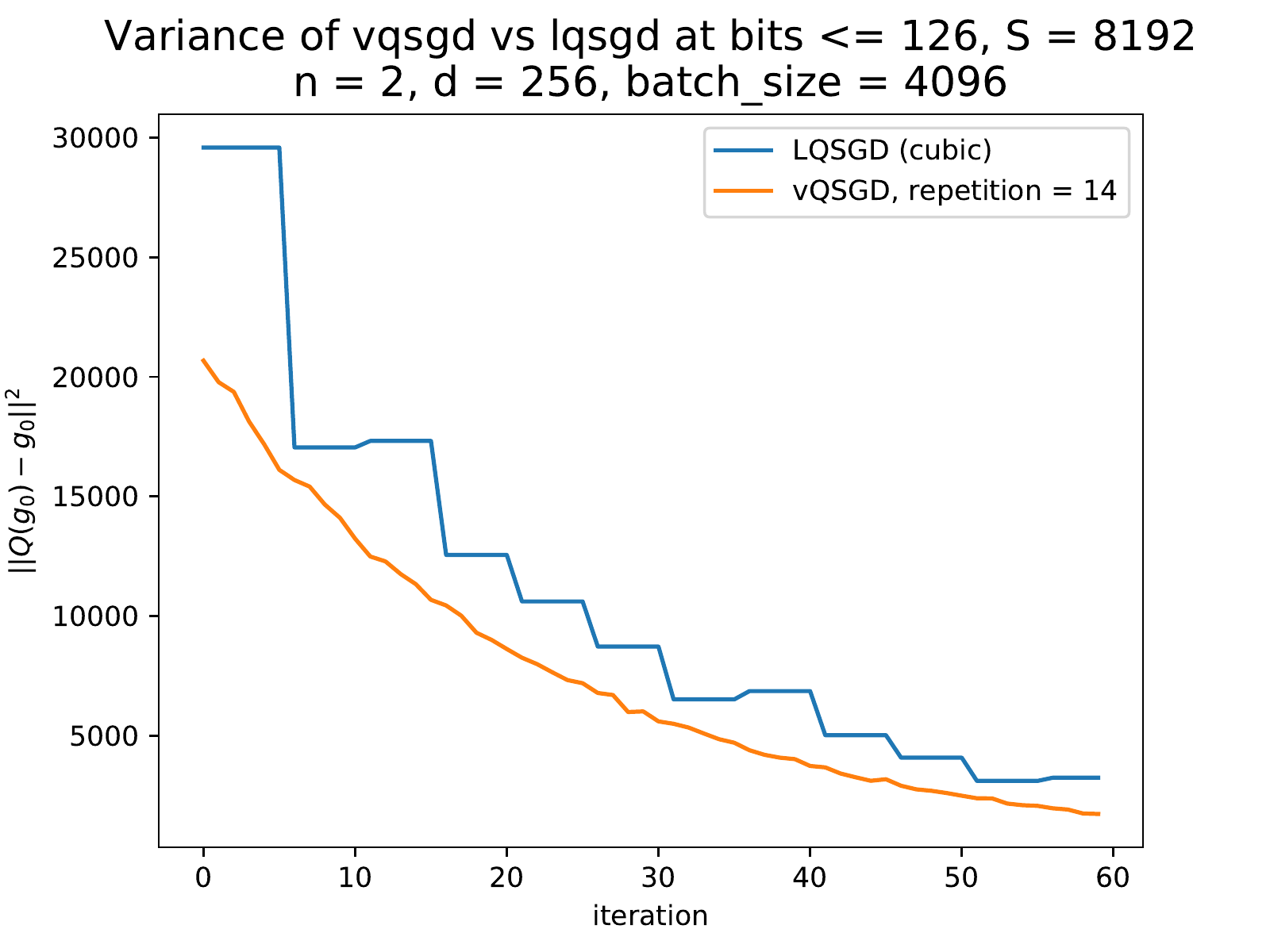}
		\caption{sublinear schemes, variance for fewer samples}
		\label{fig:subvar1}
	\end{minipage}
	\quad
	\begin{minipage}[b]{0.45\linewidth}
		\centering
		\includegraphics[width=7cm]{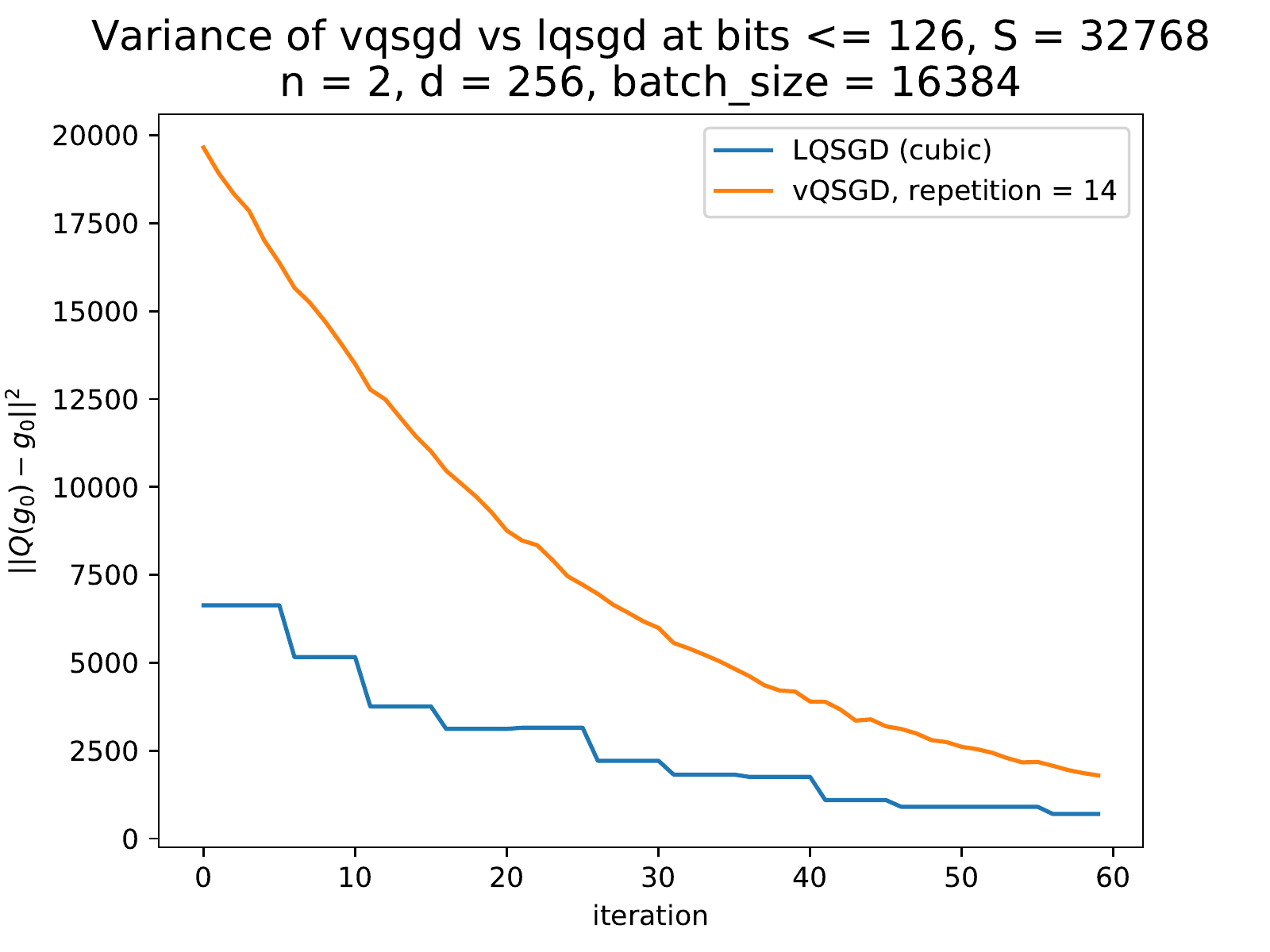}
		\caption{sublinear schemes, variance for more samples}
		\label{fig:subvar2}
	\end{minipage}

\end{figure}

\paragraph{Experiment 5: Convergence on real dataset using quantization schemes with more than 2 machines.}
We now test the performance of the (linear/superlinear) quantization scheme in a more realistic setting, on a real dataset using multiple machines. We use the dataset \texttt{cpusmall\_scale} from the LIBSVM datasets \cite{LibSVM}, which has $S=8192$ and $d=12$. We initialize the initial weight to a point far from the origin relative to $\bm w_{\text{opt}}$, specifically the vector of $-1000$ in each coordinate. This is to synthetically study the behavior one can expect on a general convex loss function when $w_{\text{opt}}$ is arbitrary. We study the convergence behavior with $q = 16$, $n = 8, 16$ and batch size = $\frac{S}{n}$.

\begin{figure}[ht]
	\centering
	\begin{minipage}[b]{0.45\linewidth}
		\centering
		\includegraphics[width=7cm]{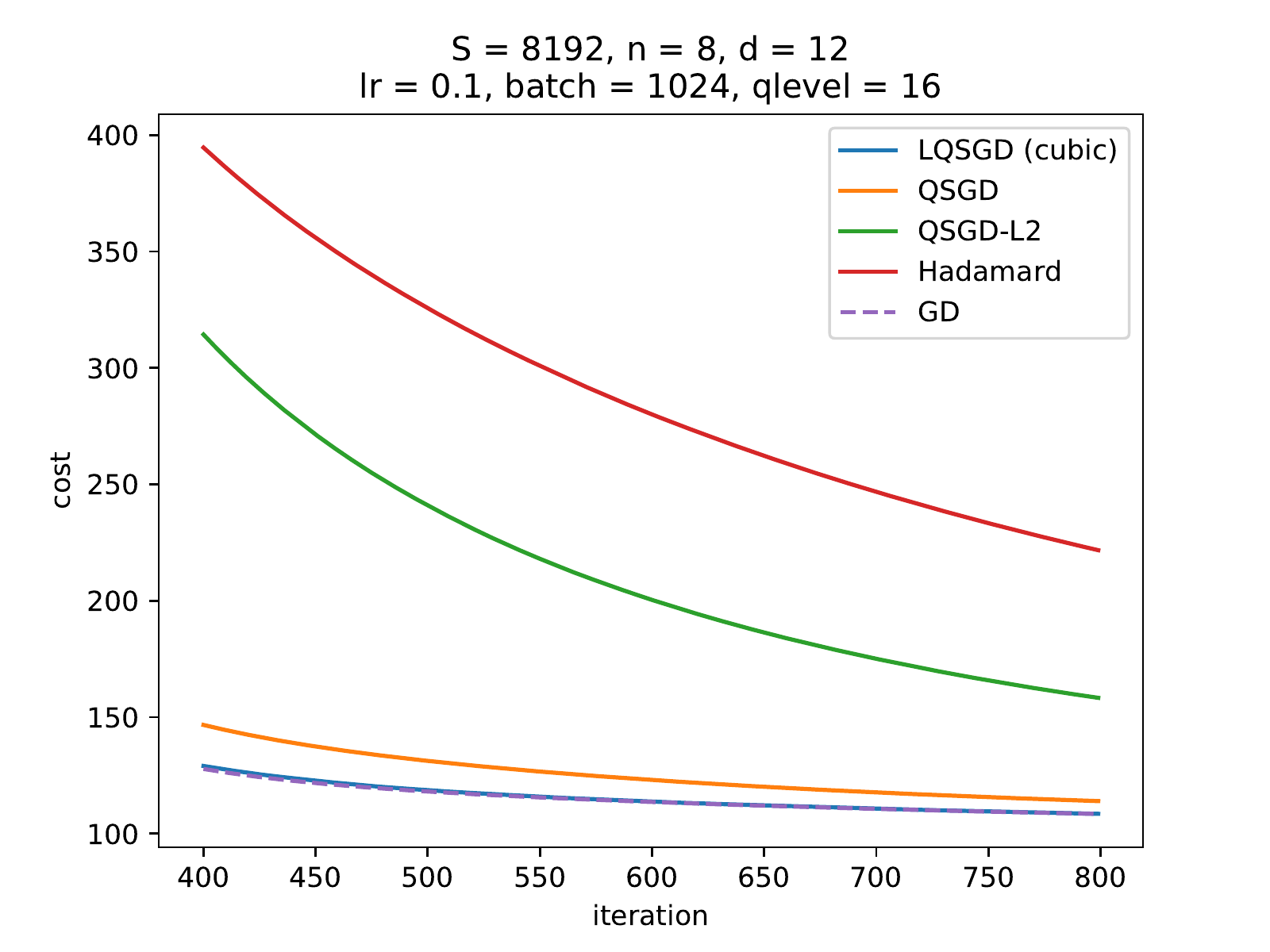}
		\caption{\texttt{cpusmall\_scale}, 8 machines}\label{fig:libsvm1}
	\end{minipage}
	\quad
	\begin{minipage}[b]{0.45\linewidth}
		\centering
		\includegraphics[width=7cm]{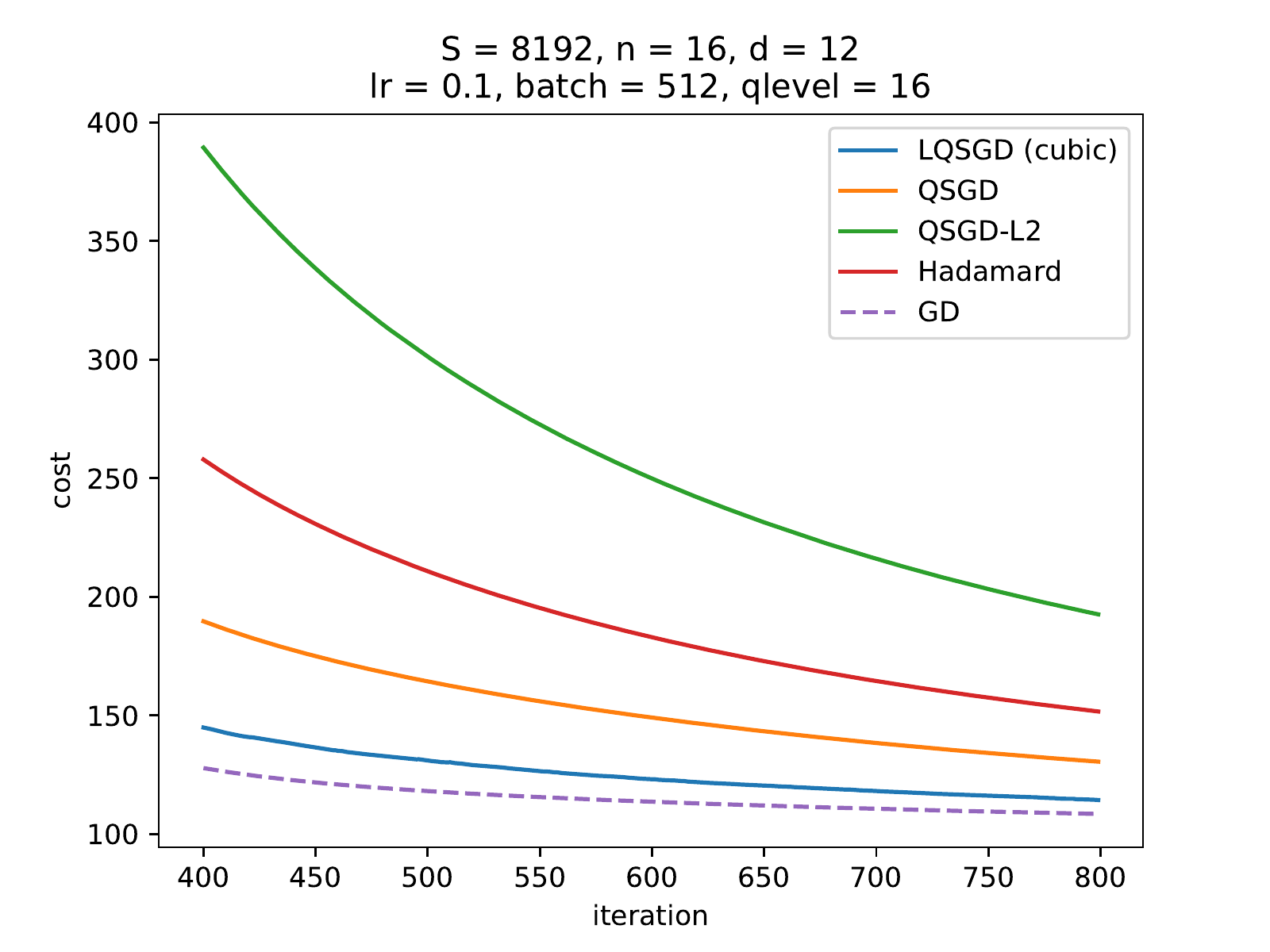}
		\caption{\texttt{cpusmall\_scale}, 16 machines}\label{fig:libsvm2}
	\end{minipage}
\end{figure}

We apply Algorithm \ref{alg:ME}, with one randomly chosen machine acting as leader to collect all the quantized batch gradients $Q(\bm g_i)$, and broadcast the average as a quantized vector i.e $Q\left(\sum_i Q(\bm g_i)/n\right)$. This machine also uses the quantized gradients $Q(\bm g_i)$ to evaluate and broadcast $y$ for the next round as a 64-bit float using $y = 3\cdot \max_{i,j}||Q(\bm g_i)-Q(\bm g_j)||_\infty$, where the factor $3$ is again chosen large enough to ensure that decoding is always successful in this experiment (and can be altered as necessary in other settings).

Figures \ref{fig:libsvm1} and \ref{fig:libsvm2} exhibit fast convergence compared to other methods.

\subsection{Local SGD}

A related application is that of compressing \emph{models} in LocalSGD~\cite{stich2018local}, where each node takes several SGD steps on its local model, followed by a global \emph{model averaging} step, among all nodes. 
(Similar algorithms are popular in Federated Learning~\cite{kairouz2019advances}.) 

\paragraph{Experiment 6: Local SGD convergence on synthetic data.}
We use RLQSGD to quantize the models transmitted by each node as part of the averaging: to avoid artificially improving our performance, we compress the \emph{model difference} $\Delta_i$ between averaging steps, at each node $i$. 
RLQSGD is a good fit since neither the models nor the $\Delta_i$ are zero-centered. We consider the same setup as in Section \ref{sec:LSexp}, averaging every 10 local SGD iterations. We illustrate the convergence behavior and the quantization error in Figure~\ref{fig:local}, which shows better convergence and higher accuracy for lattice-based quantization. 

\begin{figure}[h!]
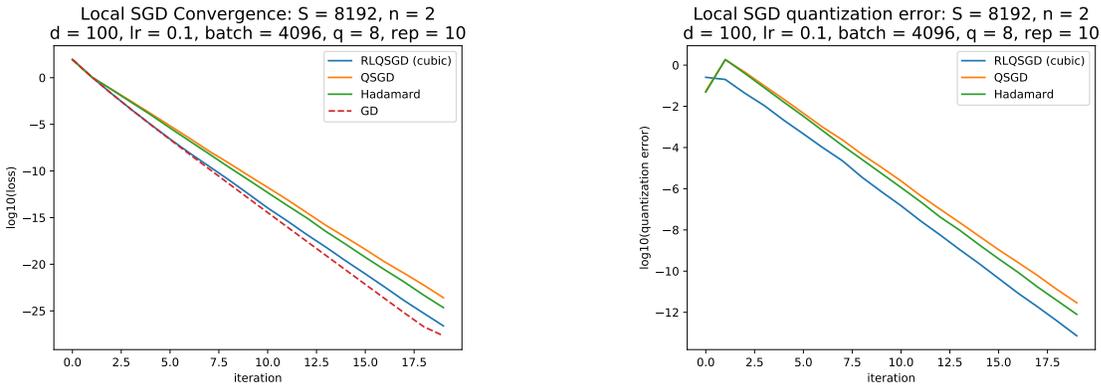

	\begin{subfigure}{0.49\textwidth}
		\includegraphics[width=7cm]{figures/local_sgd/convergence_S_8192_d_100_log.pdf}
	\end{subfigure}
	\hfill
	\begin{subfigure}{0.49\textwidth}
		\includegraphics[width=7cm]{figures/local_sgd/qe_S_8192_d_100_log.pdf}
	\end{subfigure}
	\caption{Local SGD: convergence for different quantizers (left) and quantization error (right).}
	\label{fig:local}
\end{figure}

\subsection{Neural Network Training}

Our next task is applying LQSGD for training neural networks in a data-parallel environment, where we use it to average gradients. 
We note that this application slightly extends the setting for distributed mean estimation~\cite{MeanEstimation}, since other methods can, e.g., employ historical information from previous iterations~\cite{stich2018sparsified, karimireddy19a}. 

\paragraph{Experiment 7: Gradient compression for neural network training.}
We train ResNets~\cite{he2016deep} on a  subset of 10 classes from the ILSVRC dataset~\cite{ImageNet}, with full-sized images, as well as on the CIFAR dataset~\cite{CIFAR10}. The former dataset is popular for model and hyperparameter calibration~\cite{imagenette} relative to the full version of the dataset, and model performance on it is known to correlate well with performance on the full ILSVRC dataset. The latter is a classic small dataset in computer vision.  

The results in Figure~\ref{table:quant} show the Top-1 validation accuracy on ResNet18 (11M parameters) for LQSGD with an average of $4$ bits per entry, versus 4-bit QSGD (L2 and LInf normalized variants), PowerSGD~\cite{vogels2019powersgd} with rank $16$ (as suggested by the authors), and 1bitSGD/EFSignSGD~\cite{Seide14, karimireddy19a}. For LQSGD, each node uses one batch per epoch to estimate $\sigma$, and uses $y = 3\sigma$ as its upper bound estimate. 
(This results in a rate of incorrect decodings of $\sim 3\%$, which we allow, but which has no impact on convergence.) Results are averaged over 2 runs, since variance is small.

We note that all compression algorithms lose accuracy relative to the full-precision baseline. (The large gap for EFSignSGD is reasonable, since it uses the fewest bits.) Perhaps surprisingly, LQSGD loses the least accuracy relative to the baseline, although the difference relative to QSGD is small.  
We conclude that, although gradient compression is not the main application for lattice-based quantization, it can yield competitive results for this task.  

Figure~\ref{fig:cifar100} presents additional accuracy results for neural network training with compressed gradients, training ResNet20 on the CIFAR-100 dataset~\cite{CIFAR10,cifar10data}. 
The network is trained for 200 epochs, with standard hyper-parameters. Quantization is applied at the level of each layer, and we use 4 bits per coordinate (except for EF-SignSGD which uses approximately one bit per component). 
The results confirm that LQSGD can be competitive with state-of-the-art quantization methods.

\begin{figure}
	\centering
	\begin{minipage}{.47\textwidth}
		\centering
		\includegraphics[width=0.9\linewidth]{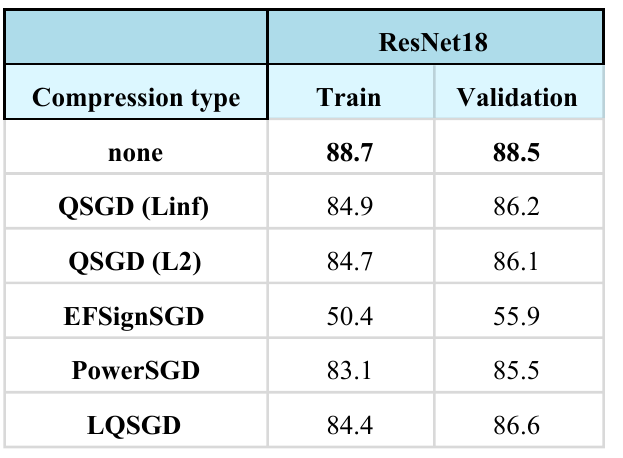}
		\caption{Accuracy results for ResNet18\\ on ILSVRC}
		\label{table:quant}
	\end{minipage}
	\begin{minipage}{.47\textwidth}
		\centering
		\includegraphics[width=0.9\linewidth]{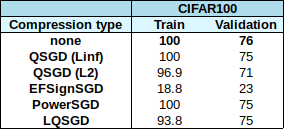}
		\caption{Top-1 Validation Accuracy  for ResNet20 on CIFAR-100.}
		\label{fig:cifar100}
	\end{minipage}%
	
\end{figure}

\subsection{Power Iteration}\label{sec:power-iteration}
Power iteration is an algorithm for finding the principal eigenvector of an input matrix $X$. In a distributed setting over two machines, the algorithm works as follows:
the rows of our input matrix $X$ are partitioned into two subsets $X_0$, $X_1$. At the beginning of each round, both machines have the same unit vector $\bx$, which is the current estimate of the principal eigenvector of $X$. During the round, they must perform the update $\bx \gets \frac{X^T X \bx}{||X^T X \bx||}$. For this, machine $i$ evaluates $\bm u_i = X_i^T X_i \bx$ and shares it with the other machine; both machines can then calculate $X^T X \bx = X_0^T X_0\bx + X_1^T X_1\bx = \bm u_0+\bm u_1$, and thereby perform the update step. We apply quantization methods to communicate these vectors $\bm u_i$, in order to test the performance of LQSGD in this setting. We also apply the method on $8$ worker machines in order to test how our methods scale with more machines.

\begin{figure}[ht]
	\centering
	\begin{minipage}[b]{0.3\linewidth}
		\centering
		\includegraphics[width=5cm]{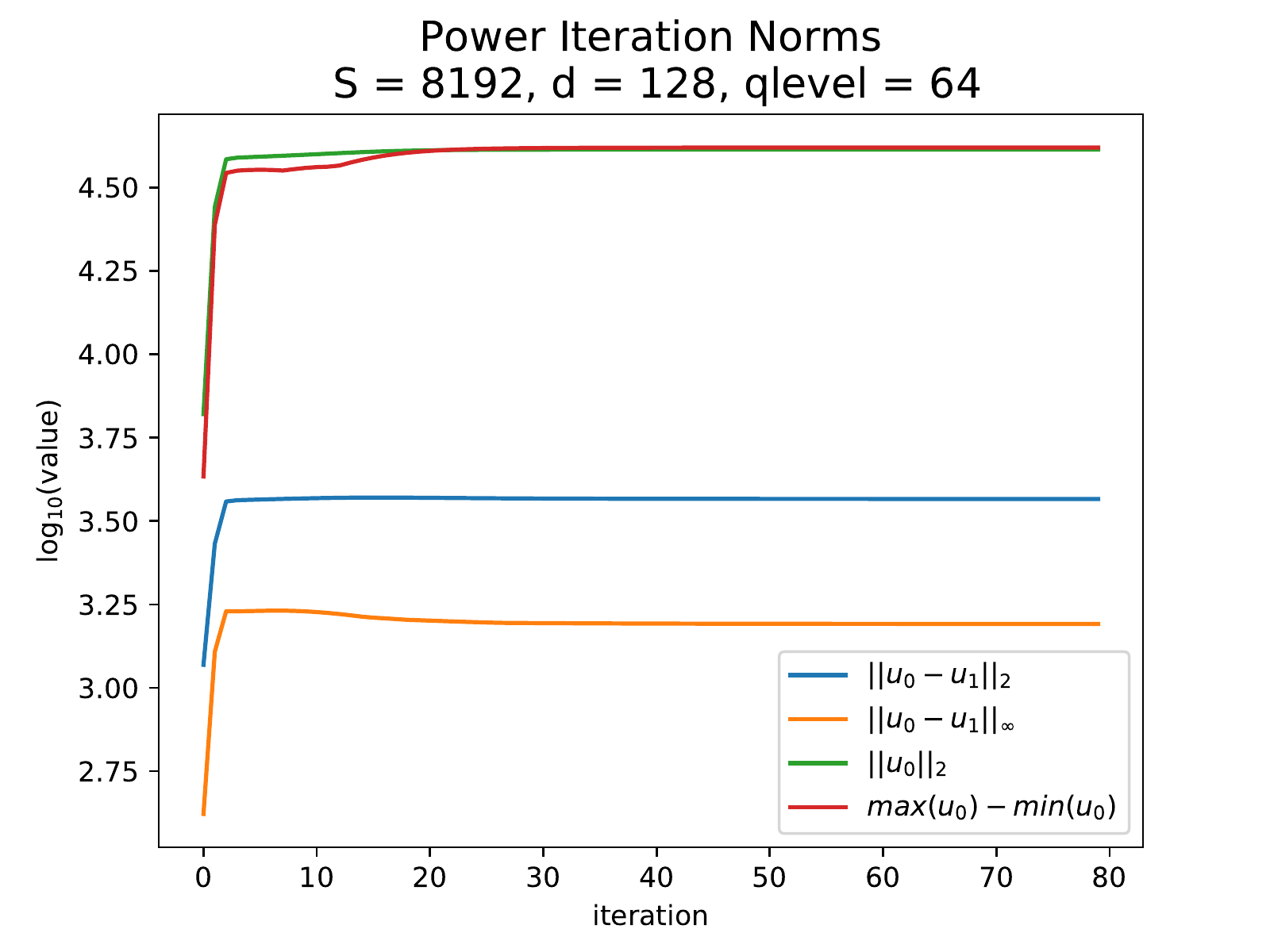}
	\end{minipage}
	\quad
	\begin{minipage}[b]{0.3\linewidth}
		\centering
		\includegraphics[width=5cm]{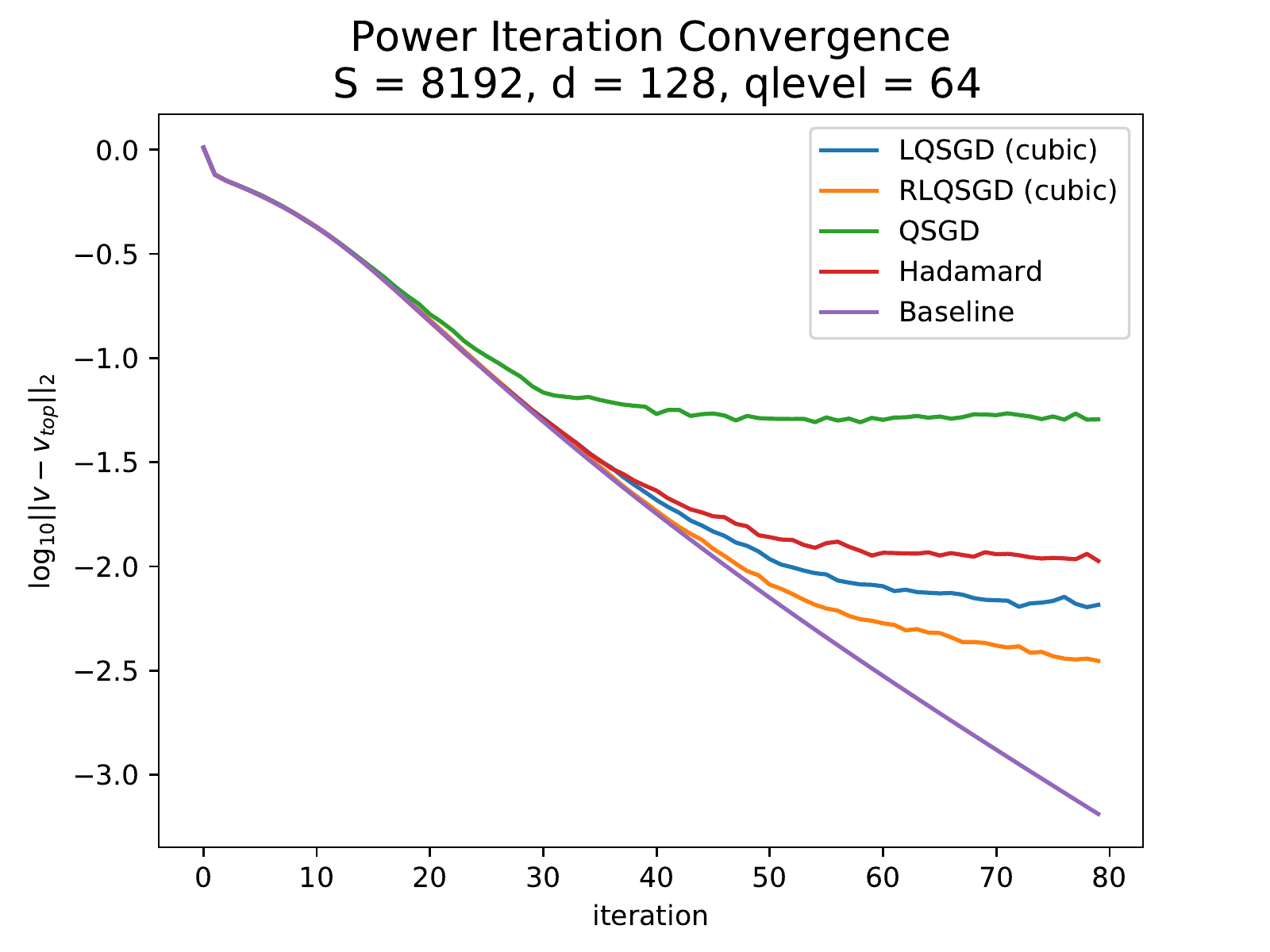}
	\end{minipage}
	\quad
	\begin{minipage}[b]{0.3\linewidth}
		\centering
		\includegraphics[width=5cm]{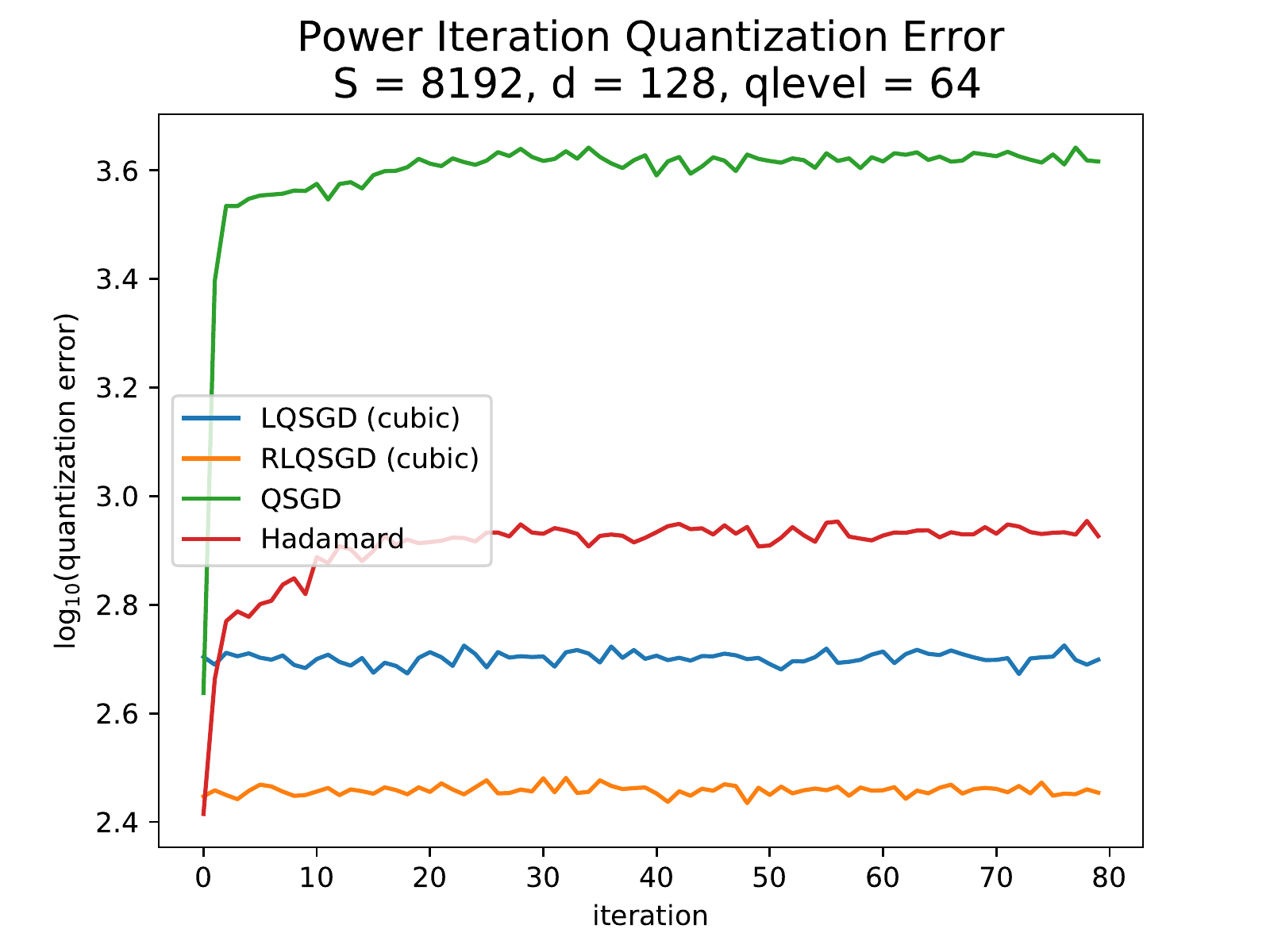}
		
	\end{minipage}
	\caption{Power iteration: input norms (left), convergence (center) and quantization error (right). Principal eigenvector is $e_2$.}
	\label{fig:power1}
\end{figure}

\begin{figure}[ht]
	\centering
	\begin{minipage}[b]{0.3\linewidth}
		\centering
		\includegraphics[width=5cm]{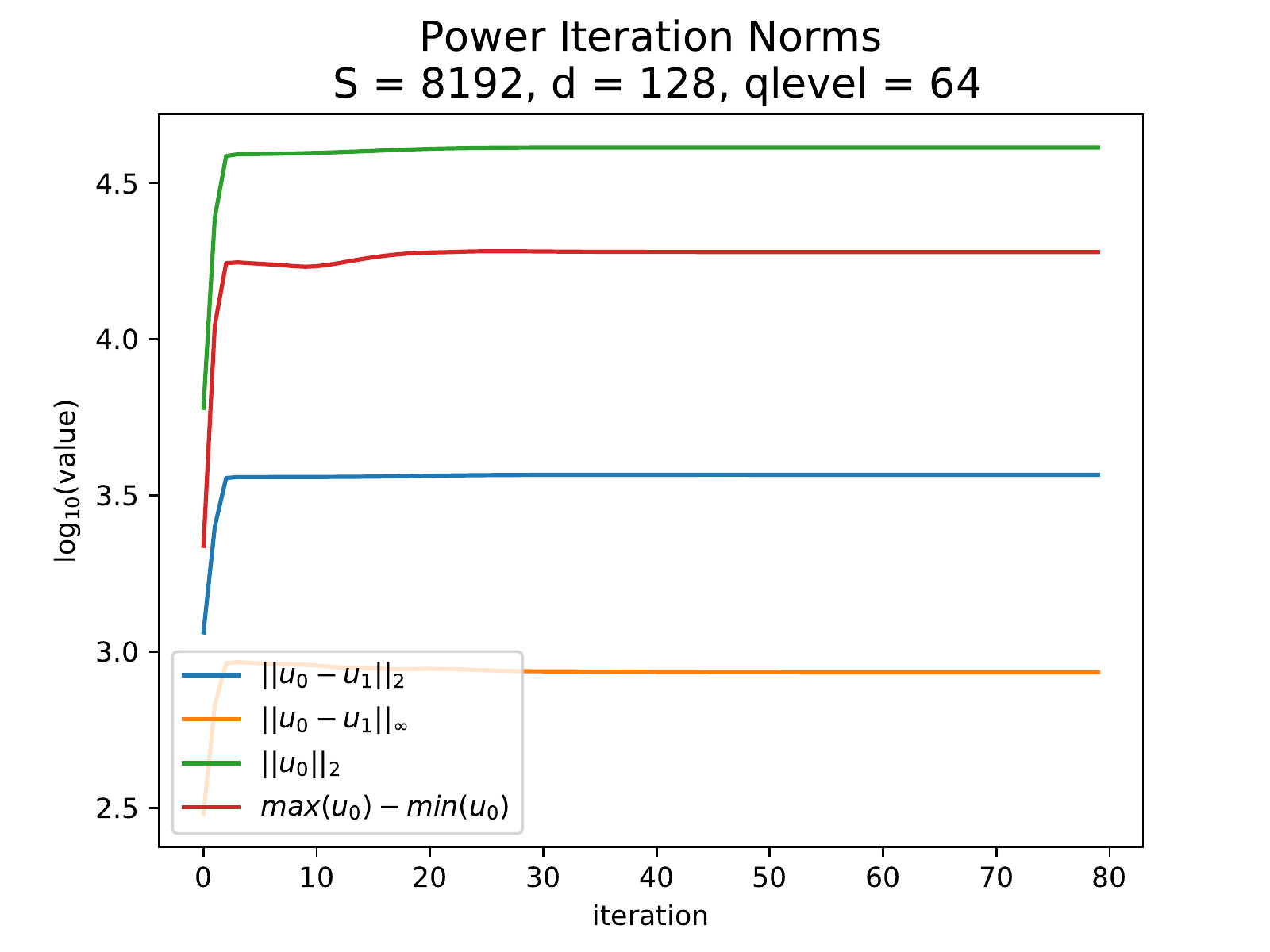}
	\end{minipage}
	\quad
	\begin{minipage}[b]{0.3\linewidth}
		\centering
		\includegraphics[width=5cm]{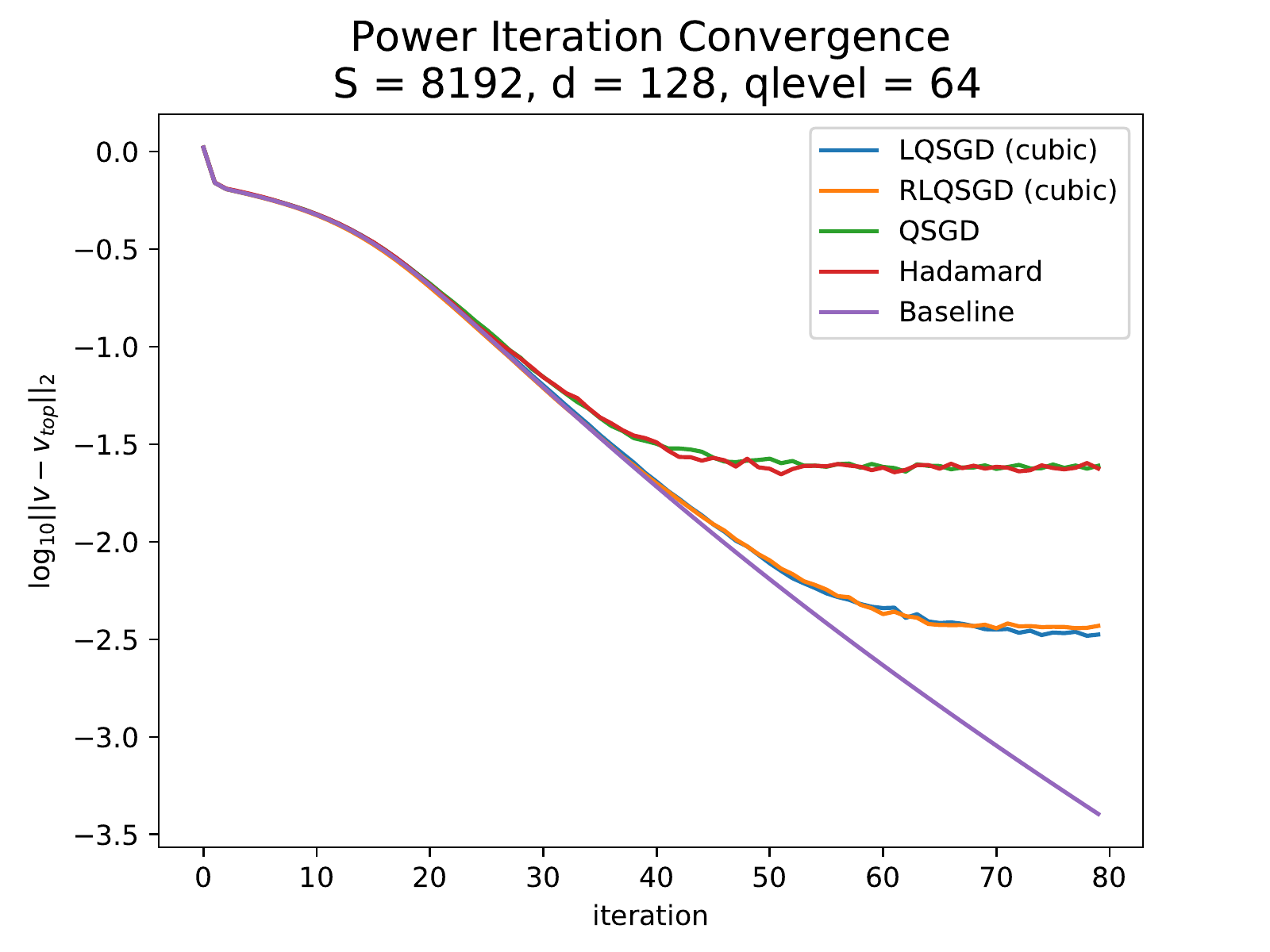}
	\end{minipage}
	\quad
	\begin{minipage}[b]{0.3\linewidth}
		\centering
		\includegraphics[width=5cm]{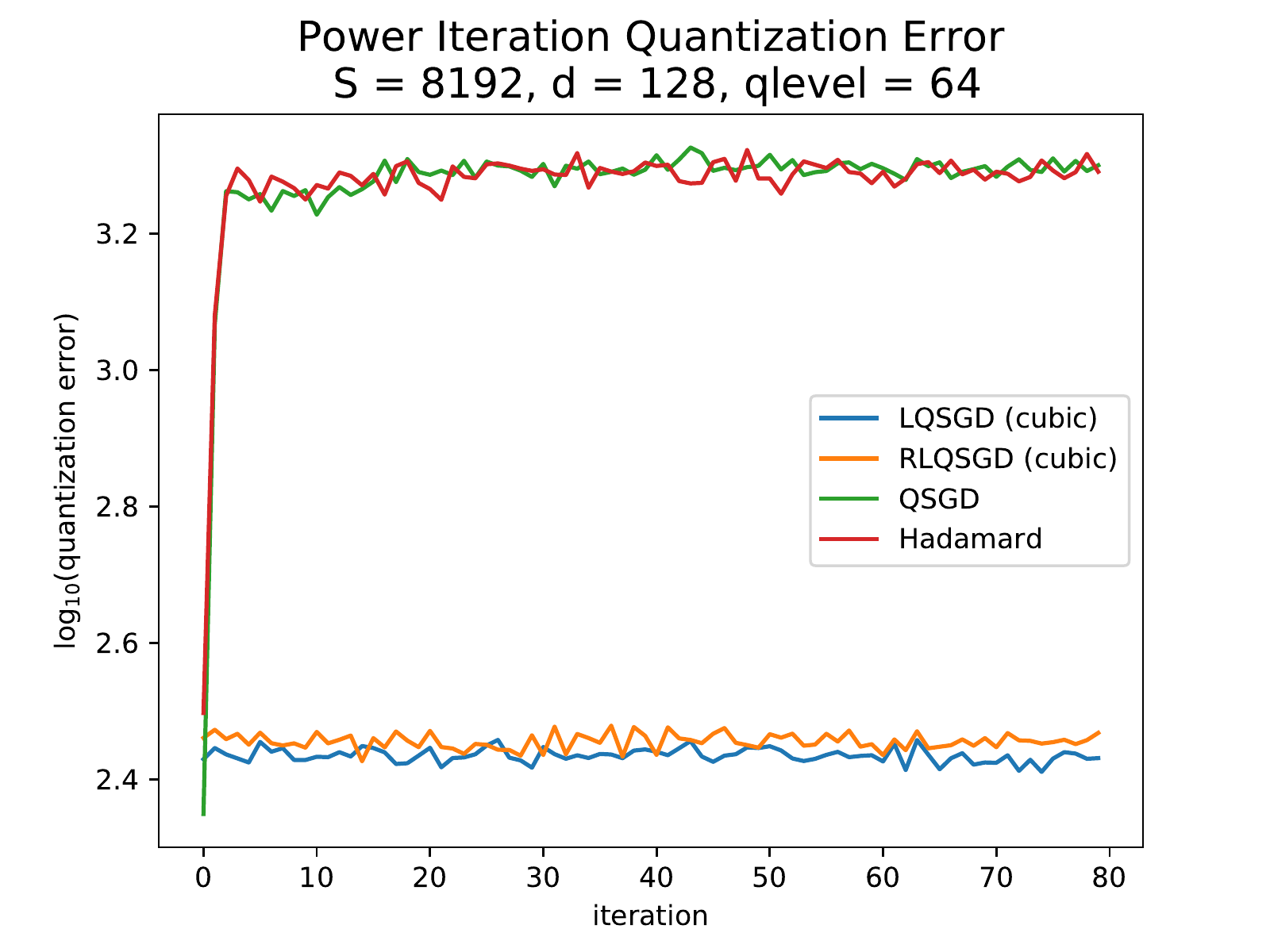}
		
	\end{minipage}
	\caption{Power iteration: input norms (left), convergence (center) and quantization error (right). Principal eigenvector is along a random direction.}
	\label{fig:power2}
\end{figure}

\begin{figure}[h]
	\centering
	\begin{minipage}[b]{0.3\linewidth}
		\centering
		\includegraphics[width=5cm]{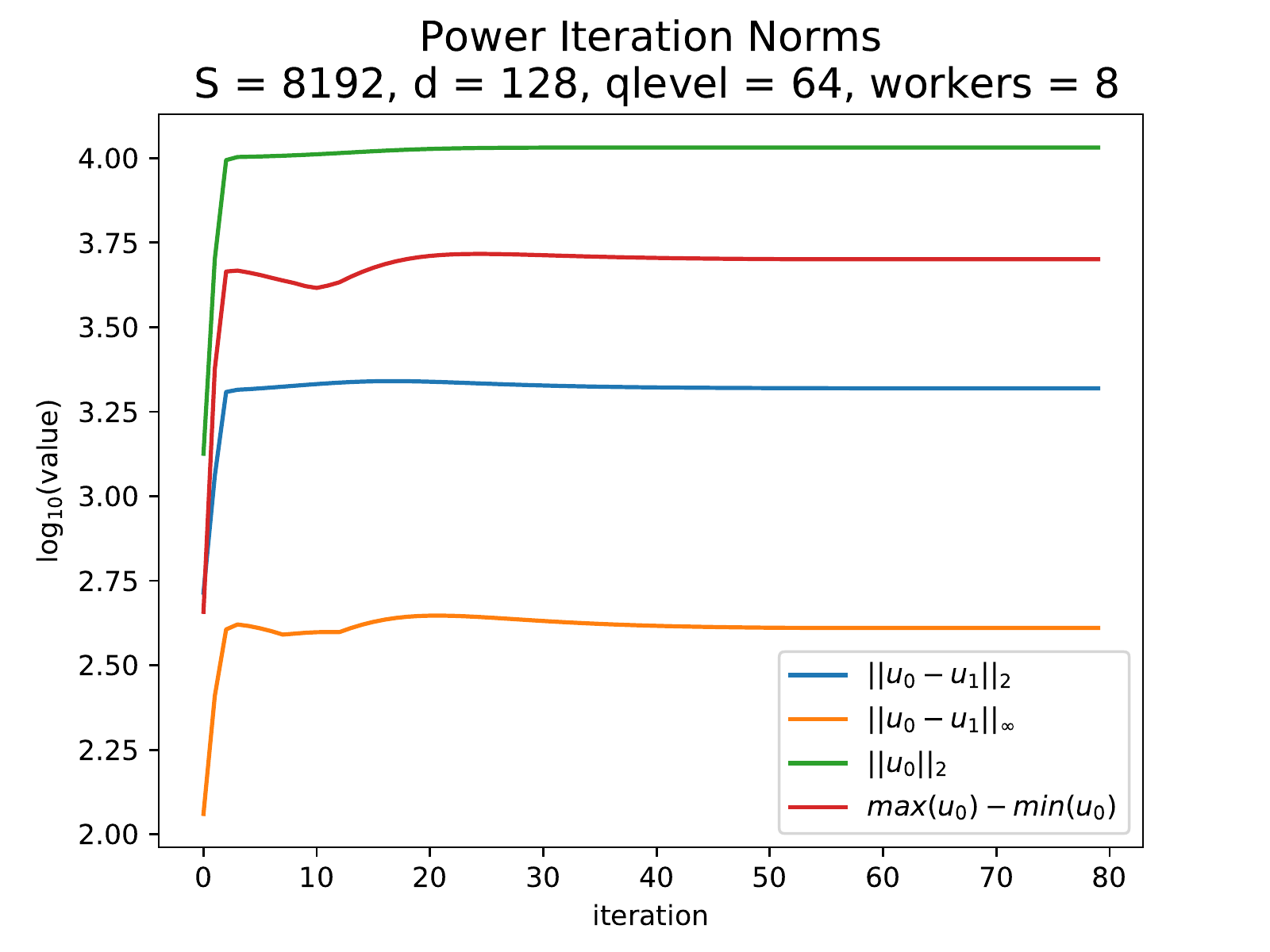}
	\end{minipage}
	\quad
	\begin{minipage}[b]{0.3\linewidth}
		\centering
		\includegraphics[width=5cm]{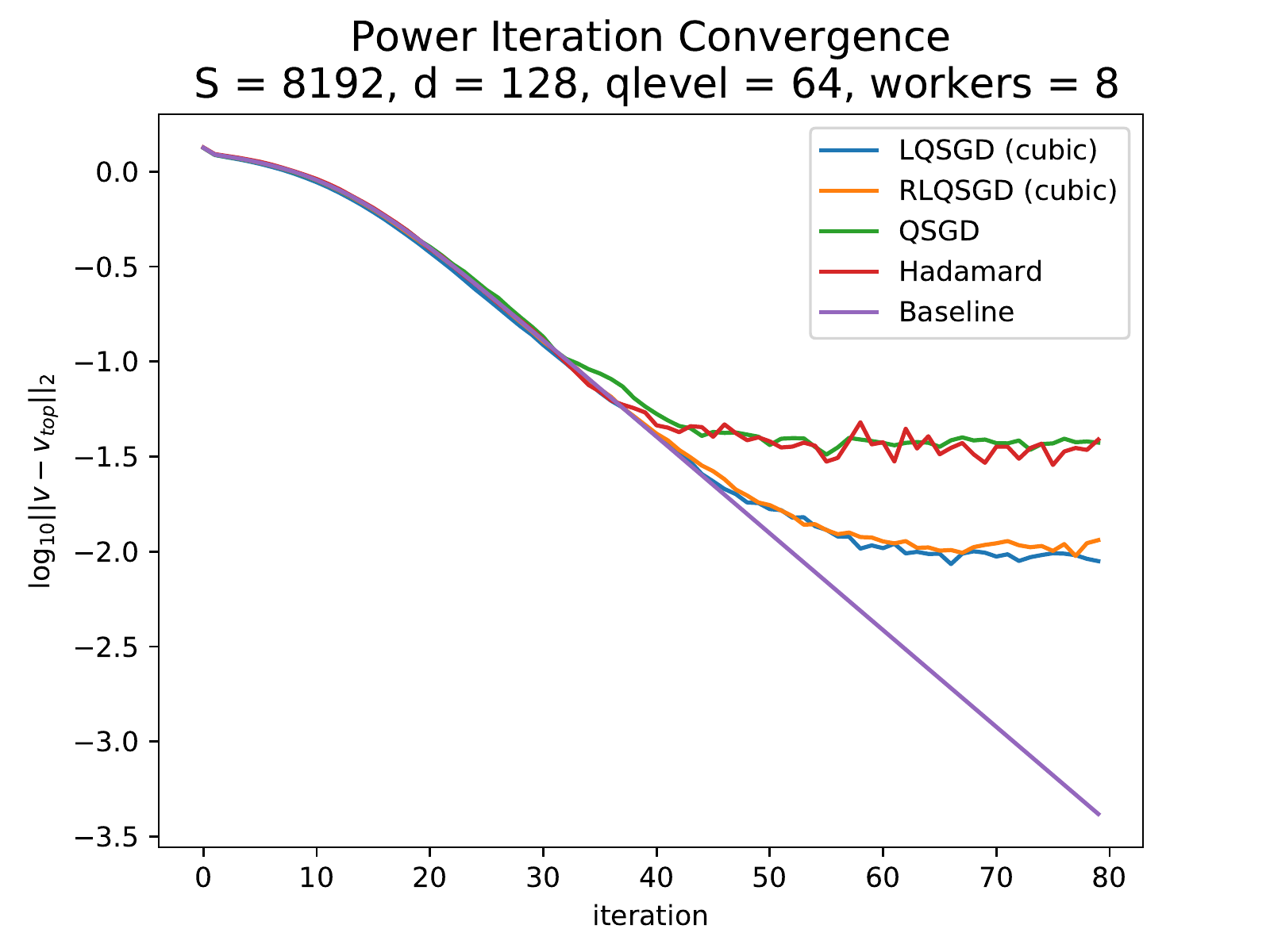}
	\end{minipage}
	\quad
	\begin{minipage}[b]{0.3\linewidth}
		\centering
		\includegraphics[width=5cm]{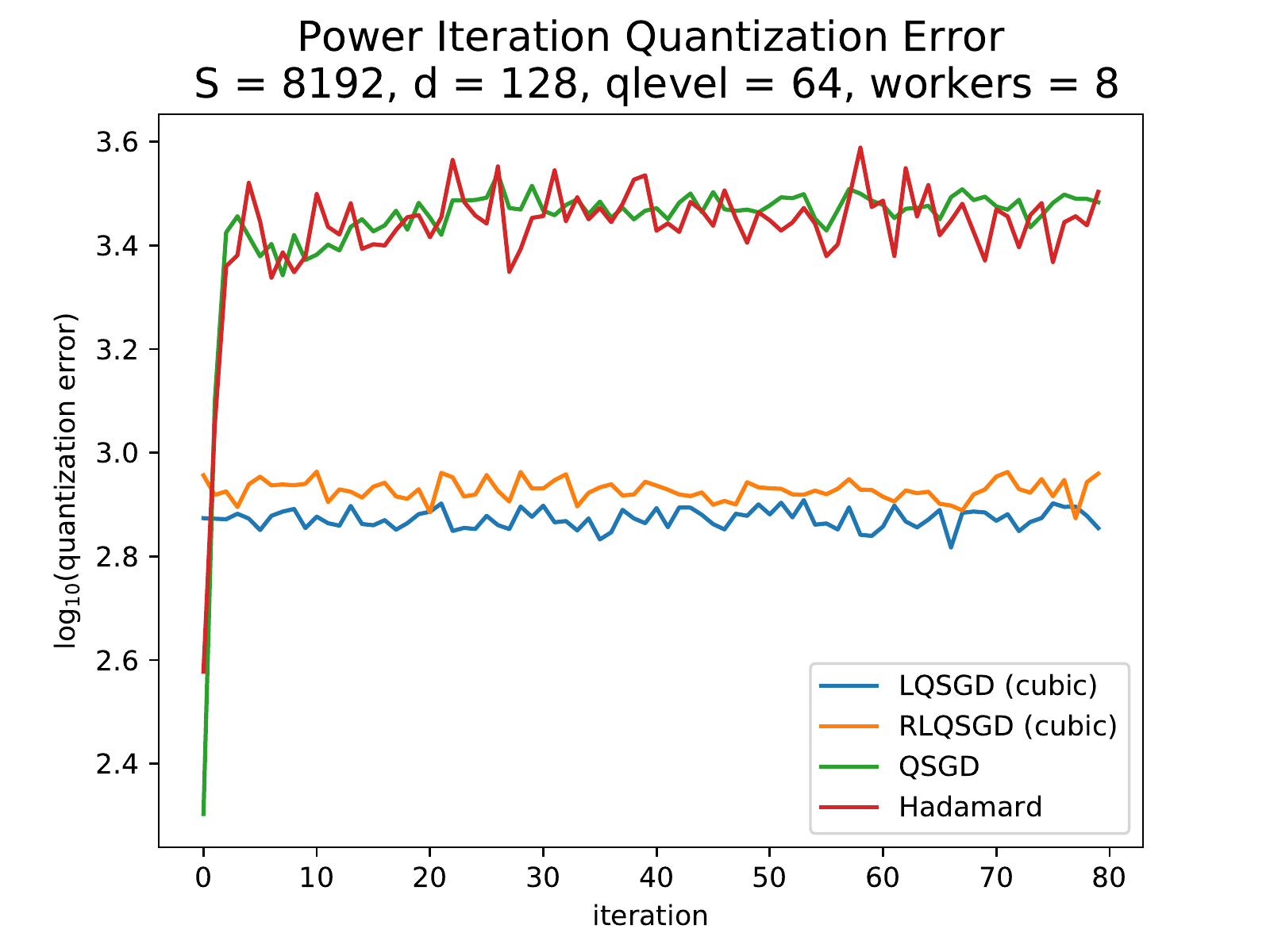}
		
	\end{minipage}
	\caption{Input norms (left), convergence (center) and quantization error (right) when executing distributed power iteration on $8$ parallel workers.}
	\label{fig:power8}
\end{figure}

\paragraph{Experiment 8: Quantization for compression during power iteration.}
Each row of the input matrix $X$ is generated from a multivariate gaussian with first two eigenvalues large and comparable, which is to ensure that power iteration doesn't converge too fast that we cannot observe the effect of quantization. The estimate $\bx$ of the principal eigenvector is initialized to a random unit vector. We use $S = 8192$ samples, dimension $d=128$, $q = 64$ i.e., 6 bits per coordinate.  

From Figure \ref{fig:power1}, \ref{fig:power2} one can see that the relevant norms fit with our general message concerning norms: distance between quantization inputs (used in LQSGD and RLQSGD) is substantially lower than difference between coordinate values within a single quantization input (used as a surrogate for input norm in QSGD). Also, we see that these norms settle quickly (relative to the number of iterations one would run to get a good estimate of the principal eigenvector) and are then near-constant. Hence, one can run the baseline algorithm (either on a single machine, or assuming full-precision communication) for a few iterations until $||\bm u_0-\bm u_1||_\infty$ stabilizes, and then set the value of $y$ for LQSGD to  $2\cdot \max ||\bm u_0-\bm u_1||_\infty$ (where the maximum is over all iterations currently seen). Similarly, for RLQSGD we set $y_R = 2\cdot \max ||HD(\bm u_0-\bm u_1)||_\infty$. We then run LQSGD and RLQSGD from iteration $0$, but using the computed value of $y$.

Upon doing so, Figure \ref{fig:power1}, \ref{fig:power2} demonstrates substantially better estimation of the principal eigenvector for {RLQSGD} and LQSGD compared to other quantization methods. Figure \ref{fig:power8} shows similar results on $8$ parallel workers.

\vspace{10pt}

	\section{Conclusions}
	We have argued in this work that for the problems of distributed mean estimation and variance reduction, one should measure the output variance in terms of the input variance, rather than the input norm as used by previous works. Through this change in perspective, we have shown truly optimal algorithms, and matching lower bounds, for both problems, independently of the norms of the input vectors. This improves significantly over previous work whenever the inputs are not known to be concentrated around the origin, both theoretically and in terms of practical performance. In future work, we plan to explore practical applications for variants of our schemes, for instance in the context of federated or decentralized distributed learning.

\bibliographystyle{plain}

	\bibliography{arxiv}

\end{document}